\pdfoutput=1
\documentclass[acmsmall,screen,noacm,authorversion]{acmart}

\usepackage{amsmath}
\usepackage{amsfonts}
\usepackage{bm}

\def\eqref#1{equation~\ref{#1}}

\def\1{\bm{1}}

\def\eps{{\epsilon}}

\def\va{{\bm{a}}}
\def\vb{{\bm{b}}}
\def\vc{{\bm{c}}}
\def\vd{{\bm{d}}}
\def\ve{{\bm{e}}}
\def\vf{{\bm{f}}}
\def\vg{{\bm{g}}}
\def\vh{{\bm{h}}}

\def\vl{{\bm{l}}}

\def\vs{{\bm{s}}}

\def\vu{{\bm{u}}}
\def\vv{{\bm{v}}}

\def\vx{{\bm{x}}}
\def\vy{{\bm{y}}}
\def\vz{{\bm{z}}}

\def\evc{{c}}

\def\mA{{\bm{A}}}

\def\mI{{\bm{I}}}

\def\mP{{\bm{P}}}
\def\mQ{{\bm{Q}}}

\def\mU{{\bm{U}}}
\def\mV{{\bm{V}}}
\def\mW{{\bm{W}}}

\DeclareMathAlphabet{\mathsfit}{\encodingdefault}{\sfdefault}{m}{sl}
\SetMathAlphabet{\mathsfit}{bold}{\encodingdefault}{\sfdefault}{bx}{n}

\def\gS{{\mathcal{S}}}

\newcommand{\R}{\mathbb{R}}

\DeclareMathOperator{\sign}{sign}

\usepackage{hyperref}
\usepackage[utf8]{inputenc} %
\usepackage{lipsum} %

\usepackage[T1]{fontenc}

\usepackage{etoolbox}
\newbool{includeappendix}
\setbool{includeappendix}{true}

\ifdefined\isoverfull
\else
\fi

\usepackage{color} %
\usepackage{xcolor} %

\definecolor{my-full-blue}{HTML}{1F77B4}

\definecolor{my-full-orange}{HTML}{FF7F0E}

\definecolor{my-full-green}{HTML}{2CA02C}

\definecolor{my-full-red}{HTML}{d62728}

\definecolor{my-full-purple}{HTML}{9467bd}

\colorlet{my-blue}{my-full-blue!30}
\colorlet{my-orange}{my-full-orange!30}
\colorlet{my-green}{my-full-green!30}
\colorlet{my-red}{my-full-red!30}
\colorlet{my-purple}{my-full-purple!30}

\usepackage{listings}

\usepackage{textcomp}

\usepackage{xcolor}

\usepackage[scaled=0.8]{beramono}

\definecolor{ckeyword}{HTML}{7F0055}
\definecolor{ccomment}{HTML}{3F7F5F}
\definecolor{cstring}{HTML}{2A0099}

\lstdefinestyle{numbers}{
	numbers=left,
	framexleftmargin=20pt,
	numberstyle=\tiny,
	firstnumber=auto,
	numbersep=1em,
	xleftmargin=2em
}

\lstdefinestyle{layout}{
	frame=none,
	captionpos=b,
}

\lstdefinestyle{comment-style}{
	morecomment=[l]//,
	morecomment=[s]{/*}{*/},
	commentstyle={\color{ccomment}\itshape},
}

\lstdefinestyle{string-style}{
	morestring=[b]",%
	morestring=[b]',%
	stringstyle={\color{cstring}},
	showstringspaces=false,%
}

\lstdefinestyle{keyword-style}{
	keywordstyle={\ttfamily\bfseries},
	morekeywords={
		function,
		constructor,
		int,
		bool,
		return,
		returns,
		uint
	},
	morekeywords = [2]{},
	keywordstyle = [2]{\text},
	sensitive=true,
}

\lstdefinestyle{input-encoding}{
	inputencoding=utf8,
	extendedchars=true,
	literate=
	{ℝ}{$\reals$}1%
	{→}{$\rightarrow$}1%
	{α}{$\alpha$}1%
	{β}{$\beta$}1%
	{λ}{$\lambda$}1%
	{θ}{$\theta$}1%
	{ϕ}{$\phi$}1%
}

\lstdefinestyle{escaping}{
	moredelim={**[is][\color{blue}]{\%}{\%}},
	escapechar=|,
	mathescape=true
}

\lstdefinestyle{default-style}{
	basicstyle=\fontencoding{T1}\ttfamily\footnotesize,
	style=numbers,
	style=layout,
	style=comment-style,
	style=string-style,
	style=keyword-style,
	style=input-encoding,
	style=escaping,
	tabsize=2,
	upquote=true
}

\lstdefinelanguage{BASIC}{
	language=C++,
	style=default-style
}[keywords,comments,strings]%

\usepackage[capitalize]{cleveref}

\crefname{listing}{Lst.}{listings}
\crefname{line}{Lin.}{Lin.}
\crefname{appendix}{App.}{App.}

\newcommand{\app}[1]{%
	\ifbool{includeappendix}{\cref{#1}}{the appendix}%
}
\newcommand{\App}[1]{%
	\ifbool{includeappendix}{\cref{#1}}{The appendix}%
}

\usepackage{amsthm}
\usepackage{dsfont}
\newtheorem{theorem}{Theorem}[section]
\crefname{theorem}{Theorem}{Theorems}

\crefname{problem}{Problem}{Problems}

\usepackage{paralist}
\usepackage{bbold}
\usepackage{threeparttable,booktabs}
\usepackage{makecell}
\usepackage{wrapfig}
\usepackage{multirow}
\usepackage{xspace}
\usepackage{xfrac}
\usepackage{subcaption}
\usepackage{listings}
\input{insbox}
\usepackage{enumitem}
\usepackage{centernot}

\usepackage[ruled,noend]{algorithm2e}
\usepackage{setspace}
\usepackage{thmtools} 
\usepackage{thm-restate}

\usepackage{tikz}
\usetikzlibrary{calc,decorations,decorations.pathmorphing}
\usetikzlibrary{positioning,fit,arrows}
\usetikzlibrary{decorations.markings}
\usetikzlibrary{shapes,shapes.geometric}
\usetikzlibrary{shadows,patterns,snakes}
\usetikzlibrary{backgrounds,decorations.pathreplacing,automata}
\usetikzlibrary{intersections}
\usetikzlibrary{angles,quotes}
\usetikzlibrary{plotmarks}

\usepackage{tkz-base}
\usepackage{aircraftshapes}

\DeclareMathOperator{\diag}{diag}
\DeclareMathOperator{\notimplies}{\centernot\implies}

\newcommand{\bb}[1]{\mathds{#1}}
\newcommand{\mbf}[1]{\mathbf{#1}}
\newcommand{\bc}[1]{\mathcal{#1}}

\newcommand{\ez}{\bm{\nu}}
\newcommand{\eb}{\bm{\eta}}
\newcommand{\eez}{\nu}

\newcommand{\0}{\mathbf{0}}
\newcommand{\Z}{\mathcal{Z}}
\newcommand{\Y}{\mathcal{Y}}

\newcommand{\aid}{\mathcal{A}}
\newcommand{\cid}{\mathcal{C}}

\providecommand{\U}{}
\renewcommand{\U}{\mathcal{U}}

\newcommand{\X}{\mathcal{X}}
\renewcommand{\S}{\mathcal{S}}

\renewcommand{\1}{\mathds{1}}
\renewcommand{\R}{\mathds{R}}
\newcommand{\norm}[1]{\left\lVert#1\right\rVert}

\newcommand{\tool}{\textsc{Craft}\xspace}
\newcommand{\toollong}{\textbf{C}onvex \textbf{R}elaxation \textbf{A}bstract \textbf{F}ixpoint i\textbf{T}eration\xspace}
\newcommand{\semi}{\textsc{SemiSDP}\xspace}

\newcommand{\convsm}{\texttt{ConvSmall}\xspace}

\newcommand{\fces}{\texttt{FCx87}\xspace}
\newcommand{\fcf}{\texttt{FCx40}\xspace}
\newcommand{\fco}{\texttt{FCx100}\xspace}
\newcommand{\fct}{\texttt{FCx200}\xspace}
\newcommand{\mnist}{MNIST\xspace}
\newcommand{\cifar}{CIFAR10\xspace}

\newcommand{\pre}{\ensuremath{\varphi}\relax\ifmmode\else\xspace\fi}
\newcommand{\post}{\ensuremath{\psi}\relax\ifmmode\else\xspace\fi}

\newcommand{\g}{\ensuremath{\vg_\alpha}\relax\ifmmode\else\xspace\fi}
\newcommand{\gSs}[1]{\ensuremath{\vg_{\alpha}^{\##1}}\relax\ifmmode\else\xspace\fi}
\newcommand{\gSss}[2]{\ensuremath{\vg_{\alpha_{#2}}^{#1\#,#2}}\relax\ifmmode\else\xspace\fi}
\newcommand{\gSsss}[2]{\ensuremath{\vg_{\alpha_{#2}}^{#1\#}}\relax\ifmmode\else\xspace\fi}
\renewcommand{\gS}{\gSs{}}

\newcommand{\domain}{CH-Zonotope\xspace}
\newcommand{\domainl}{\textbf{C}ontaining-\textbf{H}ybrid-Zonotope\xspace}
\newcommand{\fwdbwd}{FB\xspace}
\newcommand{\pr}{PR\xspace}
\newcommand{\gpr}{\vg_\alpha^{\pr}}
\newcommand{\gfwdbwd}{\vg_\alpha^{\fwdbwd}}
\newcommand{\gprS}{\ensuremath{\vg_{\alpha}^{\pr\#}}\relax\ifmmode\else\xspace\fi}
\newcommand{\gfwdbwdS}{\ensuremath{\vg_{\alpha}^{\fwdbwd\#}}\relax\ifmmode\else\xspace\fi}

\DeclareMathOperator{\con}{\texttt{contained}}
\DeclareMathOperator{\vol}{\texttt{vol}}
\DeclareMathOperator{\expand}{\texttt{expand}}
\DeclareMathOperator{\consolidate}{\texttt{consolidate}}

\DeclareMathOperator{\layer}{layer}

\crefname{algocf}{Algorithm}{Algorithms}
\Crefname{algocf}{Algorithm}{Algorithms}
\crefalias{AlgoLine}{line}
\crefname{line}{line}{lines}
\Crefname{line}{Line}{Lines}

\newenvironment{mtx}
{\left(\begin{smallmatrix}}{\end{smallmatrix}\right)}

\usepackage{pifont}%
\newcommand{\cmark}{\ding{51}}%
\newcommand{\xmark}{\ding{55}}%

\newbool{extended}

\setbool{includeappendix}{true}
\setbool{extended}{true}
\setcopyright{rightsretained}
\acmPrice{}
\acmDOI{10.1145/3591252}
\acmYear{2023}
\copyrightyear{2023}
\acmSubmissionID{pldi23main-p183-p}
\acmJournal{PACMPL}
\acmVolume{7}
\acmNumber{PLDI}
\acmMonth{6}

\AtBeginDocument{%
  }

\begin{document}

\title[Abstract Interpretation of Fixpoint Iterators with Applications to Neural Networks]{Abstract Interpretation of Fixpoint Iterators\\ with Applications to Neural Networks}

\author{Mark Niklas Müller}
\email{mark.mueller@inf.ethz.ch}
\orcid{0000-0002-2496-6542}
\author{Marc Fischer}
\email{marc.fischer@inf.ethz.ch}
\orcid{0000-0002-4157-1235}
\author{Robin Staab}
\email{robin.staab@inf.ethz.ch}
\orcid{0009-0009-9040-1214}
\author{Martin Vechev}
\email{martin.vechev@inf.ethz.ch}
\orcid{0000-0002-0054-9568}
\affiliation{%
  \institution{\newline Department of Computer Science, ETH Zurich}
  \streetaddress{Universitätsstrasse 6}
  \postcode{8092}
  \city{Zurich}
  \country{Switzerland}
}

\renewcommand{\shortauthors}{M. Müller, M. Fischer, R. Staab, and M. Vechev}
\renewcommand{\shortauthors}{Mark Niklas Müller, Marc Fischer, Robin Staab, and Martin Vechev}

\begin{abstract}
We present a new abstract interpretation framework for the precise over-approximation of numerical fixpoint iterators.%
Our key observation is that unlike in standard abstract interpretation (AI), typically used to over-approximate \emph{all} reachable program states, in this setting, one only needs to abstract the concrete fixpoints, i.e., the \emph{final} program states. Our framework targets numerical fixpoint iterators with convergence and uniqueness guarantees in the concrete and is based on two major technical contributions: (i) theoretical insights which allow us to compute sound and precise fixpoint abstractions \emph{without using joins}, and (ii) a new abstract domain, \domain, which admits efficient propagation and inclusion checks while retaining high precision.

We implement our framework in a tool called \tool and evaluate it on a novel fixpoint-based neural network architecture (monDEQ) that is particularly challenging to verify. Our extensive evaluation demonstrates that \tool exceeds the state-of-the-art performance in terms of speed (two orders of magnitude), scalability (one order of magnitude), and precision ($25\%$ higher certified accuracies).
\end{abstract}

\begin{CCSXML}
  <ccs2012>
  <concept>
  <concept_id>10003752.10010124.10010138.10011119</concept_id>
  <concept_desc>Theory of computation~Abstraction</concept_desc>
  <concept_significance>500</concept_significance>
  </concept>
  <concept>
  <concept_id>10003752.10010124.10010138.10010142</concept_id>
  <concept_desc>Theory of computation~Program verification</concept_desc>
  <concept_significance>500</concept_significance>
  </concept>
  <concept>
  <concept_id>10010147.10010257.10010293.10010294</concept_id>
  <concept_desc>Computing methodologies~Neural networks</concept_desc>
  <concept_significance>500</concept_significance>
  </concept>
  </ccs2012>
\end{CCSXML}

\ccsdesc[500]{Theory of computation~Abstraction}
\ccsdesc[500]{Theory of computation~Program verification}
\ccsdesc[500]{Computing methodologies~Neural networks}

\keywords{fixpoint, abstract interpretation, equlibrium models, adversarial robustness}  %

\maketitle

\ifbool{extended}{
\fancypagestyle{firstpagestyle}{%
\fancyfoot[R]{}%
}
\fancyfoot[RO,LE]{}
}{}

\section{Introduction} \label{sec:intro} Abstract interpretation (AI) \citep{CousotC77,CousotC92} is a popular static analysis technique, typically used to over-approximate all reachable states of a given program for a particular set of (potentially infinite) concrete inputs, captured by a pre-condition. Given an \emph{abstract domain} for representing abstract program states and \emph{abstract transformers} for capturing the effects of program statements on these abstract states, AI operates by starting with the pre-condition and applying the abstract transformers corresponding to each program statement until a so-called abstract (post-)fixpoint is reached, i.e., any further application of the abstract transformers does not change the computed abstraction. Under reasonable assumptions and in the absence of unbounded loops, this approach is guaranteed to terminate with a sound abstraction of \emph{all}  --- intermediate and final --- program states. To handle unbounded loops, special techniques such as Kleene iteration with widening and narrowing \citep{CousotC92b} are required to ensure termination. 

Interestingly, for an important class of programs with unbounded loops that themselves compute (concrete) fixpoints, e.g., numerical solvers, typically, only the resulting concrete fixpoints, i.e. the \emph{final state} of the concrete program rather than the intermediate states, are of interest. Using Kleene iteration in this setting, even with exact joins, leads to abstractions that include the union over all iteration states, making them inherently imprecise. A desirable goal then is to develop an abstract interpretation approach that targets only the precise abstraction of these final states.

\paragraph{This Work: Abstract Interpretation of Fixpoint Iterators} In this work, we introduce the first abstract interpretation framework, focusing on fixpoint iterators that possess convergence guarantees in the concrete. Our framework is based on two major contributions: (i) we present new theoretical insights which allow us to compute sound and precise fixpoint abstractions \emph{without using joins}. %
That is, we do not require Kleene iteration, typically used in AI to handle unbounded loops \citep{GangeNSSS13,putot2012static}, and further demonstrate that Kleene iteration is unsuitable for our class of programs. In addition, these insights enable us to further tighten the obtained abstractions by leveraging the convergence properties of the abstracted fixpoint iterator. 
While our method can be instantiated with any abstract domain, (ii) we introduce a novel abstract domain, called \domain, based on the Zonotope abstraction \citep{ZonotopeGhorbalGP09,DeepZSinghGMPV18}, combined with the notion of order-reduction \citep{KopetzkiSA17O,yang2018comparison}. Unlike Zonotope, our domain ensures constant representation size and allows for efficient yet precise inclusion checks -- only $\bc{O}(p^3)$ instead of $\bc{O}(p^6)$ in the dimension $p$ -- critical for handling fixpoint iterations.

We implement our framework in a tool called \tool and demonstrate its effectiveness on the robustness verification of monDEQs (Monotone Operator Deep Equilibrium Models) \citep{MonDEQWinstonK20}, a novel fixpoint-based neural architecture combining high-dimensionality and highly non-linear iterations, thus representing a particularly challenging class of fixpoint iterators. We remark that \tool can serve as a basis for future investigations of other fixpoint-based neural architectures such as stiff neural ODEs \citep{KimJDR2021} or SatNets \citep{WangDWK19}. %

\vspace{-1mm}
\paragraph{Main Contributions} Our core contributions are:
\vspace{-0.5mm}
\begin{itemize}
    \item A domain-specific abstract interpretation framework for high-dimensional fixpoint iterators with convergence guarantees in the concrete. %
	(\cref{sec:abstract-fp}).
	\item \domain, a novel abstract domain that enables both efficient abstract fixpoint iteration and inclusion checks (\cref{sec:m_zono}).
	\item \tool, a complete implementation of our framework and abstract domain (\cref{sec:tool}).
	\item An extensive evaluation demonstrating the effectiveness of \domain and showing that \tool achieves a new state-of-the-art for monDEQ verification.%
\end{itemize}

\vspace{-1.5mm}
\section{Overview}\label{sec:overview}
We now elaborate on the key challenges of abstracting fixpoint iterators and our approach to overcoming these. As a running example, we use monDEQs, a novel neural architecture based on high-dimensional fixpoint iterations and an instance of the class of programs we target. Thus, we begin with a short background on neural networks and their analysis.

\vspace{-1.5mm}
\paragraph{Neural Network Verification}
Given a neural network $\vh \colon \R^{d_{in}}$ $\mapsto \R^r$, a precondition $\pre(\vx)$, and postcondition $\post(\vh(\vx))$, the goal of neural network verification is to show that $\pre(\vx) \models \post(\vh(\vx))$. To this end, we construct a sound verifier to show $\pre(\vx) \vdash \post(\vh(\vx))$, i.e., that $\post(\vh(\vx))$ can be derived from $\pre(\vx)$, implying by the soundness of the verifier $\pre(\vx) \! \models \! \post(\vh(\vx))$, i.e., that $\pre(\vx)$ entails $\post(\vh(\vx))$.

A common instantiation of this problem is found in image classification. There, $\vx$ is an image, $\vh$ an image classifier, $\pre(\vx)$ an $\ell_p$-norm-ball around $\vx$, e.g., $\pre(\vx) \coloneqq \{ \vx' \in \R^{d_{in}} \mid \| \vx - \vx' \|_\infty \leq \eps \}$, $\post(\vh(\vx))$ denotes classification to the correct class, and showing $\pre(\vx) \models \post(\vh(\vx))$ formally verifies robustness to adversarial examples \citep{szegedy2013intriguing,GoodfellowSS14}.

A popular approach to constructing neural network verifiers is to adapt abstract interpretation techniques to handle hundreds of thousands of variables \citep{GehrMDTCV18,DeepZSinghGMPV18,SinghGPV19}. There, the precondition $\pre(\vx)$ is encoded as an abstract element and propagated through the network layer-by-layer using abstract transformers before the resulting abstraction of the network output is checked against the postcondition $\post(\vh(\vx))$.

\begin{wrapfigure}[10]{r}{0.42 \textwidth}
	\vspace{-2.mm}
	\begin{minipage}[t]{0.35\linewidth}
	\begin{lstlisting}[language=Python,mathescape=true,numbers=none,xleftmargin=-2mm]
	def NN($\vx$):
	$\vs_1 = \layer_1(\vx)$
	$\vdots$
	$\vs_{L-1} = \layer_{L-1}(\vs_{L-2})$
	$\vy = \layer_L(\vs_{L-1})$
	return $\vy$
	\end{lstlisting}
	\end{minipage}
	\hfill
	\begin{minipage}[t]{0.60\linewidth}
	\begin{lstlisting}[language=Python,mathescape=true,numbers=none,xleftmargin=0mm]
	def monDEQ($\vx$):
		$\vs_0 = \mathbf{0}$, $i = 0$
		while not converged($s_i$):
			$i = i + 1$
			$\vs_{i} = \vg(\vx, \vs_{i-1})$
		$\vy = \layer_{\vy}(\vs_i)$
		return $\vy$
	\end{lstlisting}
	\end{minipage}
	\vspace{-5mm}
	\caption{Pseudocode for a standard neural network (left) and a monDEQ (right).}
	\label{fig:nets}
\end{wrapfigure}

\vspace{-1mm}
\paragraph{Fixpoint-based Neural Networks} Neural architectures based on fixpoint computations such as monDEQs (formally discussed in \cref{sec:mondeq}), however, do not simply apply a fixed number of layers, instead iteratively applying an iterator in an unbounded loop until a fixpoint is reached. We highlight this difference in \cref{fig:nets}, where we contrast pseudocode for a standard feed-forward neural network (left) and a monDEQ (right).
Let us consider an example monDEQ classifier $\vh \colon [-1, 1]^2 \mapsto \{0, 1\}$:
\begin{equation}
	\vg(\vx, \vs) \coloneqq ReLU\left( \tfrac{1}{10}  \begin{mtx}5 & -1\\1 & \phantom{+}5\end{mtx} \vs + \tfrac{1}{10}  \begin{mtx}\phantom{+}1 & 1\\-1 & 1\end{mtx} \vx \right) 
	\qquad\qquad
	\layer_{\vy}(\vs) \coloneqq \begin{mtx}1 & -1 \end{mtx} \vs,
	\label{eq:example}
\end{equation}
returning class 1 if $\vy(\vs^*) \coloneqq \vh(\vx) = \layer_{\vy}(\vs^*) > 0$ and else class 0, where $\vs^* = \vg(\vx, \vs^*)$ denotes the fixpoint found by iterating $\vg(\vx, \vs)$.
Given an example input $\vx \coloneqq \begin{mtx}0.2 \\ 0.5\end{mtx} = \tfrac{1}{10} \begin{mtx}2 \\ 5\end{mtx}$, we compute the fixpoint $\vs^*$ by iteratively applying $\vg(\vx, \vs_i)$. We initialize the iteration with $\vs_0 = \begin{mtx} 0 \\ 0\end{mtx}$ and obtain
\begin{align*}	
	\vs_{i+1} &\coloneqq \vg(\vx, \vs_i) = ReLU\left( \tfrac{1}{10}  \begin{mtx}5 & -1\\1 & \phantom{+}5\end{mtx} \vs_i + \tfrac{1}{100} \begin{mtx}\phantom{+}1 & 1\\-1 & 1\end{mtx}  \begin{mtx}2\\5\end{mtx} \right)
		= ReLU\left( \tfrac{1}{10} \begin{mtx}5 & -1\\1 & \phantom{+}5\end{mtx}  \vs_i + \tfrac{1}{100} \begin{mtx}7\\3\end{mtx} \right)\\
	\vs_1 &\coloneqq ReLU\left( \tfrac{1}{10} \begin{mtx}0\\0\end{mtx} + \tfrac{1}{100} \begin{mtx}7\\3\end{mtx} \right) =  \tfrac{1}{100} \begin{mtx}7\\3\end{mtx} \qquad\qquad\qquad \|\vs_1-\vs_0\| = 0.0762\\
	\vs_2 &\coloneqq
	ReLU\left( \tfrac{1}{1000} \begin{mtx}32\\22\end{mtx} + \tfrac{1}{100} \begin{mtx}7\\3\end{mtx} \right) =
	\tfrac{1}{1000} \begin{mtx}102\\52\end{mtx} \qquad\qquad \|\vs_2-\vs_1\| = 0.0389\\
	\vs^* &\approx \begin{mtx}0.1231\\0.0846\end{mtx}.
\end{align*}
We observe that the residual $\|\vs_{i+1}-\vs_i\|$ decreases quickly as we converge towards the fixpoint $\vs^*$ and note that this convergence to a unique fixpoint is guaranteed for monDEQs \citep{MonDEQWinstonK20}. We can thus compute the fixpoint $\vs^*$ to arbitrary precision, only depending on the termination condition, \lstinline{converged($\vs_i$)} in \cref{fig:nets}. In our example, we obtain $\vy(\vs^*) \approx 0.0385  > 0$ and thus return class $1$.
In \cref{fig:example}, we illustrate this inference process, showing the decision landscape of $\vh$ on $[-1, 1]^2$ (\cref{fig:example:class}), the obtained fixpoints (\cref{fig:example:abstract}), and the resulting output (\cref{fig:example:y}). We highlight the points corresponding to our example input $\vx$ with a red {\color{my-full-red} $\times$} and will explain the orange and purple regions shortly.
In the following examples, we assume that \lstinline{converged($\vs_i$)} is chosen such that we reach the true fixpoints up to machine precision.

\vspace{-1mm}
\subsection{Abstract Interpretation for Fixpoint Iterators: A Motivation}\label{sec:overview:motivateframework}
While the construction of abstract interpretation based verifiers for loop-free programs such as feed-forward networks (left in \cref{fig:nets}) is conceptually straightforward, fixpoint iterators such as monDEQs present a greater challenge due to their unbounded loops (right in \cref{fig:nets}). To motivate the need for a domain-specific abstraction framework, we will first illustrate that generic abstract interpretation techniques are inherently not suitable for this task due to three fundamental reasons: (i) the analysis of fixpoint-iterators requires only the \emph{last} iteration state, containing the concrete fixpoints, instead of all intermediate iteration states, to be abstracted, (ii) while in general abstract interpretation, an abstract transformer of the termination condition has to be evaluated to refine the obtained abstract state, fixpoint iterators allow the mathematical invariants that are enforced by the termination condition to be leveraged directly, leading to much more precise results, and finally, (iii) standard techniques do not take advantage of the key convergence properties of the concrete fixpoint iterator $\vg$, which we leverage in order to drastically improve precision.

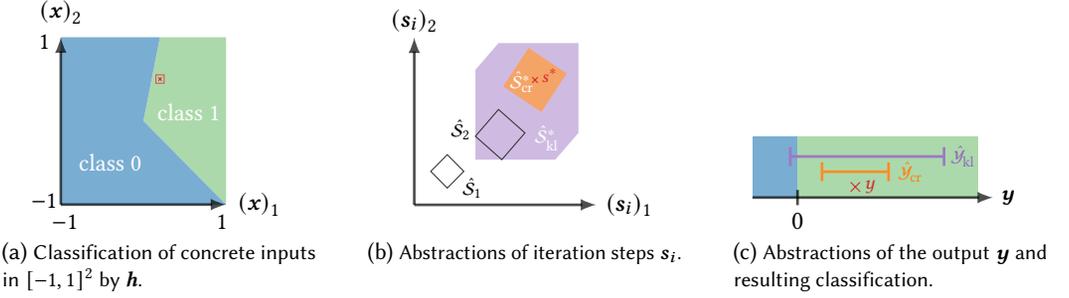
\begin{figure}[t]
\center
\begin{subfigure}[t]{0.3\textwidth}
	\center
	\begin{tikzpicture}[scale=1.1]
\tikzset{>=latex}
	\coordinate (class1_p_0) at ({1.00},{-1.00});
	\coordinate (class1_p_1) at ({1.00},{1.00});
	\coordinate (class1_p_2) at ({0.20},{1.00});
	\coordinate (class1_p_3) at ({0.00},{0.00});
	\fill [fill=my-full-green!40, rounded corners=0mm] (class1_p_0) -- (class1_p_1) -- (class1_p_2) -- (class1_p_3) -- cycle;

	\coordinate (class0_p_0) at ({-1.00},{-1.00});
	\coordinate (class0_p_1) at ({-1.00},{1.00});
	\coordinate (class0_p_2) at ({0.20},{1.00});
	\coordinate (class0_p_3) at ({0.00},{0.00});
	\coordinate (class0_p_4) at ({1.00},{-1.00});
		\fill [fill=my-full-blue!60, rounded corners=0mm] (class0_p_0) -- (class0_p_1) -- (class0_p_2) -- (class0_p_3) -- (class0_p_4) -- cycle;

	\coordinate (x) at ({0.2},{0.5});
	\coordinate (dim) at (0.1,0.1);
	\coordinate (LL) at ($(x)-0.5*(dim)$);
	\coordinate (UR) at ($(x)+0.5*(dim)$);
	\node ()[mark size=0.9pt, anchor=center, color=my-full-red] at (x) {\pgfuseplotmark{x}};
	\draw [draw=my-full-red!80, rounded corners=0mm] ($(LL)$) rectangle ($(UR)$);
	\node ()[font=\small, anchor=center, color=white] at (-0.4, -0.5) {class 0};
	\node ()[font=\small, anchor=center, color=white] at (0.55, 0.1) {class 1};
	\draw[->,opacity=0.7,line width=1] (-1, -1) -- (1.0, -1) ;
	\draw[->,opacity=0.7,line width=1] (-1, -1) -- (-1, 1.0) ;

	\node ()[font=\small, anchor=center, color=black] at (1.4, -1.0) {$\left( \vx \right)_1$};
	\node ()[font=\small, anchor=center, color=black] at (-1.0, 1.3) {$\left( \vx \right)_2$};
	
	\node ()[font=\small, anchor=center, color=black] at (-1.2, -0.95) {$-1$};
	\node ()[font=\small, anchor=center, color=black] at (-0.95, -1.2) {$-1$};
	\node ()[font=\small, anchor=center, color=black] at (-1.2, +0.95) {$1$};
	\node ()[font=\small, anchor=center, color=black] at (+0.95, -1.2) {$1$};

	\end{tikzpicture}
	\vspace{-2mm}
	\caption{\footnotesize Classification of concrete inputs in $[-1, 1]^2$ by $\vh$.} \label{fig:example:class}
\end{subfigure} 
\hfill
\begin{subfigure}[t]{0.3\textwidth}
	\center
\begin{tikzpicture}[scale=22]
\tikzset{>=latex}
	\coordinate (initial_head) at ({0.0700},{0.0300});
	\coordinate (initial_p_0) at ({0.0600},{0.0300});
	\coordinate (initial_p_1) at ({0.0700},{0.0200});
	\coordinate (initial_p_2) at ({0.0800},{0.0300});
	\coordinate (initial_p_3) at ({0.0700},{0.0400});
	\node[] at (0.1,-0.0043) {};
	\draw [color=black!80, rounded corners=0mm] (initial_p_0) -- (initial_p_1) -- (initial_p_2) -- (initial_p_3)  -- cycle;

	\coordinate (fp_kleene_head) at ({0.1182},{0.0718});
	\coordinate (fp_kleene_p_0) at ({0.1494},{0.0530});
	\coordinate (fp_kleene_p_1) at ({0.1494},{0.1067});
	\coordinate (fp_kleene_p_2) at ({0.1010},{0.1067});
	\coordinate (fp_kleene_p_3) at ({0.0870},{0.0907});
	\coordinate (fp_kleene_p_4) at ({0.0870},{0.0370});
	\coordinate (fp_kleene_p_5) at ({0.1354},{0.0370});
	\fill [fill=my-full-purple!55, opacity=0.8, rounded corners=0mm] (fp_kleene_p_0) -- (fp_kleene_p_1) -- (fp_kleene_p_2) -- (fp_kleene_p_3) -- (fp_kleene_p_4) -- (fp_kleene_p_5)  -- cycle;

	\coordinate (fp_craft_head) at ({0.1231},{0.0846});
	\coordinate (fp_craft_p_0) at ({0.1038},{0.0808});
	\coordinate (fp_craft_p_1) at ({0.1269},{0.0654});
	\coordinate (fp_craft_p_2) at ({0.1423},{0.0885});
	\coordinate (fp_craft_p_3) at ({0.1192},{0.1038});
	\fill [fill=my-full-orange!75, opacity=0.8, rounded corners=0mm] (fp_craft_p_0) -- (fp_craft_p_1) -- (fp_craft_p_2) -- (fp_craft_p_3)  -- cycle;

	\coordinate (first_head) at ({0.1020},{0.0520});
	\coordinate (first_p_0) at ({0.0870},{0.0510});
	\coordinate (first_p_1) at ({0.1030},{0.0370});
	\coordinate (first_p_2) at ({0.1170},{0.0530});
	\coordinate (first_p_3) at ({0.1010},{0.0670});
	\draw [color=black!80, rounded corners=0mm] (first_p_0) -- (first_p_1) -- (first_p_2) -- (first_p_3)  -- cycle;

	\draw[->,opacity=0.7,line width=1] (0.05, 0.01) -- (0.16, 0.01) ;
	\draw[->,opacity=0.7,line width=1] (0.05, 0.01) -- (0.05, 0.11) ;

	\node ()[font=\small, anchor=center, color=black] at (0.18, 0.01) {$\left( \vs_i \right)_1$};
	\node ()[font=\small, anchor=center, color=black] at (0.05, 0.12) {$\left( \vs_i \right)_2$};
	
	\node ()[font=\tiny, anchor=center, color=black] at (0.085, 0.02) {$\hat{\S}_1$};
	\node ()[font=\tiny, anchor=center, color=black] at (0.078, 0.055) {$\hat{\S}_2$};
	\node ()[font=\tiny, anchor=center, color=my-full-red] at (0.132, 0.0875) {$s^*$};
	\node ()[font=\tiny, anchor=center, color=white] at (0.13, 0.05) {$\hat{\S}^*_{\text{kl}}$};

	\node ()[font=\tiny, anchor=center, color=white] at (0.115, 0.0835) {$\hat{\S}^*_{\text{cr}}$};

	\coordinate (z) at ({0.1231},{0.0846});
	\node ()[mark size=1.5pt, anchor=center, color=my-full-red] at (z) {\pgfuseplotmark{x}};

\end{tikzpicture}
\vspace{-2mm}
	\caption{\footnotesize Abstractions of iteration steps $\vs_i$.} \label{fig:example:abstract}
\end{subfigure}
\hfill
\begin{subfigure}[t]{0.3\textwidth}
	\center
\begin{tikzpicture}[scale=20.0]
\tikzset{>=latex}

	\coordinate (class0_p_0) at ({-0.03},{0.01});
	\coordinate (class0_p_1) at ({0.00},{0.01});
	\coordinate (class0_p_2) at ({0.00},{0.05});
	\coordinate (class0_p_3) at ({-0.03},{0.05});
	\fill [fill=my-full-blue!60, rounded corners=0mm] (class0_p_0) -- (class0_p_1) -- (class0_p_2) -- (class0_p_3)  -- cycle;

	\coordinate (class1_p_0) at ({0.12},{0.01});
	\coordinate (class1_p_1) at ({0.00},{0.01});
	\coordinate (class1_p_2) at ({0.00},{0.05});
	\coordinate (class1_p_3) at ({0.12},{0.05});
	\fill [fill=my-full-green!40, rounded corners=0mm] (class1_p_0) -- (class1_p_1) -- (class1_p_2) -- (class1_p_3)  -- cycle;

	\draw[->,opacity=0.7,line width=1] (-0.03, 0.01) -- (0.13, 0.01) ;
	\draw[-,opacity=0.7,line width=1] (0.0, 0.005) -- (0.0, 0.015) ;
	\coordinate (zero) at ({0.0},{-0.005});
	\node ()[mark size=3pt, anchor=center, color=black] at (zero) {$0$};

	\coordinate (fp_kleene_p_0) at ({-0.0057},{0.037});
	\coordinate (fp_kleene_p_1) at ({0.0984},{0.037});
	\draw[|-|,line width=1, color=my-full-purple!90] (fp_kleene_p_0) -- (fp_kleene_p_1);

	\coordinate (fp_craft_p_0) at ({0.0154},{0.027});
	\coordinate (fp_craft_p_1) at ({0.0615},{0.027});
	\draw[|-|,line width=1,color=my-full-orange!90] (fp_craft_p_0) -- (fp_craft_p_1);

	\coordinate (y) at ({0.0385},{0.017});
	\node ()[mark size=2.5pt, anchor=center, color=my-full-red] at (y) {\pgfuseplotmark{x}};

	\node ()[font=\footnotesize, anchor=center, color=black] at (0.14, 0.01) {$\vy$};

	\node ()[font=\tiny, anchor=center, color=my-full-purple] at (0.11, 0.037) {$\hat{\Y}_{\text{kl}}$};
	\node ()[font=\tiny, anchor=center, color=my-full-orange] at (0.075, 0.027) {$\hat{\Y}_{\text{cr}}$};
	\node ()[font=\footnotesize, anchor=center, color=my-full-red, scale=0.9] at (0.048, 0.017) {$y$};

\end{tikzpicture}
\vspace{-2mm}
	\caption{\footnotesize Abstractions of the output $\vy$ and resulting classification.}  \label{fig:example:y}
\end{subfigure}

\vspace{-2mm}
\caption{Example Visualization: The concrete input $\vx$ (red {\color{my-full-red} $\times$}) yields the fixpoint $\vs^*(\vx)$ ({\color{my-full-red} $\times$}), and prediction $y(\vx)$ ({\color{my-full-red} $\times$}). Propagating the input region $\X$ (red {\color{my-full-red} $\square$}) with Kleene iteration and \tool yields the abstract iteration steps $\hat{\S}_1$ and $\hat{\S}_2$ and finally the fixpoint $\hat{\S}^*_{\text{kl}}$ and $\hat{\S}^*_{\text{cr}}$ and corresponding output abstraction $\hat{\Y}_{\text{kl}}$ and $\hat{\Y}_{\text{cr}}$, respectively.}
\label{fig:example}
\vspace{-1mm}
\end{figure}

\subsection{Challenge: Precise Loop Abstraction}\label{sec:overview:loop}
Abstract Interpretation (AI) \citep{CousotC77,CousotC92,cousot1979constructive} is an analysis technique that allows reasoning over the behavior of programs for sets of inputs. Conceptually, a set of program inputs, e.g., those specified by the precondition $\pre$, is represented symbolically and then propagated through the program to determine whether the postcondition $\post$ is satisfied for all these inputs.

Formally, we over-approximate sets of concrete inputs from domain $\cid$ with abstract elements from an abstract domain $\aid$. To retrieve the set of concrete values $\gamma(\hat{\S})$ represented by an abstract element $\hat{\S} \in \aid$, we define the concretization function $\gamma: \aid \mapsto \wp(\cid)$. %
Equipped with a partial order $\sqsubseteq$, an abstract domain forms a poset such that $\hat{\S}_1 \sqsubseteq \hat{\S}_2 \implies \gamma(\hat{\S}_1) \subseteq \gamma(\hat{\S}_2)$. This allows us to define the (quasi \citep{GangeNSSS13}) join $\hat{\S}_1 \sqcup \hat{\S}_2$ of two abstract elements $\hat{\S}_1, \hat{\S}_2$ as their least (any) upper bound with respect to $\sqsubseteq$. Importantly, if the join exists, we have $\gamma(\hat{\S}_1) \cup \gamma(\hat{\S}_2) \subseteq \gamma(\hat{\S}_1 \sqcup \hat{\S}_2)$.
We capture the effect of a concrete function $f \colon \cid \to \cid$, e.g., a program statement, on an abstract element $\hat{\S} \in \aid$, using a sound abstract transformer $f^{\#} \colon \aid \to \aid$ such that  $\forall \vs \in \gamma(\hat{\S}),\, f(\vs) \in \gamma(f^{\#}(\hat{\S}))$.

\begin{wrapfigure}[8]{r}{0.22 \textwidth}
\vspace{-6mm}
\centering
\begin{lstlisting}[language=Python,mathescape=true,numbers=none]
$\vx = \dots$
if check($\vx$):
	$\vy = f_a(\vx)$
else:
	$\vy = f_b(\vx)$
return $\vy$
\end{lstlisting}
\vspace{-5mm}
\caption{Example branching behavior.}
\label{fig:if}
\vspace{-3mm}
\end{wrapfigure}

\paragraph{Standard Abstract Post-Fixpoint Computation}
In AI, there are generally two possible outcomes when control flow, such as the if-statement in \cref{fig:if}, is encountered: (i) either the abstract state allows to show that the same branch is taken for all abstracted values and we only have to consider the abstract transformer of that branch, e.g., $f_a^{\#}$ to obtain $\hat{\Y} = f_a^{\#}(\hat{\X})$, or (ii) if we can not rule out either branch, both branches have to be considered and we obtain the join of the resulting abstract states $\hat{\Y} = f_a^{\#}(\hat{\X}) \sqcup f_b^{\#}(\hat{\X})$.
Most control flow, including bounded loops, can be handled in this way.

Unbounded loops, such as those encountered in monDEQs, however, present a special challenge as the above approach will typically not terminate. A common solution to this problem is the so-called \emph{Kleene iteration}. For a loop of the form \lstinline{while condition($\vs$): $\vs = f(\vs)$}, Kleene iteration computes $\hat{\S}_i = \hat{\S}_{i-1} \sqcup f^{\#}(\hat{\S}_{i-1})$ until convergence or formally until an order-theoretic (post)-fixpoint $\hat{\S}^* \coloneqq \hat{\S}_{i-1} \! \sqsupseteq \! \hat{\S}_{i}$ is reached.
In practice, widening \citep{CousotC92b} is often required for Kleene iteration to terminate. Unfortunately, the obtained precision is heavily dependent on the existence of precise abstract transformers for the termination condition. If we lack such transformers, e.g., due to complex, non-linear, non-convex termination conditions, the obtained abstraction is often imprecise. To recover some precision, we apply \emph{semantic unrolling} \citep{BlanchetCCFMMMR02}, i.e., unroll the first $k$ loop iterations for which we can show that the termination condition is not satisfied and thus iterate $\hat{\S}_i = \vf^{\#}(\hat{\S}_{i-1})$ for $i \leq k$, avoiding the join $\hat{\S}_i = \hat{\S}_{i-1} \sqcup \vf^{\#}(\hat{\S}_{i-1})$. %

\paragraph{Example (cont.)}
Let us apply Kleene iteration to our example to illustrate these imprecision issues. For the monDEQ in \cref{eq:example}, we let $\X = \pre(\vx) = \{\vx + \delta \mid \delta \in \R^2, \| \delta \|_\infty \leq 0.05 \}$ denote a small region around $\vx$ (red {\color{my-full-red} $\square$} in \cref{fig:example:class}) and $\post$ the classification to class $1$. We initialize $\hat{\S}_0$ such that $ \gamma(\hat{\S}_0) = \{\0\}$ and apply Kleene iteration with semantic unrolling ($k=2$) to the abstraction $\vg^\#$ of the iterator $\vg$ (\cref{eq:example}) using the Zonotope domain \citep{DeepZSinghGMPV18} to compute an abstract post-fixpoint of the loop. We illustrate the intermediate states $\hat{\S}_1$ and $\hat{\S}_2$ as well as the final fixpoint $\hat{\S}^*_{\text{kl}}$ (purple) in \cref{fig:example:abstract}. Note how the second state $\hat{\S}_2$ is included in the post-fixpoint $\hat{\S}^*_{\text{kl}}$. Applying the classification layer, we obtain $\hat{\Y}_{\text{kl}} = \vy^{\#}(\hat{\S}^*_{\text{kl}})$ (purple interval in \cref{fig:example:y}). As the interval contains $0$, we are unable to verify the postcondition that all points in $\X$ get classified to class 1 (even though they do).

\paragraph{Domain-Specific Abstraction of Fixpoint Iterators}
We now propose a domain-specific approach targeting the abstraction of fixpoint iterators. Recall that in \cref{sec:overview:motivateframework} we discussed three reasons for the imprecision of the standard approach, which we address as follows. First, in our setting, we note that it is sufficient to abstract the set containing all concrete fixpoints $\S^* = \{\vs^* \mid \vs^* = g(\vx, \vs^*), \vx \in \X\}$ instead of the union of all iteration states arriving at the loop head. Second, instead of requiring an abstract transformer for the termination condition  \lstinline{converged($\vs_i$)}, we leverage the mathematical invariants enforced by the condition, namely that a fixpoint has been reached, to show that it suffices to iterate $\hat{\S}_i = \vg^{\#}(\hat{\X}, \hat{\S}_{i-1})$ -- \emph{without the use of joins} and \emph{without requiring abstract transformers for the termination condition} -- until we reach containment ($\hat{\S}_{i-1} \sqsupseteq \hat{\S}_{i}$), to guarantee that $\gamma(\hat{\S}_{i})$ contains all concrete fixpoints $\S^*$. Finally, we leverage the convergence properties of $\vg$ and prove that an additional $j>0$ applications of $\vg^\#$ to $\hat{\S}_{i}$ yield abstractions which may not be included in $\hat{\S}_{i}$, yet are always sound, i.e., contain all concrete fixpoints $\S^*$, while being empirically much tighter than $\hat{\S}_{i}$. We refer to this property of $\vg^\#$ as \emph{fixpoint set preservation}. Finally, we leverage these insights to perform precise verification of monDEQs by computing $\hat{\Y}$ from the resulting over-approximation $\hat{\S}_{i+j}$ and then checking $\post(\hat{\Y})$.

\paragraph{Example (cont.)}
Let us continue our example:
We initialize our iteration as for Kleene but iterate $\hat{\S}_i = \vg^{\#}(\hat{\X}, \hat{\S}_{i-1})$ until we find $\hat{\S}_{i} \sqsubseteq \hat{\S}_{i-1}$, again visualizing the iteration in \cref{fig:example:abstract}.
After sharing $\hat{\S}_1$ and $\hat{\S}_2$ with the (unrolled) Kleene iteration, we, in contrast to Kleene iteration, do not have to compute the join over iteration states and thus reach the much more precise abstraction $\hat{\S}^{*}_{\text{cr}}$ (orange). Indeed, applying the last layer yields a much more precise $\hat{\Y}_{\text{cr}} = \vy^{\#}(\hat{\S}^*_{\text{cr}})$, allowing us to show that $y > 0$ and thus certify that all inputs in the red region are indeed classified to class 1.

\paragraph{Summary: Domain-Specific vs. Standard Abstractions}
To summarize, as our setting is motivated by computing fixpoint set over-approximations, we are only interested in the \emph{final} state obtained by the iterator and not in the intermediate program states. This difference in the objective combined with the mathematical properties of the iterator and termination condition allows our domain-specific approach to compute much tighter fixpoint approximations than standard AI.%

\subsection{Challenge: Efficient Computation}\label{sec:overview:efficiency}

\begin{wraptable}[7]{r}{0.465 \textwidth}
	\centering
	\vspace{-8.5mm}
	\caption{Comparison of \domain to other abstract domains, for fixpoint abstraction.}
	\vspace{-4mm}
	\renewcommand{\arraystretch}{1.0}
	\label{tab:relaxations_intro}
	\scalebox{0.82}{
		\begin{tabular}{@{}lccc@{}}
			\toprule
			& Iteration & Inclusion & Precision\\
			\midrule
			Box & \cmark & \cmark & \xmark \\
			(Hybrid) Zonotope & \xmark & \xmark & \cmark\\
			Polyhedra & \xmark & \xmark & \textbf{?}\\
			\midrule
			\domain & \cmark & \cmark & \cmark\\
			\bottomrule
		\end{tabular}
	}
	\vspace{-3mm}
\end{wraptable}

\paragraph{Choice of Abstract Domain}\label{sec:other_domains}
While our abstraction framework for fixpoint iterators is in itself a compelling result, it comes with another challenge, namely, the need for an abstract domain that satisfies all of the following criteria: (i) efficient propagation through the iterator $\vg$, (ii) efficient inclusion checks in high dimensions, \emph{and} (iii) high precision.

To motivate this challenge, we first discuss why existing abstract domains used in neural network verification, shown in \cref{tab:relaxations_intro}, are unable to satisfy these criteria. Consider $L$ iterations of the iterator $\vg$ with a latent variable $\vs \in \R^p$ of dimension $p$. %

The \emph{Box} abstraction \citep{GehrMDTCV18,DiffAIMirmanGV18,GowalDSBQUAMK18} is the simplest commonly used abstract domain. Due to its constant-size representation, it can be efficiently propagated ($\bc{O}(Lp^2)$) and permits fast $\bc{O}(p)$ inclusion checks. However, as demonstrated in \cref{sec:eval:ch_zono}, it loses too much precision to be practically effective.

\emph{(Hybrid) Zonotopes} \citep{GehrMDTCV18,DiffAIMirmanGV18,WongK18,DeepZSinghGMPV18} allow for more precision, at the cost of a growing representation size, increasing the propagation cost to $\bc{O}(L^2 p^3)$. Further, exact inclusion checks are known to be co-NP-complete \citep{Kulmburg2021OnTC} and even approximate ones \cite{SadraddiniT19} (between $\bc{O}(p^6)$ and $\bc{O}(L^3p^6)$) become intractable in high ($\gtrsim50$) dimensions. %

\emph{(Restricted) Polyhedra} based methods that propagate linear bounds \citep{zhang2018crown,SinghGPV19} are state-of-the-art for applications where runtime is critical \citep{gpupoly}. These methods are typically more precise than Zonotope and have identical time complexity ($\bc{O}(L^2 p^3)$). However, as they yield polyhedra in the input-output space of the abstracted program, the input dimensions have to first be projected out to perform inclusion checks in the output space. While the inclusion checks themselves have polynomial complexity \citep{SadraddiniT19}, the projection step is co-NP-hard \citep{kellner2015containment}, making the overall check intractable.

\paragraph{The \domain Domain}
To address the above challenges, we introduce the \domain domain in \cref{sec:m_zono}. It builds on Hybrid Zontopes \citep{DiffAIMirmanGV18}, allows for an efficient inclusion check ($\bc{O}(p^3)$), and, thanks to the strategic use of order reduction \citep{KopetzkiSA17O}, ensures constant representation size, allowing for fast and efficient propagation ($\bc{O}(L p^3)$).

\begin{wrapfigure}[10]{r}{0.33 \textwidth}
\vspace{-5.5mm}
\centering
\begin{tikzpicture}[scale=24]
\tikzset{>=latex}
	\coordinate (fp_craft_head) at ({0.1231},{0.0846});
	\coordinate (fp_craft_p_0) at ({0.1038},{0.0808});
	\coordinate (fp_craft_p_1) at ({0.1269},{0.0654});
	\coordinate (fp_craft_p_2) at ({0.1423},{0.0885});
	\coordinate (fp_craft_p_3) at ({0.1192},{0.1038});

	\coordinate (fp_craft_cont-2_head) at ({0.1236},{0.0821});
	\coordinate (fp_craft_cont-2_p_0) at ({0.1242},{0.1145});
	\coordinate (fp_craft_cont-2_p_1) at ({0.0912},{0.0826});
	\coordinate (fp_craft_cont-2_p_2) at ({0.1231},{0.0496});
	\coordinate (fp_craft_cont-2_p_3) at ({0.1561},{0.0815});

	\coordinate (fp_craft_cont-1_head) at ({0.1236},{0.0834});
	\coordinate (fp_craft_cont-1_p_0) at ({0.1106},{0.0997});
	\coordinate (fp_craft_cont-1_p_1) at ({0.0974},{0.0804});
	\coordinate (fp_craft_cont-1_p_2) at ({0.1074},{0.0704});
	\coordinate (fp_craft_cont-1_p_3) at ({0.1266},{0.0571});
	\coordinate (fp_craft_cont-1_p_4) at ({0.1366},{0.0671});
	\coordinate (fp_craft_cont-1_p_5) at ({0.1499},{0.0864});
	\coordinate (fp_craft_cont-1_p_6) at ({0.1399},{0.0964});
	\coordinate (fp_craft_cont-1_p_7) at ({0.1206},{0.1097});
	
	\coordinate (fp_kleene_head) at ({0.1182},{0.0718});
	\coordinate (fp_kleene_p_0) at ({0.1494},{0.0530});
	\coordinate (fp_kleene_p_1) at ({0.1494},{0.1067});
	\coordinate (fp_kleene_p_2) at ({0.1010},{0.1067});
	\coordinate (fp_kleene_p_3) at ({0.0870},{0.0907});
	\coordinate (fp_kleene_p_4) at ({0.0870},{0.0370});
	\coordinate (fp_kleene_p_5) at ({0.1354},{0.0370});
	\draw [style=dashed, opacity=0.5] (fp_kleene_p_0) -- (fp_kleene_p_1) -- (fp_kleene_p_2) -- (fp_kleene_p_3) -- (fp_kleene_p_4) -- (fp_kleene_p_5)  -- cycle;

	\fill [fill=my-full-green!80, opacity=0.8, rounded corners=0mm] (fp_craft_cont-2_p_0) -- (fp_craft_cont-2_p_1) -- (fp_craft_cont-2_p_2) -- (fp_craft_cont-2_p_3)  -- cycle;

	\fill [fill=my-full-blue!80, opacity=0.8, rounded corners=0mm] (fp_craft_cont-1_p_0) -- (fp_craft_cont-1_p_1) -- (fp_craft_cont-1_p_2) -- (fp_craft_cont-1_p_3) -- (fp_craft_cont-1_p_4) -- (fp_craft_cont-1_p_5) -- (fp_craft_cont-1_p_6) -- (fp_craft_cont-1_p_7)  -- cycle;

	\fill [fill=my-full-orange!50, opacity=0.8, rounded corners=0mm] (fp_craft_p_0) -- (fp_craft_p_1) -- (fp_craft_p_2) -- (fp_craft_p_3)  -- cycle;

	\coordinate (done_head) at ({0.1231},{0.0846});
	\coordinate (done_p_0) at ({0.1038},{0.0808});
	\coordinate (done_p_1) at ({0.1269},{0.0654});
	\coordinate (done_p_2) at ({0.1423},{0.0885});
	\coordinate (done_p_3) at ({0.1192},{0.1038});

	\coordinate (1_head) at ({0.0700},{0.0300});
	\coordinate (1_p_0) at ({0.0600},{0.0200});
	\coordinate (1_p_1) at ({0.0800},{0.0200});
	\coordinate (1_p_2) at ({0.0800},{0.0400});
	\coordinate (1_p_3) at ({0.0600},{0.0400});
	
	\draw [color=black!80, rounded corners=0mm] (1_p_0) -- (1_p_1) -- (1_p_2) -- (1_p_3)  -- cycle;
	
	\coordinate (2_head) at ({0.1020},{0.0520});
	\coordinate (2_p_0) at ({0.0820},{0.0600});
	\coordinate (2_p_1) at ({0.0940},{0.0320});
	\coordinate (2_p_2) at ({0.1220},{0.0440});
	\coordinate (2_p_3) at ({0.1100},{0.0720});
	
	\draw [color=black!80, rounded corners=0mm] (2_p_0) -- (2_p_1) -- (2_p_2) -- (2_p_3)  -- cycle;

	\node ()[font=\tiny, anchor=center, color=black, scale=1.2] at (0.088, 0.02) {$\hat{\S}_1$};
	\node ()[font=\tiny, anchor=center, color=black, scale=1.2] at (0.075, 0.070) {$\hat{\S}_2$};
		\node ()[font=\tiny, anchor=center, color=black!70, scale=1.2] at (0.085, 0.110) {$\hat{\S}_{kl}^*$};
	\node ()[font=\small, anchor=center, color=black] at (0.18, 0.01) {$\left( \vs_i \right)_1$};
	\node ()[font=\small, anchor=center, color=black] at (0.05, 0.12) {$\left( \vs_i \right)_2$};
	\node ()[font=\tiny, anchor=center, color=my-full-blue, scale=1.2] at (0.110, 0.115) {$\hat{\S}_{i}$};
	\node ()[font=\tiny, anchor=center, color=my-full-green, scale=1.2] at (0.164, 0.072) {$\hat{\S}_{i-1}$};
	\node ()[font=\tiny, anchor=center, color=my-full-orange!0!black!0, scale=1.2] at (0.125, 0.0850) {$\hat{\S}^*_{cr}$};

	\draw[->,opacity=0.7,line width=1] (0.05, 0.01) -- (0.16, 0.01) ;
	\draw[->,opacity=0.7,line width=1] (0.05, 0.01) -- (0.05, 0.11) ;

	\coordinate (z) at ({0.1231},{0.0846});
\end{tikzpicture}
\vspace{-4mm}
\caption{Analyzing our running example with the \domain domain.}%
\label{fig:example_ch}
\end{wrapfigure}
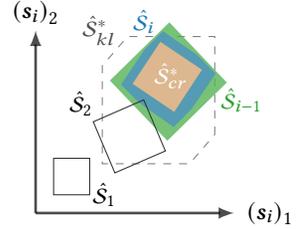

\paragraph{Example (cont.)}
We now use the \domain domain to analyze our running example and illustrate the result in \cref{fig:example_ch}. We regularly apply order reduction (discussed later) to the intermediate \domain to limit its representation size and thus obtain different, simpler intermediate states than with Zonotope. We find a post-fixpoint of the iterator when our efficient inclusion check (also discussed later) shows that the blue region $\hat{\S}_{i}$ is contained within the green one ($\hat{\S}_{i-1}$). While the blue $\hat{\S}_{i}$ is thus an abstraction of the true fixpoint set, it is still relatively loose. Leveraging fixpoint set preservation (as discussed earlier) and applying additional abstract iterations $\vg^\#$, we obtain the much tighter fixpoint set abstraction $\hat{\S}_{\text{cr}}^*$ (orange). It is almost identical to the one obtained with the much more expensive analysis based on standard Zonotope (see \cref{fig:example:abstract}), infeasible in higher dimensions, and much more precise than the one obtained with Kleene iteration ($\hat{\S}_{kl}^*$ shown dashed).

\subsection{The \tool Verifier}
Combining our theoretical insights and the \domain domain, we introduce \tool, an efficient verifier of high-dimensional fixpoint iterations with convergence guarantees, based on the abstract fixpoint iterations outlined above.
We discuss \tool in detail in \cref{sec:tool}, providing soundness proofs and detailed engineering considerations, before demonstrating in an extensive evaluation that it achieves state-of-the-art performance for monDEQ verification in \cref{sec:eval}.

\section{Abstracting Fixpoint Iterations}
\label{sec:abstract-fp}
In this section, we propose a novel, domain-specific abstract interpretation approach for (high-dimensional) fixpoint iterations. %

\paragraph{Fixpoint Iterations}
We consider the general case of a function $\vf$ with a unique fixpoint $\vz^*(\vx) = \vf(\vx, \vz^*)$ given a bounded input $\vx$, i.e. $\|\vx\| < \infty$. Allowing for preprocessing on $\vx$ and postprocessing on $\vz^*$, this encompasses a wide range of problems including monDEQs.

\begin{wrapfigure}[8]{r}{0.31 \textwidth}
	\vspace{-6.5mm}
	\begin{lstlisting}[language=Python,mathescape=true,numbers=none]
	def fixpoint($x$):
		$z$ = $0$
		$u$ = $0$
		while not converged(z):
			$z, u$ = $g_\alpha(x, z, u)$
		return $z$
	\end{lstlisting}
	\vspace{-5mm}
	\captionof{figure}{Iterative fixpoint computation with solver $\vg_\alpha(\vx, \vz, \vu)$.}
	\label{fig:fixpoint_iter}
\end{wrapfigure}
\paragraph{Fixpoint solvers}
Often, iteratively applying $\vf$ converges only slowly or not at all towards a fixpoint \citep{MonDEQWinstonK20}. Instead, iterative root-finding algorithms are applied to $\vf(\vx, \vz) - \vz$ to find the fixpoint \citep{DEQBaiKK19}. We will introduce specific instantiations later (see \cref{sec:mondeq}) and for now assume that we have access to a so-called \emph{fixpoint solver} $\vg_\alpha$ with parameters $\alpha$ which converges to a unique fixpoint in finitely many steps, i.e., $\forall \epsilon \in \R^{>0}, \exists l \in \mathbb{N}, \forall k > l, \norm{\vz_k(\vx) - \vz^*(\vx)} < \epsilon$. 

\paragraph{Concrete Semantics} We write $\g(\vx,\vs_{n})$ for the concrete semantics of one iteration of a generic fixpoint solver, where the latent variable $\vs \to [\vz; \vu]$ contains an auxiliary variable $\vu$ in addition to $\vz$. We build the concrete semantics for specific instantiations of $\g$ directly from those of the constituting mathematical operations, e.g., in Python \citep{guth2013formal}. We write $\boldsymbol{\Xi}(\vx)$ for the concrete semantics of the fixpoint solver iterated until convergence (i.e., with $\eps \to 0$) given an initialization of $\vs_{0} = \0$, as illustrated in the pseudo-code for a fixpoint solver shown in \cref{fig:fixpoint_iter}. We construct these concrete semantics from those of any $\g$ satisfying the above convergence guarantees.

\begin{wrapfigure}[13]{r}{0.5 \textwidth}
	\centering
	\vspace{-6mm}
	\scalebox{0.7}{\begin{tikzpicture}
\tikzset{>=latex}

	\def\dx{5.2}
	\def\dy{1.9}
	\def\yr{0.55}

	\node (input_state) [rectangle,
	minimum width=0cm, minimum height=5cm, align=center, scale=0.8, rounded corners=2pt,
	anchor=center] at (-\dx, 0) {
		\begin{tikzpicture}
			\coordinate (zono_out_head) at ({1.0000},{1.0000});
			\coordinate (zono_out_p_0) at ({-0.1103},{-1.6393});
			\coordinate (zono_out_p_1) at ({3.4443},{1.3159});
			\coordinate (zono_out_p_2) at ({3.4443},{2.5159});
			\coordinate (zono_out_p_3) at ({2.5103},{3.6393});
			\coordinate (zono_out_p_4) at ({2.1103},{3.6393});
			\coordinate (zono_out_p_5) at ({-1.4443},{0.6841});
			\coordinate (zono_out_p_6) at ({-1.4443},{-0.5159});
			\coordinate (zono_out_p_7) at ({-0.5103},{-1.6393});
			
			\fill [fill=my-full-green!45, opacity=1, rounded corners=0mm] (zono_out_p_0) -- (zono_out_p_1) -- (zono_out_p_2) -- (zono_out_p_3) -- (zono_out_p_4) -- (zono_out_p_5) -- (zono_out_p_6) -- (zono_out_p_7)  -- cycle;

			\coordinate (zono_in_proj_head) at ({1.0200},{0.9200});
			\coordinate (zono_in_proj_p_0) at ({1.2700},{-0.2800});
			\coordinate (zono_in_proj_p_1) at ({2.8700},{1.9200});
			\coordinate (zono_in_proj_p_2) at ({2.8700},{2.9200});
			\coordinate (zono_in_proj_p_3) at ({2.5700},{2.9200});
			\coordinate (zono_in_proj_p_4) at ({1.9700},{2.7200});
			\coordinate (zono_in_proj_p_5) at ({0.7700},{2.1200});
			\coordinate (zono_in_proj_p_6) at ({-0.8300},{-0.0800});
			\coordinate (zono_in_proj_p_7) at ({-0.8300},{-1.0800});
			\coordinate (zono_in_proj_p_8) at ({-0.5300},{-1.0800});
			\coordinate (zono_in_proj_p_9) at ({0.0700},{-0.8800});
		
			\fill [fill=my-full-blue!60, opacity=1, rounded corners=0mm] (zono_in_proj_p_0) -- (zono_in_proj_p_1) -- (zono_in_proj_p_2) -- (zono_in_proj_p_3) -- (zono_in_proj_p_4) -- (zono_in_proj_p_5) -- (zono_in_proj_p_6) -- (zono_in_proj_p_7) -- (zono_in_proj_p_8) -- (zono_in_proj_p_9)  -- cycle;

			\node ()[minimum height=0cm, anchor=south east, scale=1.2] at ($(zono_out_p_4)+(0.02,-0.15)$) {$\hat{\S}_n$};	
			\node ()[minimum height=0cm, anchor=north west, scale=1.2] at ($(zono_in_proj_p_1)+(-0.05,0.7)$) {$\hat{\S}_{n+1}$};	
		\end{tikzpicture}
			
	};

	\node (output_state) [fill=black!00, rectangle,
minimum width=0cm, minimum height=5cm, align=center, scale=0.8, rounded corners=2pt,
anchor=center] at (0, 0) {
	\begin{tikzpicture}
	
	\coordinate (zono_in_proj_head) at ({1.0200},{0.9200});
	\coordinate (zono_in_proj_p_0) at ({1.2700},{-0.2800});
	\coordinate (zono_in_proj_p_1) at ({2.8700},{1.9200});
	\coordinate (zono_in_proj_p_2) at ({2.8700},{2.9200});
	\coordinate (zono_in_proj_p_3) at ({2.5700},{2.9200});
	\coordinate (zono_in_proj_p_4) at ({1.9700},{2.7200});
	\coordinate (zono_in_proj_p_5) at ({0.7700},{2.1200});
	\coordinate (zono_in_proj_p_6) at ({-0.8300},{-0.0800});
	\coordinate (zono_in_proj_p_7) at ({-0.8300},{-1.0800});
	\coordinate (zono_in_proj_p_8) at ({-0.5300},{-1.0800});
	\coordinate (zono_in_proj_p_9) at ({0.0700},{-0.8800});
	
	\fill [fill=my-full-blue!60, opacity=1, rounded corners=0mm] (zono_in_proj_p_0) -- (zono_in_proj_p_1) -- (zono_in_proj_p_2) -- (zono_in_proj_p_3) -- (zono_in_proj_p_4) -- (zono_in_proj_p_5) -- (zono_in_proj_p_6) -- (zono_in_proj_p_7) -- (zono_in_proj_p_8) -- (zono_in_proj_p_9)  -- cycle;
	
	\coordinate (zono_in_2_head) at ({0.9200},{0.9500});
	\coordinate (zono_in_2_p_0) at ({1.0700},{0.0000});
	\coordinate (zono_in_2_p_1) at ({2.6700},{2.2000});
	\coordinate (zono_in_2_p_2) at ({2.6700},{2.8000});
	\coordinate (zono_in_2_p_3) at ({2.5700},{2.8000});
	\coordinate (zono_in_2_p_4) at ({1.3700},{2.2000});
	\coordinate (zono_in_2_p_5) at ({0.7700},{1.9000});
	\coordinate (zono_in_2_p_6) at ({-0.8000},{-0.3000});
	\coordinate (zono_in_2_p_7) at ({-0.8000},{-0.9000});
	\coordinate (zono_in_2_p_8) at ({-0.7300},{-0.9000});
	
	\fill [fill=my-full-orange!50, opacity=1, rounded corners=0mm] (zono_in_2_p_0) -- (zono_in_2_p_1) -- (zono_in_2_p_2) -- (zono_in_2_p_3) -- (zono_in_2_p_4) -- (zono_in_2_p_5) -- (zono_in_2_p_6) -- (zono_in_2_p_7) -- (zono_in_2_p_8)  -- cycle;

	\coordinate (z_p_0) at ({2.3},{1.9});
	\coordinate (z_p_1) at ({1},{0.9});
	\coordinate (z_p_2) at ({0},{0.8});
	\coordinate (z_p_3) at ({0},{0.8});
	\coordinate (z_p_4) at ({-0.5},{-0.5});
	\coordinate (z_p_5) at ({1},{0.3});
	\fill [fill=my-full-red!50, opacity=1, rounded corners=0mm] (z_p_0) -- (z_p_1) -- (z_p_2) -- (z_p_3) -- (z_p_4) -- (z_p_5) -- cycle;

	\node ()[minimum height=0cm, anchor=south east, scale=1.2] at ($(zono_in_proj_p_3)+(0.02,-0.10)$) {$\hat{\S}_{n+1}$};	
	\node ()[minimum height=0cm, anchor=south east, scale=1.2] at (0.7,0) {$\S^{*}$};
	\end{tikzpicture}

};

	\draw[->] ($(input_state.north)+(1.93,-1.32)$)  to [out=30,in=150, looseness=1] ($(output_state.north)+(-0.1,-1.06)$);
	\node ()[minimum height=0, anchor=south, align=center] at ($(input_state.east)!0.50!(output_state.west)+(0.5,1.50)$) {$\vg^\#_\alpha(\hat{\X},\hat{\S}_{n})$};

	\draw [<-, draw=black!80] ($(input_state.east)+(-1.36,0.0)$) -- ($(output_state.west)+(0.85,0.0)$);
	\node ()[minimum height=0, anchor=south, align=center] at ($(input_state.east)!0.50!(output_state.west)$) {$\hat{\S}_{n+1} \sqsubseteq \hat{\S}_{n}$};

	\draw[<-] ($(output_state.south)+(-1.8,0.52)$)  to [out=210,in=330, looseness=1] ($(input_state.south)+(-1.0,0.76)$);
	\node ()[minimum height=0, anchor=south, align=center] at ($(input_state.east)!0.50!(output_state.west)+(-1.8,-2.02)$) {$\vg_\alpha(\gamma(\hat{\X}),\gamma(\hat{\S}_{n+1}))$};

	\node ()[anchor=north, align=center] at (-1.1*\dx, -2.4) { Input to iteration step};	
	\node ()[anchor=north, align=center] at (-0.1*\dx, -2.4) { Output of iteration step};	

\end{tikzpicture}}
	\vspace{-4mm}
	\caption{
		If an over-approximated solver iteration $\hat{\S}_{n+1} = \gS(\hat{\X},\hat{\S}_{n})$
		(blue) is contained in the previous state $\hat{\S}_{n}$ (green), any (exact) iteration of $\vg_{\alpha}(\gamma(\hat{\X}),\gamma(\hat{\S}_{n+1}))$ (orange)
		will not escape from $\hat{\S}_{n+1}$. This implies containment of the true fixpoint set (red) $\S^{*}  \subseteq \gamma(\hat{\S}_{n+1})$.
	}
	\label{fig:abstract_iteration}
\end{wrapfigure}
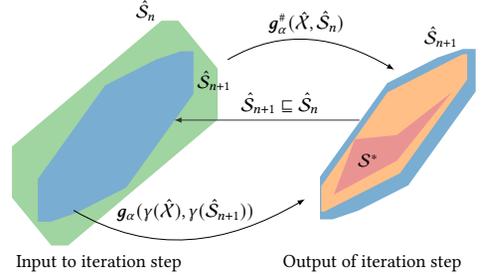
\paragraph{Abstract Semantics}
We let $\X \subseteq \R^q$ denote a set of inputs, $\S_{n} \subseteq \R^p$ the corresponding intermediate solver states at step $n$ (we write $[\Z_n, \U_n] \gets \S_n$ to obtain the two constituting sets), and $\Z^{*} \coloneqq \{\vz^{*}(\vx) \mid \vx \in \X\}$ the corresponding fixpoints.
We now define the abstract semantics for a single step of the fixpoint solver as any sound abstract transformer $\gS: \aid \times \aid \mapsto \aid$ of the iterated function $\g: \R^{p+q} \mapsto \R^p$, i.e., any $\gS$ satisfying $\gamma(\gS(\hat{\X},\hat{\S}_n)) \supseteq \{\g(\vx, \vs) \mid \vx \in \gamma(\hat{\X}), \vs \in \gamma(\hat{\S}_n)\}$, for the abstract elements $\hat{\X}, \hat{\S}_n, \hat{\Z}_n, \hat{\U}_n$ from domain $\aid$, which over-approximate the corresponding sets, e.g., $\gamma(\hat{\X}) \supseteq \X$.

We define the abstract semantics of the fixpoint solver $\boldsymbol{\Xi}$ yielding the concrete $\Z^{*}$ such that it satisfies $\gamma(\boldsymbol{\Xi}^\#(\hat{\X})) \supseteq \{\boldsymbol{\Xi}(\vx) \mid \vx \in \X\} \eqqcolon \Z^{*}$. To this end, we derive the following theorem:

\begin{restatable}[Fixpoint contraction]{theorem}{contraction}
	\sloppy
	\label{the:contraction}
	Let
	\begin{itemize}
		\item \g be an iterative process, guaranteed to converge to a unique fixpoint $\vz^*$ in finitely many steps for any bounded input $\vx$, and \gS its sound abstract transformer,
		\item $\hat{\S}_{n+1} \coloneqq \gS(\hat{\X},\hat{\S}_n)$ an abstract element in $\aid$ describing a closed set and denoting an over-approximation of applying \g ($n+1$)-times for some $\vz_{0},\vu_{0}$ on all inputs $\vx' \in \hat{\X}$.
	\end{itemize}
	Then for $[\hat{\Z}_n, \hat{\U}_n] \gets \hat{\S}_n$:
	\begin{equation}
	\label{eq:contraction}
	\hat{\S}_{n+1} \sqsubseteq \hat{\S}_n
	\mspace{9.0mu}\implies\mspace{9.0mu} \Z^{*} \subseteq \gamma(\hat{\Z}_{n+1}).
	\end{equation}
\end{restatable}

Intuitively, if we consistently apply \gS until we detect contraction ($\hat{\S}_{n+1} \sqsubseteq \hat{\S}_n$), then $\hat{\S}_{n}$ is a so-called post-fixpoint and its concretization includes the true fixpoint set $\Z^{*} \subseteq \gamma(\hat{\Z}_{n+1}) \subseteq \gamma(\hat{\Z}_{n})$.

We defer the formal proof to \cref{sec:proofs} but provide an intuition along the illustration in \cref{fig:abstract_iteration} below.
Once an iteration of $\gS$ maps $\hat{\S}_{n}$ (green in \cref{fig:abstract_iteration}) to a subset of itself $\hat{\S}_{n+1} \sqsubseteq \hat{\S}_{n}$ (blue), it follows from the soundness of \gS that any number $k>0$ of further applications of the concrete \g to $\gamma(\hat{\S}_{n+1}) \subseteq \gamma(\hat{\S}_{n})$ will map into $\S_{n+1} \subseteq \gamma(\hat{\S}_{n+1})$ (orange) and thus never `escape' $\gamma(\hat{\S}_{n+1})$.
As $\vg_{\alpha}$ is guaranteed to converge in finitely many steps in the concrete, the previously obtained $\gamma(\hat{\S}_{n+1}) \supseteq \S_{n+1} \supseteq \S_{n+k}$ must thus contain the true fixpoint set $[\Z^{*}, \U^*] \gets \S^{*}$ (red). Note that this does not necessarily hold for applications of the abstract transformer $\gS$, as it, in contrast to \g, is not necessarily monotonic, i.e., $\hat{\S}_A \sqsubseteq \hat{\S}_B \notimplies \gS(\hat{\S}_A) \sqsubseteq \gS(\hat{\S}_B)$.

Empirically, the abstractions found via \cref{the:contraction} are often relatively loose. However, we can obtain a more precise abstraction by applying additional iterations of a fixpoint set preserving abstract solver \gS to $\hat{\S}_{n+1}$.

\paragraph{Fixpoint Set Preservation}
We call an abstract transformer \gS of \g fixpoint set preserving if and only if applying it to any abstract state $\hat{\S}_{n}$ that contains the true fixpoint set, i.e., $\gamma(\hat{\S}_{n}) \supseteq \S^{*}$, results in an abstract state $\hat{\S}_{n+1}$ that still contains the true fixpoint set $\S^{*}$. Formally:
\begin{definition}[Fixpoint set preservation] \label{thm:preserving}
	We call an abstract transformer \gS fixpoint set preserving if and only if
		$\S^* \subseteq \gamma(\hat{\S}_{n}) \implies \S^* \subseteq \gamma(\gS(\hat{\X}, \hat{\S}_n)).$
\end{definition}

Indeed, we can show that a broad class of \gS are fixpoint set preserving:

\begin{restatable}[Fixpoint set preservation]{theorem}{iterpw}
	\label{the:iter_PR}
	Every sound abstract transformer \gS of a locally Lipschitz \g with convergence guarantees in the concrete is fixpoint set preserving.
\end{restatable}

We again defer the formal proof to \cref{sec:proofs}. Intuitively, given the local Lipschitzness of \g and the uniqueness of the fixpoint, the convergence guarantee can only hold if fixpoints are preserved.

\paragraph{Abstract Interpreter}
The above results can be used to construct abstract interpreters for arbitrary locally Lipschitz iterative processes converging to unique fixpoints in finitely many steps. %
While we focus on the verification of monDEQs, we illustrate the wider applicability of our approach on a toy example of a square-root computation using the Householder method in \cref{sec:householder}. 

To actually construct an abstract interpreter for high-dimensional problems based on the above results, we need a suitable abstract domain $\aid$ equipped with an efficient containment check $\sqsubseteq$ and precise transformers $\vg_\alpha^\#$ for the used fixpoint solver.

\section{The \domain Abstract Domain}\label{sec:m_zono}
In this section, we introduce the \domainl (\domain), a novel abstract domain that enables our efficient domain-specific abstract interpreter. %
Based on Zonotope \cite{ZonotopeGhorbalGP09}, our domain is designed to carefully balance three features: (i) efficient propagation of abstract elements,  (ii) fast (abstract) inclusion checks, and (iii) precision of all abstract transformers needed for (i) and (ii). Recall that none of the abstract domains typically used for neural network verification satisfies all three of these requirements, as they were designed for neural architectures with a constant (small) number of layers (see \cref{sec:other_domains}).

\paragraph{Zonotope} We begin with a brief recap of the Zonotope domain \citep{ZonotopeGhorbalGP09,DeepZSinghGMPV18}.
A Zonotope $\hat{\Z} \in \aid$ describing a volume $\gamma(\hat{\Z}) \subseteq \R^{p}$, is defined as $ \hat{\Z} = \mA \ez + \va,$ where $\mA \in \R^{p \times k}$ is called the error coefficient matrix, $\va \in \R^{p}$ the center, and $\ez = [-1,1]^k$ the Zonotope error terms. 
Its concretization function is defined as $\gamma(\hat{\Z}) \coloneqq \{ \vx = \mA \ez + \va \mid \ez \in [-1,1]^k\}$. Using the exact abstract transformer $f^\#(\hat{\Z})$ of \citet{DeepZSinghGMPV18} for an affine transformations $f(\vx) = \mW \vx + \vc$, we obtain $\hat{\Z}'$ with $\mA' = \mW\mA$ and $\va'=\mW \va + \vc$.

\paragraph{\domain}
We define a \domain $\hat{\Z} \in \aid$ describing a volume $\gamma(\hat{\Z}) \subseteq \R^{p}$ as
\begin{align}
	\hat{\Z} = \mA \ez + \diag(\vb) \eb + \va,
\end{align}
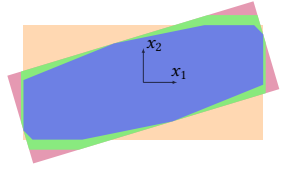
\begin{wrapfigure}[10]{r}{0.355 \textwidth}
	\centering
	\vspace{-4mm}
	\scalebox{0.43}{\begin{tikzpicture}
	\tikzset{>=latex}

	\node () [rectangle, minimum height=3cm, align=center, scale=0.7, rounded corners=2pt,
	anchor=center] at (11, 1) {
		\begin{tikzpicture}

			\coordinate (box_head) at ({1.0000},{1.0000});
			\coordinate (box_p_0) at ({-4.3000},{-1.5000});
			\coordinate (box_p_1) at ({6.3000},{-1.5000});
			\coordinate (box_p_2) at ({6.3000},{3.5000});
			\coordinate (box_p_3) at ({-4.3000},{3.5000});

			\fill [fill=orange!50, opacity=0.6, rounded corners=0mm] (box_p_0) -- (box_p_1) -- (box_p_2) -- (box_p_3)  -- cycle;
			\coordinate (zono3_head) at ({1.0000},{1.0000});
			\coordinate (zono3_p_0) at ({-5.0143},{1.2925});
			\coordinate (zono3_p_1) at ({-3.8695},{-2.5420});
			\coordinate (zono3_p_2) at ({7.0143},{0.7075});
			\coordinate (zono3_p_3) at ({5.8695},{4.5420});
			
			\fill [fill=purple!50, opacity=0.8, rounded corners=0mm] (zono3_p_0) -- (zono3_p_1) -- (zono3_p_2) -- (zono3_p_3)  -- cycle;
			\coordinate (zono2_head) at ({1.0000},{1.0000});
			\coordinate (zono2_p_0) at ({-4.4112},{1.4726});
			\coordinate (zono2_p_1) at ({-4.4112},{-0.7274});
			\coordinate (zono2_p_2) at ({-4.0495},{-1.9389});
			\coordinate (zono2_p_3) at ({-1.8495},{-1.9389});
			\coordinate (zono2_p_4) at ({6.4112},{0.5274});
			\coordinate (zono2_p_5) at ({6.4112},{2.7274});
			\coordinate (zono2_p_6) at ({6.0495},{3.9389});
			\coordinate (zono2_p_7) at ({3.8495},{3.9389});

			\fill [fill=green!55, opacity=0.8, rounded corners=0mm] (zono2_p_0) -- (zono2_p_1) -- (zono2_p_2) -- (zono2_p_3) -- (zono2_p_4) -- (zono2_p_5) -- (zono2_p_6) -- (zono2_p_7)  -- cycle;
			\coordinate (zono_head) at ({1.0000},{1.0000});
			\coordinate (zono_p_0) at ({-4.3000},{1.1000});
			\coordinate (zono_p_1) at ({-4.3000},{-1.1000});
			\coordinate (zono_p_2) at ({-3.9000},{-1.5000});
			\coordinate (zono_p_3) at ({-1.7000},{-1.5000});
			\coordinate (zono_p_4) at ({2.3000},{-0.7000});
			\coordinate (zono_p_5) at ({6.3000},{0.9000});
			\coordinate (zono_p_6) at ({6.3000},{3.1000});
			\coordinate (zono_p_7) at ({5.9000},{3.5000});
			\coordinate (zono_p_8) at ({3.7000},{3.5000});
			\coordinate (zono_p_9) at ({-0.3000},{2.7000});

			\fill [fill=blue!60, opacity=0.8, rounded corners=0mm] (zono_p_0) -- (zono_p_1) -- (zono_p_2) -- (zono_p_3) -- (zono_p_4) -- (zono_p_5) -- (zono_p_6) -- (zono_p_7) -- (zono_p_8) -- (zono_p_9)  -- cycle;

			\draw[->,opacity=0.7,line width=1] (zono_head) -- +(0,1.5) ;
			\draw[->,opacity=0.7,line width=1] (zono_head) -- +(1.5,0) ;
			
			\node [align=center, minimum height=0cm, scale=2.3] at (2.6,1.4) {$x_1$};
			\node [align=center, minimum height=0cm, scale=2.3] at (1.5,2.6) {$x_2$};
			
		\end{tikzpicture}
		
	};

\end{tikzpicture}}
	\vspace{-2mm}
	\caption{Over-approximations of an improper \domain (blue) by a proper one with (green) and without (red) Box component and a Box (orange).} \label{fig:mzono_vis}
\end{wrapfigure}
by extending the Zonotope domain with the Box error vector $\vb \in \left(\R^{\geq 0} \right)^{p}$ and corresponding Box error terms $\eb = [-1,1]^p$. 
If $\mA$ is invertible, i.e. full rank and $k=p$, we call $\hat{\Z}$ a \emph{proper} \domain, or else an \emph{improper} one. 
We adapt the concretization function to $\gamma(\hat{\Z}) = \{ \vx = \mA \ez + \diag(\vb) \eb + \va \mid \ez \in [-1,1]^k, \eb \in [-1,1]^p \}$ and define a partial order $\sqsubseteq$ over \domain based on the set inclusion $\subseteq$ of their concretizations. 
Formally, any \domain can be seen as the Minkowski sum of a Zonotope ($\mA \ez$) and a Hyperbox ($\diag(\vb) \eb$), also called Hybrid Zonotope \citep{DiffAIMirmanGV18,GoubaultP2008}. However, not every Hybrid-Zonotope is a \emph{proper} \domain as their error matrix is generally not invertible, which is crucial for our efficient containment check (discussed later).

Computing $\hat{\Z}' \sqsubseteq \hat{\Z}$ exactly is generally intractable. Therefore, we introduce an efficient over-approximation (discussed later) that is sound but not complete and requires the outer \domain to be proper. By slight abuse of notation, we also denote it by $\sqsubseteq$. While a similar inclusion check is possible for any standard Zonotope with $p$ linearly independent error terms ($\vb=\0$), equivalent to a Parallelotope \citep{ParallelotopeAmatoS12}, and any Box approximation ($\mA = \0$), a \domain yields a tighter abstraction than either since it can effectively employ twice as many error terms. 
We visualize this in \cref{fig:mzono_vis}, where we show a Box (orange), Parallelotope (red), and proper \domain (green) abstraction of the original set (blue).

\paragraph{Abstract Transformers}
For affine transformations, we use the Zonotope transformer described above, casting the Box errors as Zonotope errors by setting $\hat{\mA} = [\mA, \diag(\vb)]$ and $\hat{\vb}=\mbf{0}$ before applying the transformer. This yields an improper \domain with a zero Box component.
To encode the ReLU function, $y=\max(x, 0)$, for a \domain $\hat{\bc{Z}}$, we modify the Zonotope transformer proposed by \citet{DeepZSinghGMPV18} (recovered for $\vb=\mbf{0}$):
\begin{alignat*}{2}
	\hat{\Z}' = \mA' \ez &+ \diag(\vb') \eb +  &&\va' = ReLU^\#_{\boldsymbol{\lambda}}(\hat{\Z})\\
	\mA' &= \boldsymbol{\lambda} \mA' 	&&\vb' = \boldsymbol{\lambda} \vb + \boldsymbol{\mu}\\
	\va' &= \va + \boldsymbol{\mu}	&&\boldsymbol{\mu} = \begin{cases}
		(1- \boldsymbol{\lambda}) \, \vu_x/2 \quad  & \text{if } 0 \leq \boldsymbol{\lambda} \leq \vu_x/(\vu_x-\vl_x)\\
		-\boldsymbol{\lambda} \, \vl_x /2 \quad & \text{if } \vu_x/(\vu_x-\vl_x) \leq \boldsymbol{\lambda} \leq 1
	\end{cases}.
\end{alignat*}
Applying this transformer %
will result in a \domain with a non-zero Box component without changing its properness. By default, we choose $\lambda = u_x/(u_x-l_x)$ leading to the smallest volume in the $2$d input-output space.

\paragraph{Consolidating Error Terms}
To enable efficient inclusion checks and limit the number of error terms, we regularly over-approximate an improper \domain $\hat{\Z}$ ($\mA \in \R^{p \times k}$) with a proper one  $\hat{\Z}'$ (invertible $\mA' \in \R^{p \times p}$).
We call this process \emph{error consolidation}. If the Box component is zero ($\vb=0$) and $k>p$, this is known in the literature as order reduction via outer-approximation \citep{SadraddiniT19,KopetzkiSA17O}.
If $k > p$, we consolidate the $k$ old error terms into $p$ new ones, thus reducing the representation size. We ensure that $\mA'$ has full rank and is therefore invertible. If $k \leq p$, we pick a subset with full rank and complete it to a basis.
In monDEQ certification, $p = \dim(\vz)$ is the size of the latent dimension.

\begin{restatable}[Consolidating errors]{thm}{consolidation}
	\label{thm:consolidation}
	Let $\hat{\Z} = \mA \ez + \diag(\vb)\eb + \va$ be an improper \domain with $\mA \in \R^{p \times k}$. Further, let $\tilde{\mA} \in \R^{p \times p}$ be invertible. Then the proper \domain \mbox{$\hat{\Z}' = \mA'\ve'_1 + \diag(\vb)\eb + \va$} with
	\begin{equation}
	\mA' =  \diag(\vc) \tilde{\mA}
	\qquad
	\qquad
	\qquad
	\text{where }
	\vc = |\tilde{\mA}^{-1} \mA| \mbf{1}
	\label{eq:consolidation}
	\end{equation}
	is a sound over-approximation, i.e., $\hat{\Z}' \sqsupseteq \hat{\Z}$ of the improper one, where $\mbf{1}$ denotes the $k$-dimensional one vector and $|\cdot|$ the elementwise absolute. We call $\vc$ the consolidation coefficients.
\end{restatable}

\begin{wrapfigure}[10]{r}{0.42 \textwidth}
	\centering
	\vspace{-8mm}
	\scalebox{0.77}{\hspace{-3mm}\begin{tikzpicture}
	\tikzset{>=latex}
	\clip (-2,-2) rectangle (3.25, 3.05);
	\node () [rectangle,
	minimum width=9.2cm, minimum height=5cm, align=center, scale=0.9, rounded corners=2pt,
	anchor=center] at (1, 1) {
		 \begin{tikzpicture}
			
			\coordinate (zono_head) at ({1.0000},{1.0000});
			\coordinate (zono_p_0) at ({1.0000},{2.0000});
			\coordinate (zono_p_1) at ({-1.0000},{-0.4000});
			\coordinate (zono_p_2) at ({1.0000},{0.0000});
			\coordinate (zono_p_3) at ({3.0000},{2.4000});

			\coordinate (zono_3_head) at ({1.0000},{1.0000});
			\coordinate (zono_3_p_0) at ({-1.3624},{0.0359});
			\coordinate (zono_3_p_1) at ({-0.3792},{-1.1467});
			\coordinate (zono_3_p_2) at ({3.3624},{1.9641});
			\coordinate (zono_3_p_3) at ({2.3792},{3.1467});

			\fill [ fill=black!30, opacity=0.4, rounded corners=0mm] (zono_3_p_0) -- (zono_3_p_1) -- (zono_3_p_2) -- (zono_3_p_3)  -- cycle;
			\fill [ fill=my-full-green!50, opacity=0.6, rounded corners=0mm] (zono_p_0) -- (zono_p_1) -- (zono_p_2) -- (zono_p_3)  -- cycle;
			
		\coordinate (c_1) at ($(zono_p_0)+(zono_p_1)-(zono_3_p_0)$);
		\fill [ fill=my-full-blue!55, opacity=0.8, rounded corners=0mm] (zono_p_1) -- (zono_3_p_0) -- (zono_p_0) -- (c_1) -- cycle;
		\coordinate (c_2) at ($(zono_p_0)+(zono_p_3)-(zono_3_p_3)$);
		\fill [ fill=my-full-red!55, opacity=0.8, rounded corners=0mm] ($(zono_p_0)+(zono_p_1)-(zono_3_p_0)$) -- ($(zono_3_p_3)+(zono_p_1)-(zono_3_p_0)$) -- ($(zono_3_p_3)+(zono_p_1)-(zono_3_p_0)+(zono_p_3)-(zono_3_p_3)$) -- ($(c_2) +(zono_p_1)-(zono_3_p_0)$)-- cycle;

		\draw[->, draw=black!80, line width=0.8] (zono_3_p_0) -- (zono_3_p_3);
		\draw[->, draw=black!80, line width=0.8] (zono_3_p_3) -- (zono_3_p_2);

		\draw[->, draw=my-full-blue!80, line width=0.8] (zono_p_1) -- (zono_p_0);
		\draw[->, draw=my-full-blue!80, style=dashed, line width=0.8] (c_1) -- (zono_p_0);
		\draw[->, draw=my-full-blue!80, style=dashed, line width=0.8] (zono_p_1) -- (c_1);

		\draw[->, draw=my-full-red!80, line width=0.8] (c_1) -- (zono_3_p_2);
		\draw[->, draw=my-full-red!80, style=dashed, line width=0.8] (c_1) -- ($(c_2) +(zono_p_1)-(zono_3_p_0)$) ;
		\draw[->, draw=my-full-red!80, style=dashed, line width=0.8] ($(c_2) +(zono_p_1)-(zono_3_p_0)$) -- (zono_3_p_2);

		\node ()[minimum size=4pt, inner sep=0pt, anchor=center] at (zono_head) {$+$};

		\node[rotate=45, inner sep=0pt, anchor=center] at ($(zono_3_p_0)!0.5!(zono_3_p_3) + (-0.35,0.25)$)  {$\mA'_{\cdot,1}$};
		\node[rotate=-45, inner sep=0pt, anchor=center] at ($(zono_3_p_2)!0.5!(zono_3_p_3) + (0.35,-0.0)$)  {$\mA'_{\cdot,2}$};
		\node[text=my-full-blue!85!black,rotate=59.19301888, inner sep=0pt, anchor=center] at (-0.8, 0.2)  {$\mA_{\cdot,1}$};
		\node[text=my-full-red!90!black,rotate=18.94569599, inner sep=0pt, anchor=center] at (2.1, 1.4)  {$\mA_{\cdot,2}$};
		\node[text=my-full-green!95!black,rotate=0, inner sep=0pt, anchor=center] at (1.2, -.2)  {$\hat{\Z}$};
		\node[text=black!60,rotate=0, inner sep=0pt, anchor=center] at (0.3, -1.0)  {$\hat{\Z}'$};

		\end{tikzpicture}
			
	};

\end{tikzpicture}}
	\vspace{-9mm}
	\caption{Illustration of error consolidation via \cref{thm:consolidation}.  %
		All vectors are scaled by factor 2.}
	\label{fig:error_cons}
\end{wrapfigure}
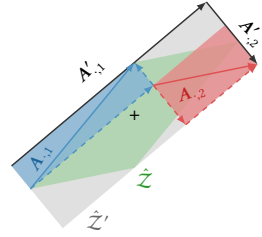
The intuition behind this approximation is shown in \cref{fig:error_cons}, where we over-approximate the green $\hat{\Z}$ with the gray $\hat{\Z}'$ (choosing a suboptimal basis for illustration purposes).
All error vectors (columns) in the old error matrix $\mA$ (shown as solid red and blue arrows) are first decomposed into a linear combination of error vectors in the new basis $\tilde{\mA}$ (dashed red and blue arrows). Then, the absolute values (to correct for their orientation) of these contributions $|\tilde{\mA}^{-1} \mA|$ are summed up over all error vectors to obtain the consolidation coefficients $\vc$. Finally, we multiply these consolidation coefficients with the error directions of the new basis $\tilde{\mA}$ to obtain the new error matrix $\mA'_{\cdot, i} = \evc_{i} \tilde{\mA}_{\cdot, i}$ (solid black arrows). We note that this has complexity $\bc{O}(p^2(p+k))$.

\paragraph{Choosing the New Error Basis}
To minimize the imprecision incurred when consolidating error terms, a suitable basis $\tilde{\mA}$ has to be chosen.
We use the PCA-basis of the original error matrix $\mA$, as it has been empirically shown to yield the tightest approximation while being computationally feasible in high dimensions \citep{KopetzkiSA17O}. %

\paragraph{Inclusion Checks for \domain}
Enabling efficient inclusion checks in high dimensions is one of the main motivations for the \domain domain. Here, we first provide a high-level outline of our approach, illustrated in \cref{fig:m_zono_cont}, before giving more detail on the individual steps. 

We aim to determine whether the proper \domain $\hat{\Z} = \mA\ez + \diag(\vb)\eb + \va$ (green in \cref{fig:m_zono_cont:true}) contains the improper $\hat{\Z}' =  \mA'\ez + \diag(\vb')\eb + \va'$ (red), i.e., $\hat{\Z} \sqsupseteq \hat{\Z}'$.
At a high level, we first consolidate the errors of the improper $\hat{\Z}'$ (blue) before decomposing both \domain into their Zonotope and Box components (shown in \cref{fig:m_zono_cont:decomposition}) and checking whether the outer components contain their respective inner counterparts (shown in \cref{fig:m_zono_cont:inclusion}). %

To determine containment of the Zonotope component, we consolidate the error matrix $\mA'$ with basis $\mA$ as discussed above (shown in blue in \cref{fig:m_zono_cont:true}). This leads to perfectly aligned error vectors, enabling us to directly compare their lengths. If all error terms of the consolidated $\tilde{\mA}'$ are shorter than their counterparts in $\mA$, the Zonotope components are contained (shown overlayed in \cref{fig:m_zono_cont:inclusion}).
More efficiently, we only compute the consolidation coefficients and check $ \vc = |\mA^{-1} \mA'| \mbf{1} < \mbf{1}$.

To show containment of the Box components, we can simply check that $\vb' \leq \vb$. %
However, we observe that negative values in the difference vector $\vb' - \vb$ denote directions in which $\vb$ is larger than $\vb'$ and can hence compensate for differences in the center terms $\va' - \va$. Positive values in $\vb' - \vb$ denote directions in which $\vb$ is too small to cover $\vb'$.
Combining these two, we obtain a residual Box component $\vd = \max(0, |\va' - \va| + \vb' - \vb)$ that needs to additionally be covered by the Zonotope component.
To this end, we can cast $\vd$ as additional error terms of $\mA'$ and update the Zonotope inclusion check to $|\mA^{-1} \mA'| \mbf{1} + |\mA^{-1} \diag(\vd)|\mbf{1} < \mbf{1}$.
This compensation is not necessary in \cref{fig:m_zono_cont}.
We formalize this containment check as follows, deferring the formal proof to \cref{app:containment}.

\begin{figure*}
	\centering
	\vspace{-1mm}
	\input{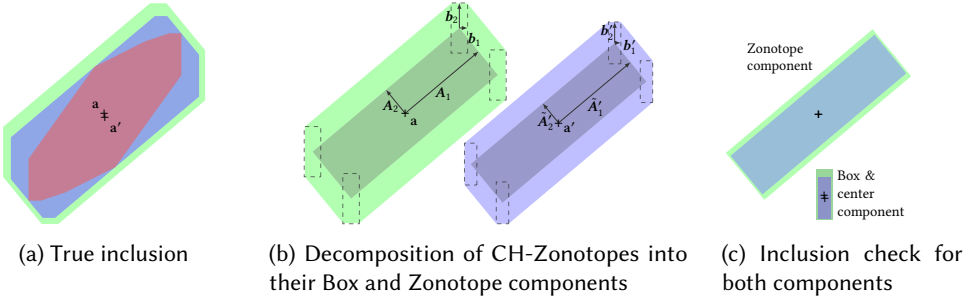}
	\vspace{-3mm}
	\caption{Illustration of checking the containment of an improper \domain (red) in a proper \domain (green), by consolidating errors (blue). \cref{fig:m_zono_cont:decomposition} shows the proper \domain{}s decomposed into their Box and Zonotope components. In \cref{fig:m_zono_cont:inclusion} we illustrate the containment check of these components individually.}
	\vspace{-3mm}
	\label{fig:m_zono_cont}
\end{figure*}

\begin{restatable}[\domain Containment]{thm}{containment}
	\label{the:containment}
	Let $\hat{\Z} = \mA\ez + \diag(\vb)\eb + \va$ be a proper \domain and $\hat{\Z}' =  A'\ez' + \diag(\vb')\eb' + \va'$ an improper one. $\hat{\Z'}$ is contained in $\hat{\Z}$ if
	\begin{align}
	  \left|\mA^{-1} \mA'\right| \mbf{1} + \left|\mA^{-1} \diag\left(\max\left(\mbf{0}, \left| \va'-\va \right| + \vb' - \vb \right) \right) \right| \mbf{1}  \leq \mbf{1} \label{eq:containment}
	\end{align}
	holds element-wise. Where $\mA^{-1}$ always exists as $\hat{\Z}$ is proper and therefore $\mA \in \R^{p \times p}$ invertible.
\end{restatable}

In contrast to exact containment checks for general Zonotope which are co-NP-complete and hence infeasible \citet{Kulmburg2021OnTC}, our \cref{the:containment} constitutes a sound but not complete check with complexity $\bc{O}(p^2(p+k))$. Another approximate method with polynomial time complexity was proposed by \citet{SadraddiniT19}, which they show to be close to loss-less in low dimensions ($p\leq10$). However, their method involves solving a linear program in $\bc{O}(k_\text{inner}k_\text{outer})$ variables with $\bc{O}(p k_\text{inner})$ constraints, where $k_\text{inner} \geq p$ and $k_\text{outer} \geq p$ are the number of error terms and $p$ is the dimensionality. Making generous assumptions on the number of error terms and the complexity of the LP-solver \citep{JiangLP20}, this leads to an overall complexity of $\tilde{\bc{O}}(p^6)$, which makes it practically intractable for our use-case, as we will show later (see \cref{sec:eval:ch_zono}).

\section{Application to Fixpoint-Based Neural Networks}
\label{sec:mondeq}

In this section, we first introduce (monotone Operator) Deep Equilibrium Models (monDEQs) for concrete points in \cref{sec:concrete-mondeq} before considering their abstraction in \cref{sec:tool}.

\subsection{Deep Equilibrium Models on Points}  \label{sec:concrete-mondeq}

\paragraph{Deep Equilibrium Models (DEQs)}
Implicit-Layer \citep{OPTNetAmosK17,ImplicitDLGhaoui2019} and Deep Equilibrium Models (DEQs) \citep{DEQBaiKK19} were recently introduced to enable more memory-efficient model parameterizations.
Unlike traditional deep neural networks, which propagate inputs through a finite number of different layers, DEQs conceptually apply the same layer repeatedly until converged to a fixpoint, corresponding to an infinite depth model with parameter-sharing. A DEQ $\vh$ obtains its final prediction $\vy$ by applying a linear layer to this fixpoint:
\begin{equation}
	\vy = \vh(\vx) \coloneqq \mV \vz^{*} + \vv, \qquad\qquad  \vz^* = \vf(\vx, \vz^*).
	\label{eqn:deq}
\end{equation}

\paragraph{Monotone Operator Deep Equilibrium Models (monDEQs)}
A major drawback of general DEQs is that neither the existence nor uniqueness of their fixpoints is guaranteed. %
To address this issue, \citet{MonDEQWinstonK20} introduced monDEQs as a particular form of DEQs guaranteed to have a unique fixpoint by parametrizing
\begin{equation}
	\vf(\vx, \vz) = \sigma(\mW\vz +\mU\vx+\vb)
	\label{eqn:mon_deq}
\end{equation}
with $\vx \in \R^{q}, \vz \in \R^{p}, \mU \in \R^{p \times q}$, $\mW=(1-m)\mI-\mP^T\mP + \mQ - \mQ^T$ where $\mP,\mQ \in R^{p \times p}$, and monotonicity parameter $m > 0$. These existence and uniqueness properties allow a certification of monDEQs that is independent of how a fixpoint was obtained, yielding far stronger guarantees than possible in the general DEQ setting. Throughout this paper, we will focus on the ReLU activation (i.e., $\sigma \coloneqq ReLU$) and discuss considerations for other activations in \cref{app:activation_functions}.

\paragraph{Fixpoint solvers}
As discussed in \cref{sec:abstract-fp}, iteratively applying $\vf$ often does not converge and iterative fixpoint solvers which converge to the unique fixpoint under mild conditions are employed instead. 
For a monDEQ $\vh$ with iteration function $\vf(\vx,\vz) = ReLU(\mW\vz + \mU\vx +\vb)$, we let $\vg$ denote an iteration of a fixpoint solver using operator splitting:
\begin{itemize}[leftmargin=1.5em]
	\item  \textbf{Forward-Backward Splitting (\fwdbwd)} where $\vs_{n+1}$ is computed as
		\begin{equation}
		\vs_{n+1} \coloneqq \gfwdbwd(\vx,\vs_{n}) = ReLU((1-\alpha)\vs_{n} + \alpha(\mW\vs_{n} + \mU\vx + \vb)), \label{eqn:fwd_bwd}
		\end{equation}
		converging to  $\vs^* = \vz^*$ of $\vf$, for any $0 < \alpha < \frac{2m}{\| \mI - \mW \|_{2}^{2}}$ \citep{MonDEQWinstonK20}.
	\item  \textbf{Peaceman-Rachford Splitting (\pr)} where $\vs_{n+1} \coloneqq \gpr(\vx,\vs_{n})$ is computed as
		\begin{equation}
		\begin{aligned}[c]
			[\vz_n; \vu_n] &\gets \vs_n \\
			\vu_{n+1/2}  &= 2 \vz_{n} - \vu_{n} \\
			\vz_{n+1/2}  &= (\mI + \alpha(\mI-\mW))^{-1}(\vu_{n+1/2} + \alpha (\mU\vx +\vb))
		\end{aligned}
		\qquad
		\begin{aligned}[c]
			\vu_{n+1} &= 2\vz_{n+1/2} - \vu_{n+1/2}\\
			\vz_{n+1} &= ReLU(\vu_{n+1})\\
			\vs_{n+1} &\gets [\vz_{n+1}; \vu_{n+1}].
			\label{eqn:pr}
		\end{aligned}
		\end{equation}
		 \pr splitting converges to $\vz^*$ for any $\alpha > 0$ \citep{ryu2016primer}.%
\end{itemize}

While there exist many similar strategies, we restrict our discussion to the above examples. For both, we initialize $\vs_{0} = \0$ and write $\g(\vx,\vs_{n})$ for one iteration and $[\vz; \vu] \gets \vs$ for the unpacking of the latent state for both \pr and \fwdbwd. For the latter, we simply assume $\vu_{n}$ to be zero-dimensional.

Both Forward-Backward (\fwdbwd, \cref{eqn:fwd_bwd}) and Peaceman-Rachford splitting (\pr, \cref{eqn:pr}) are guaranteed to converge to the fixpoint of $\vf(\vx,\vz)$, as defined in \cref{eqn:deq}. %
In practice, they are iterated until $\|\vz_{n}-\vz_{n-1}\|$ becomes smaller than a predetermined stopping criterion, yielding $\vz_{n} \approx \vz^*(\vx)$.

\paragraph{Example (cont.)}
Our example from \cref{eq:example} is a monDEQ using \fwdbwd splitting and parametrized with:
\begin{align*}	
	&m = 4,\; \alpha = \tfrac{1}{10},\;
	P = \begin{mtx}1 & 0\\ 0 & 1\end{mtx},\;
	Q = \begin{mtx}1 & 0\\ 1 & 0\end{mtx}\;
    W = \begin{mtx}-4 & -1\\\phantom{+}1 & -4\end{mtx},\; 
    U = \begin{mtx}\phantom{+}1 & 1\\-1 & 1\end{mtx},\;
    b = \begin{mtx}0\\0\end{mtx}.
\end{align*}
For these parameters, the iterative functions are
\begin{align*}	
	\vf(\vx, \vz_n) &= ReLU \left(
	\begin{mtx}-4 & -1\\\phantom{+}1 & -4\end{mtx} \vz_n + \begin{mtx}\phantom{+}1 & 1\\-1 & 1\end{mtx}\vx + \begin{mtx}0\\0\end{mtx}
	\right)\\
	\vg_{\alpha}(\vx, \vs_n) &= ReLU \left( \tfrac{1}{10} (\mW + 9\mI) \vs_n + \tfrac{1}{10}  \begin{mtx}\phantom{+}1 & 1\\-1 & 1\end{mtx}\vx + \begin{mtx}0\\0\end{mtx} \right) = ReLU \left(  \tfrac{1}{10} \begin{mtx}\phantom{+}5 & 1\\-1 & 5\end{mtx}   \vs_n + \tfrac{1}{10}  \begin{mtx}\phantom{+}1 & 1\\-1 & 1\end{mtx}\vx \right).
\end{align*}
Observe that $0 < \alpha = \tfrac{1}{10} < \tfrac{2m}{\| \mI - \mW \|_{2}^{2}} \approx 0.1538$. %
Interestingly, directly iterating $\vf(\vx, \vz_n)$ diverges in our example, highlighting the importance of a suitable iterative solver.

\subsection{Abstract Interpretation of monDEQs} \label{sec:tool}

Equipped with the building blocks discussed so far, we now introduce our abstract interpreter, \tool (\toollong).

At a high level, given an input $\vx$, a precondition \pre, and a postcondition \post over a monDEQ $\vh$, \tool iteratively applies an abstract solver iteration \gS to \domain abstractions of the input and solver state until our inclusion check (\cref{the:containment}) can show that the resulting state $\hat{\S}_{n+1}$ is contained in the previous one $\hat{\S}_n$. By the contraction-based termination condition of \cref{the:contraction}, we have thus found an over-approximation $\hat{\Z}^*$ of the true fixpoint set $\Z^{*}$.
Propagating the corresponding \domain $\hat{\Z}^*$ through the last layer to obtain the \domain abstraction of the output $\hat{\Y} = \mV \hat{\Z}^{*} + \vv$ (by slight abuse of notation), we check the postcondition $\post(\hat{\Y})$.
Below, we first discuss this process, outlined in \cref{alg:verif}, informally, before formally showing its correctness.

While \tool is applicable to general pre- and postconditions \pre and \post, the presented version assumes that $\vx \in \pre(\vx)$ and that \post is a statement over the outputs of the monDEQ $\vh$.
In particular, we focus on $\ell_\infty$ robustness certification, where
$\pre(\vx) \coloneqq \{ \vx' \mid \|\vx - \vx'\|_\infty \leq \eps \}$ and
$\post \coloneqq \vh_t(\vx') - \vh_i(\vx') > 0, \forall i \neq t$. That is, the monDEQ $\vh$ classifies all $\vx'$ in an $\ell_\infty$-ball around $\vx$ as class $t$.

\vspace{-1mm}
\paragraph{What to Certify}
As previously discussed and in agreement with prior work \citep{chen2021semialgebraic,ELMonDEQPabbarajuWK21}, \tool certifies properties for the true mathematical fixpoints $\vz^{*}$, rather than any particular solver behavior. %
This yields stronger certificates as
solvers are guaranteed to converge to these unique fixpoints with arbitrary precision.

\begin{wrapfigure}[18]{r}{0.50\textwidth}
	\vspace{-5mm}
	\setlength{\textfloatsep}{2pt}
	\scalebox{0.82}{
		\begin{minipage}{1.2\linewidth}
		\input{algorithm_tool}
		\end{minipage}
	}
	\setlength{\textfloatsep}{10pt plus 1.0pt minus 2.0pt}
\end{wrapfigure}
\paragraph{\tool}
\tool can be divided into two stages: First, it leverages the contraction-based termination condition of \cref{the:contraction} to compute a first abstraction of the fixpoint set for a given precondition \pre. Second, it tightens this abstraction by leveraging fixpoint set preservation (\cref{thm:preserving}) to show the postcondition \post. 
This is detailed in \cref{alg:verif}. 
We start by initializing, by slight abuse of notation, $\hat{\Z}_{0} = \hat{\U}_{0} = \{\vz^{*}(\vx)\}$ to a concrete fixpoint (line~\ref{alg:verif:init}). %
To compute an abstract fixpoint set via \cref{the:contraction}, we perform iterations of the $\lnot \con$ branch (lines~\ref{alg:verif:consolidate}-\ref{alg:verif:check}). In each iteration, we first consolidate the errors of the current abstraction (line~\ref{alg:verif:consolidate}) via \cref{thm:consolidation}, then perform one step of \gSss{}{1} (line~\ref{alg:verif:g1}), and finally check inclusion $\hat{\S}_{n+1} \sqsubseteq \hat{\S}_{n}$ via \cref{the:containment} (line~\ref{alg:verif:check}). 
After detecting containment, we aim to tighten the thus obtained fixpoint abstraction in order to show \post. To this end, we perform iterations of the $\con$ branch (lines \ref{alg:verif:g2}-\ref{alg:verif:post_r}). Here, we apply \gSss{}{2} (line~\ref{alg:verif:g2}) and the classification layer (line~\ref{alg:verif:logits}) before checking \post on the resulting \domain (line~\ref{alg:verif:post}).

\paragraph{Iterator Requirements}
While \cref{alg:verif} does not require a specific operator splitting method $\g$, \gSss{}{1} has to be chosen such that \cref{the:contraction} on the contraction based termination condition is applicable and \gSss{}{2} has to be fixpoint-preserving according to \cref{thm:preserving}. 

More concretely, we require the abstract transformer \gSss{}{1} (line~\ref{alg:verif:g1}) to be a sound abstraction of an operator splitting method $\vg_{\alpha_1}$ which, in the concrete, is guaranteed to converge to a unique fixpoint in finitely many steps. %
And we require \gSss{}{2} (line~\ref{alg:verif:g2}) to be fixpoint set preserving (\cref{thm:preserving}), i.e., to map fixpoints upon themselves. By \cref{the:iter_PR} the latter is the case for all $\vg_{\alpha_2}^\#$ abstracting a fixed, locally-Lipschitz operator splitting method $\vg_{\alpha_2}$ with convergence guarantees, including \pr and \fwdbwd. 
However, the following points have to be considered:
\pr iterations use auxiliary variables $\vu$ which depend on $\alpha$. Consequently, fixpoint set preservation is only guaranteed for a fixed $\alpha$, preventing us from optimizing $\alpha$ to obtain tighter over-approximations. 
This limitation does not apply to \fwdbwd splitting, as it does not use any auxiliary variables. However, the convergence requirement of \cref{the:iter_PR} still limits $\alpha$ to $0 < \alpha < 2m/\| \mI - \mW \|_{2}^{2}$. We now lift this restriction (again deferring a formal proof to \cref{sec:proofs}), allowing us to apply further iterations of $\gfwdbwd$ with arbitrary $\alpha \in [0,1]$ to tighten $\hat{\S_{n}}$ after showing containment with any method:
\begin{restatable}[Fixpoint set preservation for \fwdbwd splitting]{theorem}{iterfwbw}
	\label{the:iter_FWBW}
	Every sound abstract transformer $\gfwdbwdS$ of $\gfwdbwd$ is fixpoint set preserving for $0 \leq \alpha \leq 1$.
\end{restatable}

Intuitively, we show that $\gfwdbwd$ maps all concrete fixpoints onto themselves and hence that any sound abstract transformer $\gfwdbwdS$ will map an over-approximation of the fixpoint set to another over-approximation of the fixpoint set. Please see \cref{sec:proofs} for a formal proof.
In contrast to \cref{the:contraction}, \cref{the:iter_FWBW} does not assume that the same iterative solver (including hyperparameters) is applied at each step. Instead, it makes a statement about one application of Forward-Backward splitting using arbitrary parameters.
This result allows us to apply further iterations of $\gfwdbwd$ with arbitrary $\alpha \in [0,1]$ to tighten $\hat{\S_{n}}$ after showing containment with any method.

For \fwdbwd, we show an even stronger property in \cref{the:iter_FWBW}, guaranteeing fixpoint preservation even under changing $\alpha$.

\paragraph{Choice of Iterator}
Given the constraints discussed above, we choose different algorithms for \gSss{}{1} and \gSss{}{2}, optimizing for containment and tight final abstractions, respectively.
For \gSss{}{1} we typically use \pr, as it empirically is significantly less sensitive to hyperparameter choices (see \cref{fig:alpha_ablation}) and contracts to the actual fixpoint set more quickly \citep{MonDEQWinstonK20}.
For \gSss{}{2}, both \pr \gprS and \fwdbwd \gfwdbwdS are used depending on the underlying problem.
In some settings the stronger contractive properties of \pr yield tighter abstractions while in others choosing an optimal dampening parameter $\alpha$ via line search for \fwdbwd works better.
We further discuss this in \cref{sec:eval:ablation}.

\paragraph{Expansion}
The key to showing containment is not the absolute tightness of the abstract iteration state $\hat{\S}_n$, but rather how much it tightens under application of \gSss{}{1}. As further tightening an already very tight approximation can be challenging, we -- perhaps counter-intuitively -- expand our over-approximation as part of the error consolidation by setting
\begin{equation}
c = (1 + w_{mul}) |\tilde{\mA}^{-1} \mA| \mbf{1} + w_{add} \mbf{1}
\end{equation}
in \cref{eq:consolidation} until containment is found. Here, $w_{mul}, w_{add} \geq 0$ are the multiplicative and additive expansion parameters, respectively. The resulting looseness between the current approximation and the exact fixpoint set can make tightening the approximation and hence showing containment easier. %
As this expansion leads to a strictly larger over-approximation, it is a sound operation.

While this is similar to widening \citep{CousotC92b} at first glance, it aims to break a non-monotonic iteration of incomparable abstractions instead of ensuring termination of an otherwise infinite iteration of monotonically increasing abstractions.

\paragraph{Correctness}
We now show the correctness of \tool w.r.t. to the concrete and abstract semantics defined in \cref{sec:abstract-fp}, instantiated for monDEQs.
\begin{restatable}[Soundness of \tool]{theorem}{soundness}
	\label{thm:soundness}
	For sound \gSss{}{1} fulfilling \cref{the:contraction} and sound
	\gSss{}{2} fulfilling fixpoint-preservation (\cref{thm:preserving}),
	\cref{alg:verif} is sound. In particular:
\begin{enumerate}
	\item Once $\con$ ($\hat{\S}_{n+1} \sqsubseteq \hat{\S}_{n}$), $\hat{\S}_{n+1}$ contains the true fixpoint set.
	\item \cref{alg:verif} returns \texttt{true} only if $\pre(\vx) \models \post(\vh(\vx))$.
\end{enumerate}
\end{restatable}
\begin{proof}
	(1) follows directly from the soundness of the containment check (\cref{the:containment}) and the contraction-based termination criterion of \cref{the:contraction} and (2) from \cref{the:iter_FWBW,the:iter_PR} and the use of sound abstract transformers.
\end{proof}

\paragraph{Completeness} While \tool is sound, it is not complete. In particular, $\consolidate$, $\expand$, \gSss{}{1}, and \gSss{}{2} are sources of imprecision. The inclusion check is also sound, but not complete.

\paragraph{Generality}
\tool can be instantiated with any abstract domain supporting the required abstract transformers. Only the consolidation in line~\ref{alg:verif:consolidate} is specific to \domain and it can be removed without affecting \tool's soundness. However, while all domains discussed in \cref{sec:overview} possess the required transformers, only \domain combines sufficient precision with tractable containment checks and efficient propagation (see \cref{sec:eval:ch_zono}).

\section{Experimental Evaluation} \label{sec:eval}
In this section, we present an extensive evaluation of \tool, the implementation of our abstraction framework and the \domain domain, on monDEQs using multiple architectures and datasets including \cifar \citep{krizhevsky2009learning}, \mnist \citep{lecun1998gradient}, and HCAS \citep{julian2019guaranteeing}. First, we evaluate \tool in the setting of local robustness certification against the challenging $\ell_\infty$-perturbations (\mnist and \cifar). There, we demonstrate that \tool outperforms the current state-of-the-art in scalability, speed, and precision. Second, we show in the HCAS setting that \tool is also suitable for deriving global guarantees. Third, we investigate the impact of different algorithmic components in an ablation study. Finally, we demonstrate \tool's broader applicability on a numerical program.

\paragraph{Experimental Setup}
We implement \tool in PyTorch \citep{PyTorch} and evaluate it on single Nvidia TITAN RTX using a 16-core Intel Xeon Gold 6242 CPU at 2.80GHz. For implementation and experimental details as well as (hyper)parameter choices, please see \cref{sec:impl-deta,sec:parameter-choices} as well as the detailed description and full code in our artifact.

\paragraph{Implementation Details}
The version of \cref{alg:verif} presented here is slightly simplified for the sake of clarity.
We discuss additional engineering considerations in \cref{sec:impl-deta}. 
These implementation details, however, impact neither the soundness of the algorithm nor the intuitions outlined here.

\begin{table*}[t]
	\renewcommand{\arraystretch}{1.2}
	\centering
	\caption{Overview of the obtained natural accuracy (\textit{Acc.}), adversarial accuracy (\textit{Bound}), the number of samples for which we found a fixpoint over-approximation (\textit{Cont.}), the certified accuracy (\textit{Cert.}), and the average time per correctly classified sample for the first 100 samples from the corresponding test set.}
	\vspace{-3mm}
	\label{tab:main_results}
	\footnotesize
	\scalebox{0.95}{
		\renewcommand{\arraystretch}{1.0}
		\begin{threeparttable}
			\centering
			\begin{tabular}{@{}lllrrr rrr }
				\toprule
				\textit{Dataset} & \textit{Model} & \textit{Latent Size} & \textit{\# Acc.} & \textit{$\epsilon$} & \textit{\# Bound} & \textit{\# Cont.} & \textit{\# Cert.} & \textit{Time [s]} \\
				\midrule
				\mnist
				& \fcf & 40 		& 99 & 0.05 & 70 & 100 & 36 & 17.2 \\
				& \fces & 87 		& 99 & 0.05 & 75 & 100 & 30 & 15.8 \\
				& \fco & 100 	& 96 & 0.05 & 73 & 100 & 24 & 13.2 \\
				& \fct & 200 	& 99 & 0.05 & 80 & 100 & 26	& 14.0 \\
				& \convsm & 648 & 97 & 0.05 & 80 & 100 & 68 & 22.4 \\
				\midrule
				\cifar
				& \fct & 200 	& 63  	& 2/255 & 36 & 100 & 22 & 16.8 \\
				& \convsm & 800 & 55   	& 2/255 & 32 & 100  & 29 & 41.1 \\
				\bottomrule
			\end{tabular}
		\end{threeparttable}
	}
	\label{Ta:Relu}
	\vspace{-4mm}
\end{table*}

\subsection{Local Robustness Certification with \tool}
Similar to prior work \citep{chen2021semialgebraic}, we evaluate the first 100 test set samples and report the mean runtime for correctly classified samples (Time), the certified accuracy (Cert.), and the number of samples for which we found an abstract post-fixpoint (Cont.).

In \cref{tab:main_results}, we show results for a range of fully connected and convolutional monDEQs. There, \textit{\#Bound} denotes the number of samples empirically robust to PGD attacks \citep{MadryMSTV18} and constitutes an upper bound to the certified accuracy, see \cref{sec:adversarial-attack} for details.
We generally observe that smaller fully-connected networks have lower empirical robustness but are easier to certify, with the smallest network yielding the highest certified accuracy.
Surprisingly, we find that on both \mnist and \cifar, convolutional networks are comparatively easy to verify, yielding the highest certified accuracies.

\paragraph{Comparison with \semi}
\citet{chen2021semialgebraic} introduce three models suitable for robustness certification. In \cref{tab:semisdp_comp}, we compare against the (by far) most precise of these approaches, the `Robustness Model' (\semi), which is the current state-of-the-art for verifying $\ell_\infty$ robustness properties for monDEQs. As the underlying SDP solver limits \semi to \mnist networks with a latent space size of at most $87$ neurons \citep{chen2021semialgebraic}, we compare to them only on our two smallest networks: their \fces and our \fcf. %
For the smallest perturbations of $\eps = 0.01$, both tools are able to certify (almost) all empirically robust samples, with \semi failing to certify one sample on \fces. However, while \tool requires only around $1$s per sample, \semi takes three to four orders of magnitude longer ($401.5$s and $1388.1$s). For larger perturbation magnitudes $\eps \in \{0.02, 0.05, 0.07\}$, \tool is consistently more precise and much faster, certifying up to $100\%$ more samples (36 vs 18 for \fcf at $\epsilon=0.05$) with around two orders of magnitude shorter average runtime.
For $\eps=0.1$, few samples are empirically robust and neither tool can verify robustness for any on either network.
Finally, as shown in \cref{tab:main_results}, \tool scales to much larger networks ($10$x) and more challenging datasets (\cifar) than \semi. The two alternative certification models proposed by \citet{chen2021semialgebraic}, the `Lipschitz Model' and the `Ellipsoid Model', are significantly less precise, verifying no property at all for $\epsilon=0.05$ and \fces. Thus we omit a detailed comparison.

\begin{table}[t]
	\renewcommand{\arraystretch}{0.97}
	\centering
	\caption{Comparison of \tool to the `Robustness Model' (\semi) of \citet{chen2021semialgebraic}.}
	\vspace{-4mm}
	\footnotesize
	\scalebox{0.98}{
		\renewcommand{\arraystretch}{0.95}
		\centering
		\begin{tabular}{llllrrrrr}
			\toprule
			\multirow{2.2}{*}{\textit{Model}} & \multirow{2.2}{*}{\textit{Latent Size}} & \multirow{2.2}{*}{\textit{\# Acc.}}& \multirow{2.2}{*}{\textit{$\epsilon$}} & \multirow{2.2}{*}{\textit{\# Bound}} & \multicolumn{2}{c}{\semi} & \multicolumn{2}{c}{\tool (ours)} \\
			\cmidrule(lr){6-7}   \cmidrule(lr){8-9}
			&&&&& \# Cert. & Time [s] & \# Cert. & Time [s] \\
			\midrule
			\multirow{5}{*}{\fcf} & \multirow{5}{*}{40} & \multirow{5}{*}{99}
			&  0.01 & 98 & \textbf{98} & 401.5 & \textbf{98} & \textbf{0.97} \\
			&&&0.02 & 95 & 88 & 357.7 & \textbf{94} & \textbf{8.82} \\
			&&&0.05 & 70 & 18 & 196.4 & \textbf{36} & \textbf{17.19} \\
			&&&0.07 & 29 & 5  & 121.0 & \textbf{8}  & \textbf{21.25} \\
			&&&0.10 & 10 & 0 & 63.0 & 0 & 12.88 \\
			\cmidrule(lr){1-3}
			\multirow{5}{*}{\fces}& \multirow{5}{*}{87} & \multirow{5}{*}{99}
			&  0.01 & 99 & 98 & 1388.1 & \textbf{99} & \textbf{1.40} \\
			&&&0.02 & 98 & 92 & 1186.8 & \textbf{98} & \textbf{2.66}  \\
			&&&0.05 & 75 & 24 & 599.9 & \textbf{30} & \textbf{15.75}  \\
			&&&0.07 & 42 & \textbf{5} & 387.6 & \textbf{5}  & \textbf{14.53} \\
			&&&0.10 & 8  & 0 & 214.46 & 0  & 9.75 \\
			\bottomrule
		\end{tabular}
	}
	\label{tab:semisdp_comp}
	\vspace{-4mm}
\end{table}

\paragraph{Comparison with Lipschitz-Bound-Based Methods}
Three existing works derive Lipschitz-Bounds for monDEQs, either via a posteriori analysis \citep{ELMonDEQPabbarajuWK21,chen2021semialgebraic} or construction \citep{LMonDEQRevay2020}. However, they all obtain significantly lower certified accuracies.

\begin{wrapfigure}[9]{r}{0.44\textwidth}
	\centering
	\vspace{-4mm}
	\scalebox{0.8}{
	\def\layersep{1.6cm}
\def\layerseP{1.12cm}

\begin{tikzpicture}[scale=0.7]

\coordinate (Own) at (0,0);
\coordinate (Int) at (5.5,2.5);

\node[aircraft top,fill=black,draw=white,minimum width=1cm,rotate=0,scale = 0.85] (Own) at (0,0) {} node [below,yshift=-0.5cm,font=\footnotesize] {Ownship};
\coordinate[label=right:$v_\text{own}$] (S0) at ($(Own)+(1.5,0)$);

\node at (Int) [above,xshift=1.1cm,font=\footnotesize] {Intruder};
\coordinate[] (IntN) at ($(Int)+(1.5,0)$);
\coordinate[label=below:$v_\text{int}$] (S1) at ($(Int)+(0,-1.5)$);

\draw [thick, ->] (Own) -- (S0);
\draw [dashed,->] (Int) -- (IntN);
\draw [thick, ->] (Int) -- (S1);

\def \xo{-0.5}
\path let \p1 = (Int) in coordinate (IntO) at (\x1,\xo);
\path let \p1 = (Own) in coordinate (OwnO) at (\x1,\xo);
\draw [dashed, |->|] (OwnO) -- (IntO);

\node[] () at ($(OwnO)!0.5!(IntO)+(0,-0.3)$) {$x$};

\def \yo{-1.0}
\path let \p1 = (Int) in coordinate (IntO) at (\yo,\y1);
\path let \p1 = (Own) in coordinate (OwnO) at (\yo,\y1);
\draw [dashed, |->|] (OwnO) -- (IntO);
\node[] () at ($(OwnO)!0.5!(IntO)+(0.3,0)$) {$y$};

\draw[dashed] (Own) -- (Int);

\pic [draw,stealth-,red,thick,dashed,angle radius=0.7cm,"$-\vartheta$"{anchor=west,text = black, below}, angle eccentricity=1] {angle = S1--Int--IntN};

\node[aircraft top,fill=black,draw=gray,minimum width=1cm,rotate=00,scale = 0.85]  at (Own) {};
\node[aircraft top,fill=red,draw=black,minimum width=1cm,rotate=-90, scale = 0.85] at (Int) {};

\end{tikzpicture}
	}
	\vspace{-4mm}
	\caption{Visualization of the HCAS Geometry. Adapted from \citet{julian2019guaranteeing}.}
	\label{fig:HorizontalCAS}
\end{wrapfigure}
\subsection{Global Robustness Certification with \tool}
To demonstrate that \tool is also suitable for computing global robustness certificates, we analyze the HCAS (Horizontal Collision Avoidance System) setting which is illustrated in \cref{fig:HorizontalCAS} and has been used as a benchmark for feed-forward networks in the past \citep{julian2019guaranteeing,Fu2021Repair}. Given the relative position ($x$- and $y$-coordinates) and heading ($\vartheta$) of an intruder aircraft (shown in red) with respect to one's own position and heading (shown in black), one of five action recommendations (COC - Clear of Conflict, WL/WR - Weak Left/Right, SL/SR - Strong Left/Right) is given. The training data is generated by framing this as a Markov Decision Process (MDP) and solving it for many parameters, yielding a large look-up table (see \citet{julian2019guaranteeing} for more details). We train a monDEQ (\fco) on this large and discrete tabular dataset to obtain a continuous and compressed mapping.

\begin{wrapfigure}[15]{r}{0.54 \textwidth}
	\centering
	\vspace{-1mm}
	\scalebox{0.98}{\input{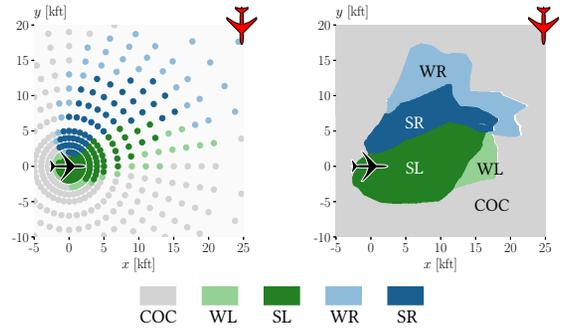}}
	\vspace{-7mm}
	\caption{HCAS policy training data (left) for ($\vartheta = -90$°) and verified monDEQ prediction (right) ($\vartheta \in [-90.5\text{°}, -89.5\text{°}]$). The colored regions are certified to yield the indicated recommendation. No certificate for the white regions.}
	\vspace{-2mm}
	\label{fig:hcas_slice}
\end{wrapfigure}
To confidently use this monDEQ representation, we aim to certify that it yields consistent predictions across large regions of the input space.
Using \tool, we apply a domain splitting approach \citep{Wang19Formal} in order to exhaustively certify decisions for the whole input space.%
This way, we can certify the prediction on $82.8$\% of the relevant input region. %
For visualization, we pick a thin slice of this space %
and visualize the resulting certified decision regions (right) and the corresponding tabular data (left) in \cref{fig:hcas_slice}. There, regions for which we obtain a certificate are colored depending on the action recommended, and regions for which no certificate is obtained are shown in white. We observe that, as expected, the regions directly at the decision boundary can not be certified. However, we also observe a small unexpected pocket of non-certifiable decisions where a strong right is certifiably recommended all around.

\begin{figure}[t]
	\centering
	\begin{minipage}{1.02 \textwidth}
		\centering
		\begin{subfigure}[t]{.19\textwidth}
			\centering
			\includegraphics[width=.98\linewidth]{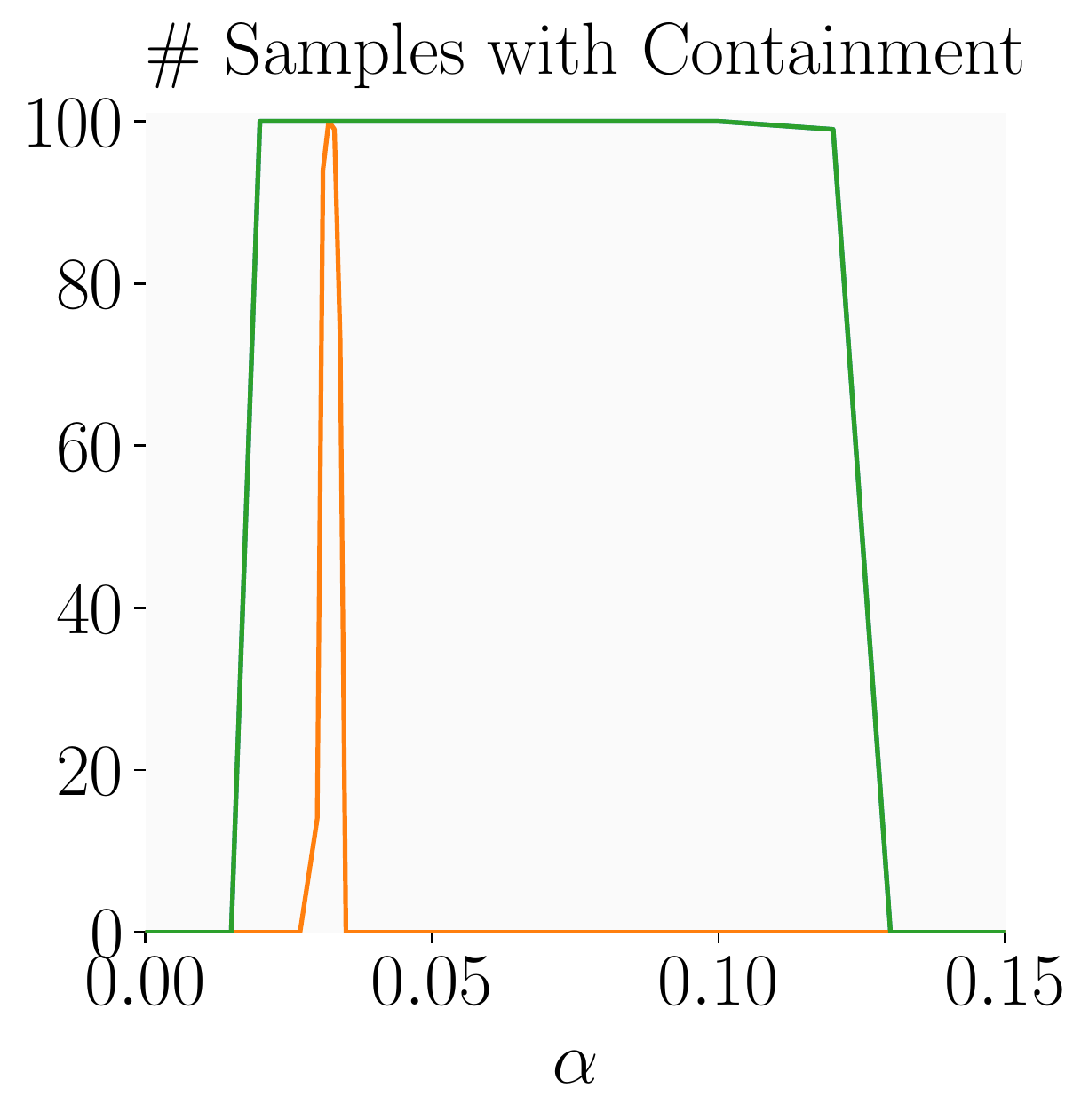}
			\caption{\footnotesize Containment}
			\label{fig:alpha_ablation_c}
		\end{subfigure}
		\hfil
		\begin{subfigure}[t]{.20\textwidth}
			\centering
			\includegraphics[width=.93\linewidth]{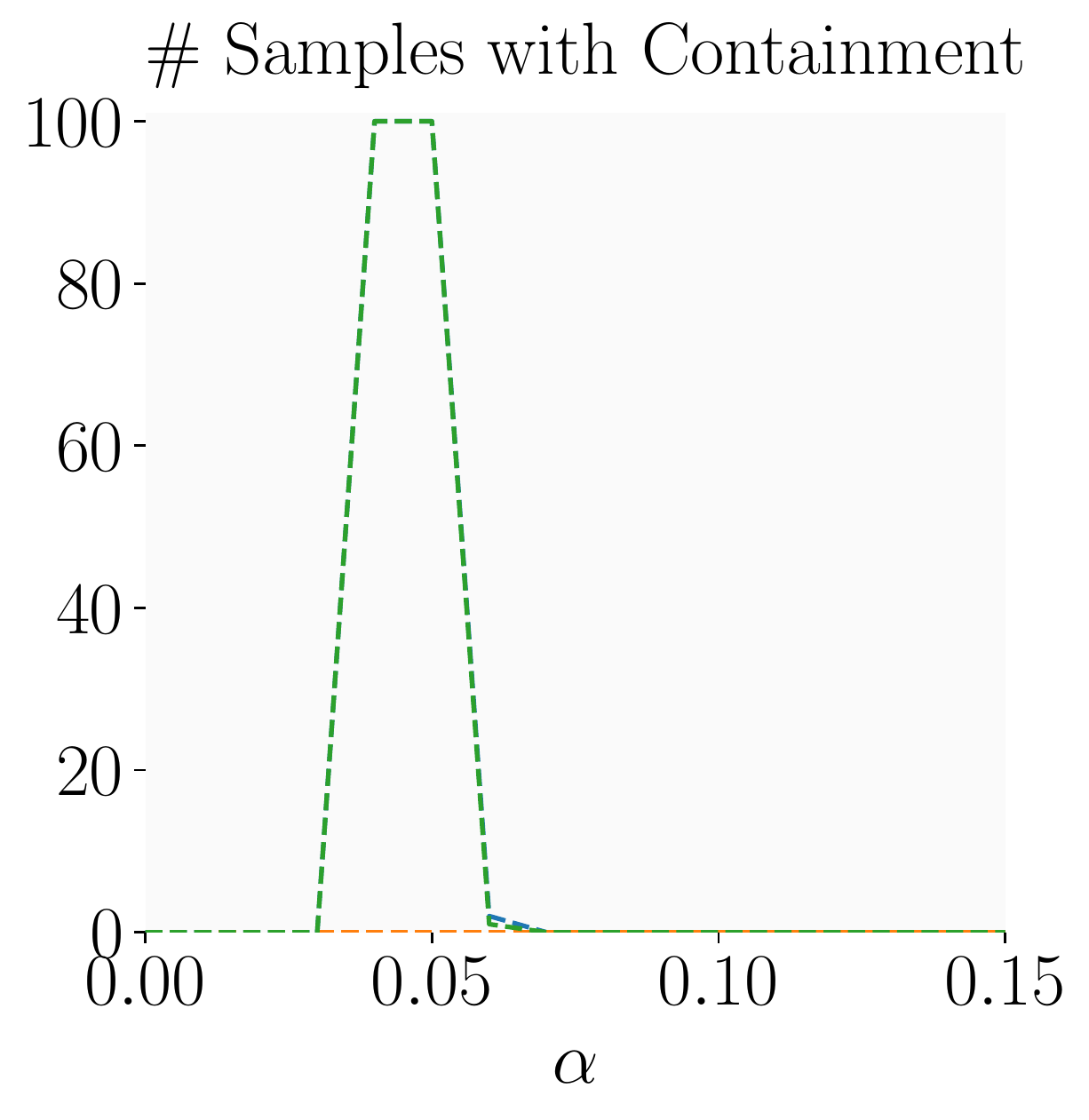}
			\caption{\footnotesize Containment w/o box}
			\label{fig:alpha_ablation_c_nb}
		\end{subfigure}
		\hfil
		\begin{subfigure}[t]{.19\textwidth}
			\centering
			\includegraphics[width=0.98\linewidth]{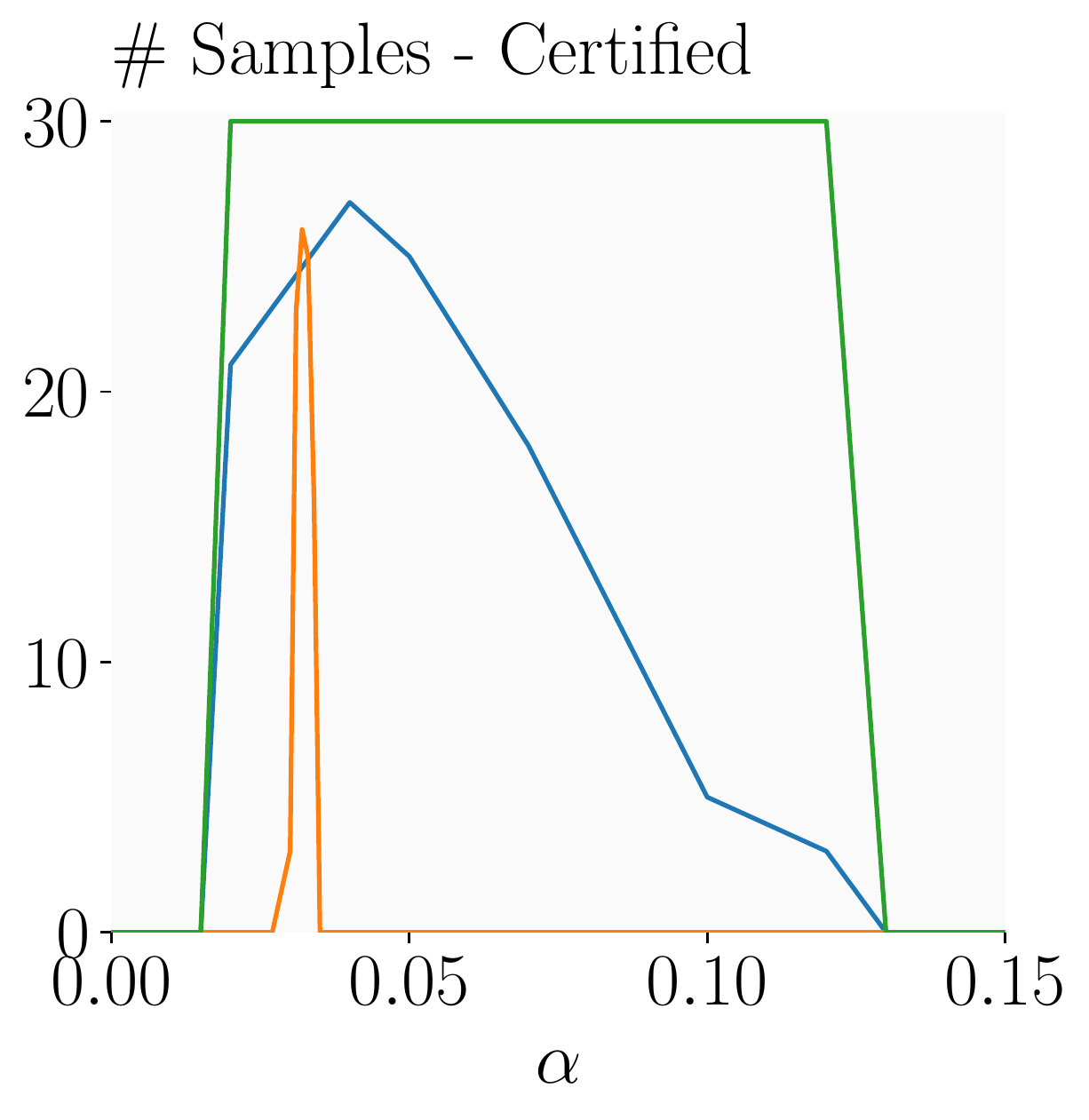}
			\caption{\footnotesize Certification}
			\label{fig:alpha_ablation_v}
		\end{subfigure}
		\hfil
		\begin{subfigure}[t]{.37\textwidth}
			\centering
			\includegraphics[width=1.\linewidth]{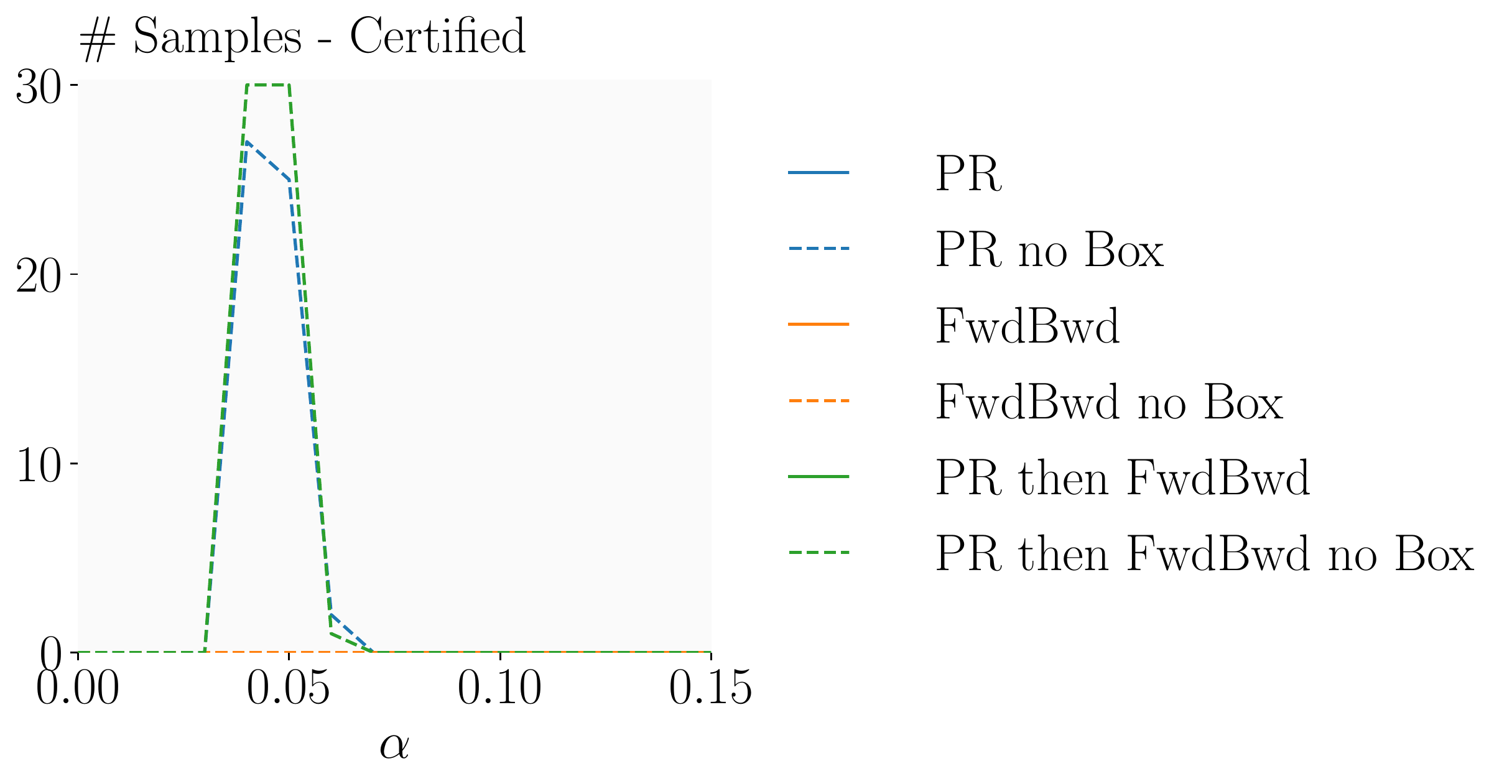}
			\caption{\footnotesize Certification w/o box$\qquad\qquad\qquad\qquad$}
			\label{fig:alpha_ablation_v_nb}
		\end{subfigure}
	\end{minipage}
	
	\vspace{-3mm}
	\caption{Illustration of the stability ranges for $\alpha$, depending on the use of the box component, and fixpoint solver. Note that the blue (\pr) and green (\pr then \fwdbwd) lines are identical in the two containment plots (left).}
	\vspace{-3mm}
	\label{fig:alpha_ablation}
\end{figure}

\subsection{Ablation Study on \tool} \label{sec:eval:ablation}
\begin{wraptable}[16]{r}{0.48\textwidth}
	\vspace{-4.5mm}
	\renewcommand{\arraystretch}{1.2}
	\caption{Overview of the natural accuracy (\textit{Acc.}), the number of samples for which the fixpoint set iteration converged (\textit{Cont.}) the certified accuracy (\textit{Cert.}), and the average time per sample on \fces.}
	\vspace{-2mm}
		\scalebox{0.81}{
	\begin{threeparttable}
		\centering
		\begin{tabular}{lrrr }
			\toprule
			\textit{Ablation} & \textit{\# Cont.}& \textit{\# Cert.} & \textit{Time [s]} \\
			\midrule
			Reference                   & 100 & \textbf{30} & 17.48 \\
			No Zono component 			& 100 & 0  & 0.38 \\
			No Box component  			& 100 & \textbf{30} & 23.18 \\
			Only \pr                     & 100 & 27 & \textbf{4.10} \\
			Only \fwdbwd			        & 100 & 26$^\dagger$\hspace{-1.5mm} & 7.99 \\
			No $\lambda$ optimization		& 100 & 24 & 7.81\\
			Reduced $\lambda$ optimization	& 100 & 27 & 13.85\\
			Same iter. containment       & 100 & 0  & 7.14\\
			No Expansion       			 & 50  & 9  & 18.89\\
			\bottomrule
		\end{tabular}
		\begin{tablenotes}
			\item[]{$^\dagger$ No formal guarantee as conditions for \cref{the:contraction} are not satisfied.}
		\end{tablenotes}
	\end{threeparttable}
		}
	\label{tab:ablation_study}
	\vspace{-3mm}
\end{wraptable}
We conduct an extensive ablation study on the key features of \tool and report results for \fces in  \cref{tab:ablation_study} and \cref{fig:alpha_ablation}, deferring additional results to \cref{app:ablation}.

\paragraph{\domain}
We analyze the effectiveness of our domain by setting either $\vb = 0$ (no Box) or $\mA = \0$ (no Zono). Disallowing the Zonotope component leaves a standard Box, which converges quickly, but fails to prove any property (see \cref{tab:ablation_study}). Disallowing the Box component, leaving a \domain which still utilizes error consolidation rather than a standard Zonotope, can yield the same precision but significantly reduces the range of dampening parameters $\alpha$ leading to convergence (compare \cref{fig:alpha_ablation_c} and \cref{fig:alpha_ablation_c_nb}), to the point where for some solvers and networks (e.g. \fwdbwd and \fces) we were unable to find such an $\alpha$.

\paragraph{Iteration Method}
As discussed in \cref{sec:tool}, we can choose different operator splitting methods for the containment-finding and the tightening phase of \tool. 
When using only \fwdbwd, the $\alpha$ range for which we can detect containment is extremely narrow (see \cref{fig:alpha_ablation_c}) and does not overlap the region $0 < \alpha < 2m/$\mbox{$\| \mI - \mW \|_{2}^{2} = 0.0125$} for which we have convergence guarantees in the concrete. This is problematic, as these guarantees are a condition for \cref{the:contraction} and thus our formal soundness guarantee.
Using \pr until we find containment and then \fwdbwd avoids this issue, is significantly more robust to the choice of $\alpha$, and yields the tightest abstractions of all three approaches, leading to the most certified properties (see \cref{fig:alpha_ablation_v}).
Only using \pr leads to slightly less precise final abstractions and thus worse certification performance.
First using \fwdbwd and then \pr is not supported by \cref{the:iter_PR}, as we would not have computed $\hat{\U}^*$.
We thus use first \pr and then \fwdbwd, for all other experiments. While we fix $\alpha_{1}$ for \pr, we choose $\alpha_2$ for \fwdbwd adaptively. 
See \cref{app:ablation_adaptive_alpha} for more details and a corresponding ablation study.

\paragraph{Transformer Optimization}
Recall that the abstract ReLU transformer has a parametrizable slope $\lambda$, which can be optimized to tighten our final abstractions \citep{WongK18,WengZCSHDBD18,zhang2018crown} by unrolling several iterations of the solver and using (projected) gradient descent to optimize $\lambda$ individually for each of these iterations.
We distinguish three settings, `No $\lambda$ optimization', `Reduced $\lambda$ optimization', and `Reference', where we unroll no, $20$, and $40$ iterations and optimize lambda over no, $60$, and $200$ gradient steps, respectively.
We only perform this optimization for samples that are already close to being certified, allowing us to certify six additional samples while only increasing the mean certification time by $10$s (see \cref{tab:ablation_study}). For more details, see \cref{sec:impl-deta}.

\paragraph{Same Iteration Containment}
To demonstrate the value of fixpoint set preservation (\cref{thm:preserving,the:iter_FWBW,the:iter_PR}), we consider the setting `Same iter. containment`, where we always require the abstraction $\hat{\S}_{n+1}$ that we use to certify the postcondition to be contained in its predecessor $\hat{\S}_{n}$. In this setting, we are not able to certify a single property (see \cref{tab:ablation_study}), as we can only detect strictly smaller abstractions if the current abstraction is still relatively loose. %

\paragraph{Expansion}
To illustrate the effect of artificially expanding our abstractions as part of error consolidation (see \cref{sec:tool}), we consider 'No Expansion' in \cref{tab:ablation_study}, where we turn expansion off by setting $w_{mul}$ and $w_{add}$ to $0$. Most notably, for $50\%$ of samples, we do not detect abstraction containment and thus do not obtain sound fixpoint abstractions at all. Further, even if we obtain fixpoint abstractions, we often do not certify the corresponding sample. 

\subsection{Effectiveness of \domain}\label{sec:eval:ch_zono}
\begin{wrapfigure}[10]{r}{0.56 \textwidth}
	\vspace{-10.5mm}
	\begin{subfigure}{.48\linewidth}
		\centering
		\includegraphics[width=1.0\linewidth]{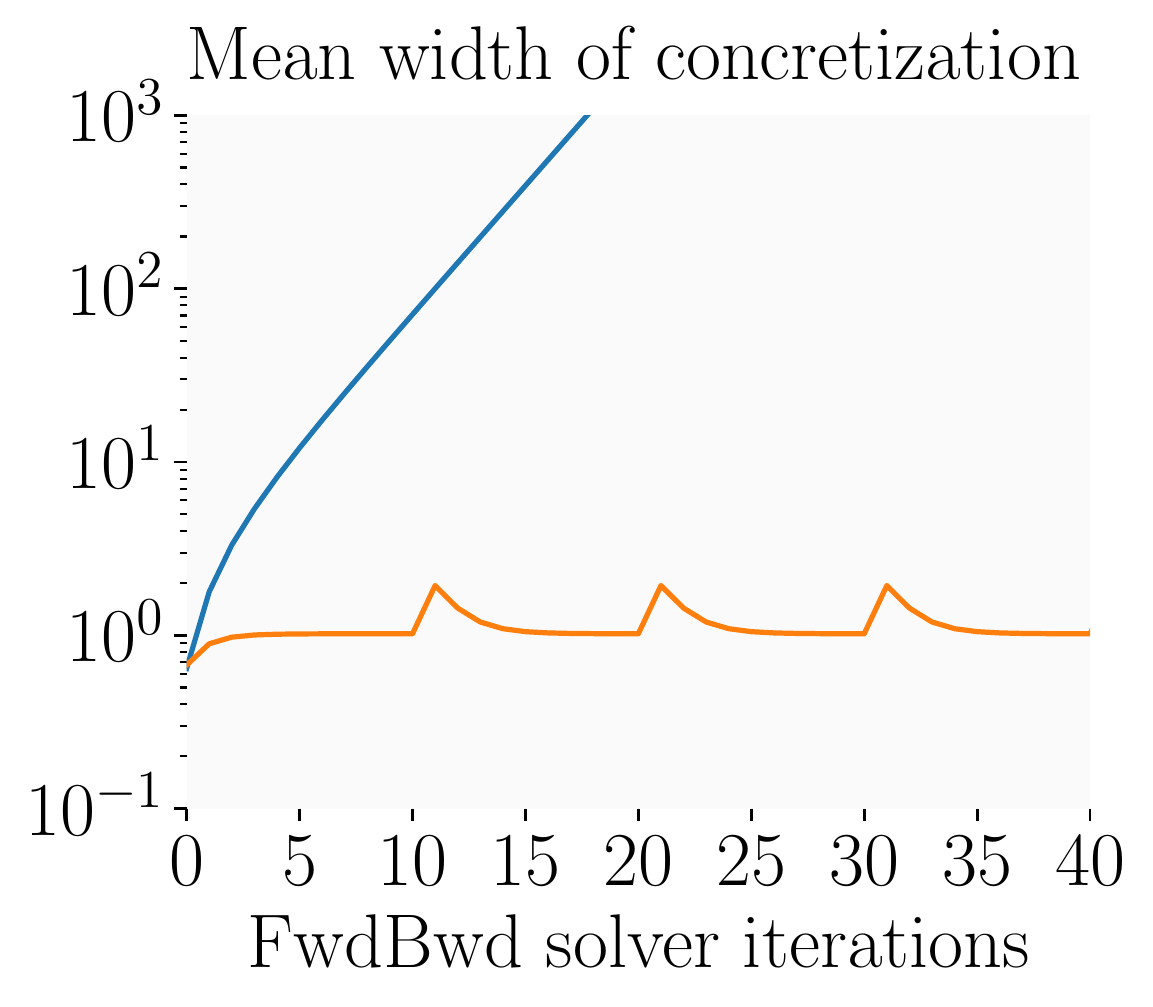}
		\vspace{-5mm}
		\caption{\gfwdbwdS}
		\vspace{-2mm}
	\end{subfigure}
	\hfil
	\begin{subfigure}{.48\linewidth}
		\centering
		\includegraphics[width=1.0\linewidth]{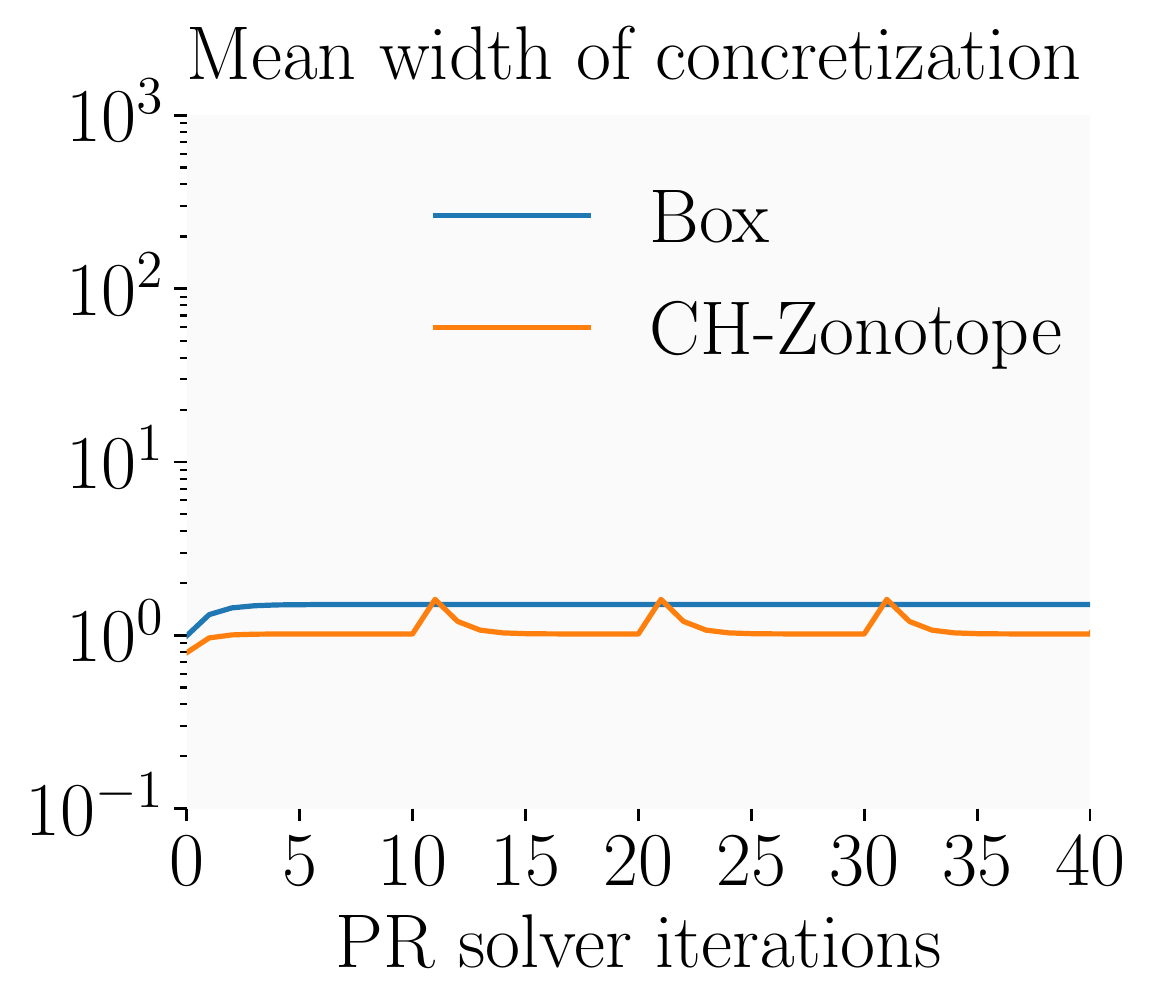}
		\vspace{-5mm}
		\caption{\gprS}
		\vspace{-2mm}
	\end{subfigure}
	\caption{Mean width of concretizations over the solver iteration for a representative sample on \fcf.}
	\vspace{-2mm}
	\label{fig:error_growth}
\end{wrapfigure}
Here, we evaluate the effectiveness of our novel \domain. %

\paragraph{Precision}
We compare \domain to Box, the only other domain commonly used in neural network verification that enables a tractable containment check (see \cref{tab:relaxations_intro}). In \cref{fig:error_growth}, we show the mean width of the concretized abstractions as a proxy for the domain's precision over the number of abstract solver iterations for a representative sample. Empirically, we find that the Box domain is significantly less precise, diverging quickly when using \fwdbwd splitting and being too imprecise to prove any property when using \pr splitting. For \domain, we observe how error consolidation periodically simplifies the abstraction, increasing its size, before additional solver applications tighten it again.
Empirically, error consolidation does not only enable our orders of magnitude faster but similarly precise inclusion check (see \cref{app:ablation_containment}), but also speeds up the analysis due to smaller representation sizes, while having a negligible impact on overall precision. (see \cref{app:ablation_error_consolidation}).
For a more in-depth analysis of the tightness of our approximate containment check (\cref{the:containment}) and the effect of error consolidation (\cref{thm:consolidation}) please see \cref{app:ablation_containment,app:ablation_error_consolidation}.

\begin{figure}[t]
	\centering
	\begin{minipage}[c]{0.38\textwidth}
	   \vspace{6mm}
		\begin{lstlisting}[language=Python,mathescape=true,numbers=none]
def root($x$):
   $s$ = $s_0$
   while $s \leq 0$ or $|s*s - 1/x| \geq \epsilon$:
	   $h$ = $(1 - x*s*s)$
	   $s$ = $s + s * (0.5 * h + 0.375 * h * h)$
   return $s$
		\end{lstlisting}
		\vspace{6mm}
		\captionof{figure}{Program \lstinline{root}, which calculates the square root of input x through iterative Householder approximation.}
		\label{fig:root_iter}
		\vspace{-5mm}
	\end{minipage}
	\hfill
	\begin{minipage}[c]{0.59\textwidth}
		\centering
	   \begin{subfigure}[t]{0.49\textwidth}
			\includegraphics[width=0.95 \textwidth]{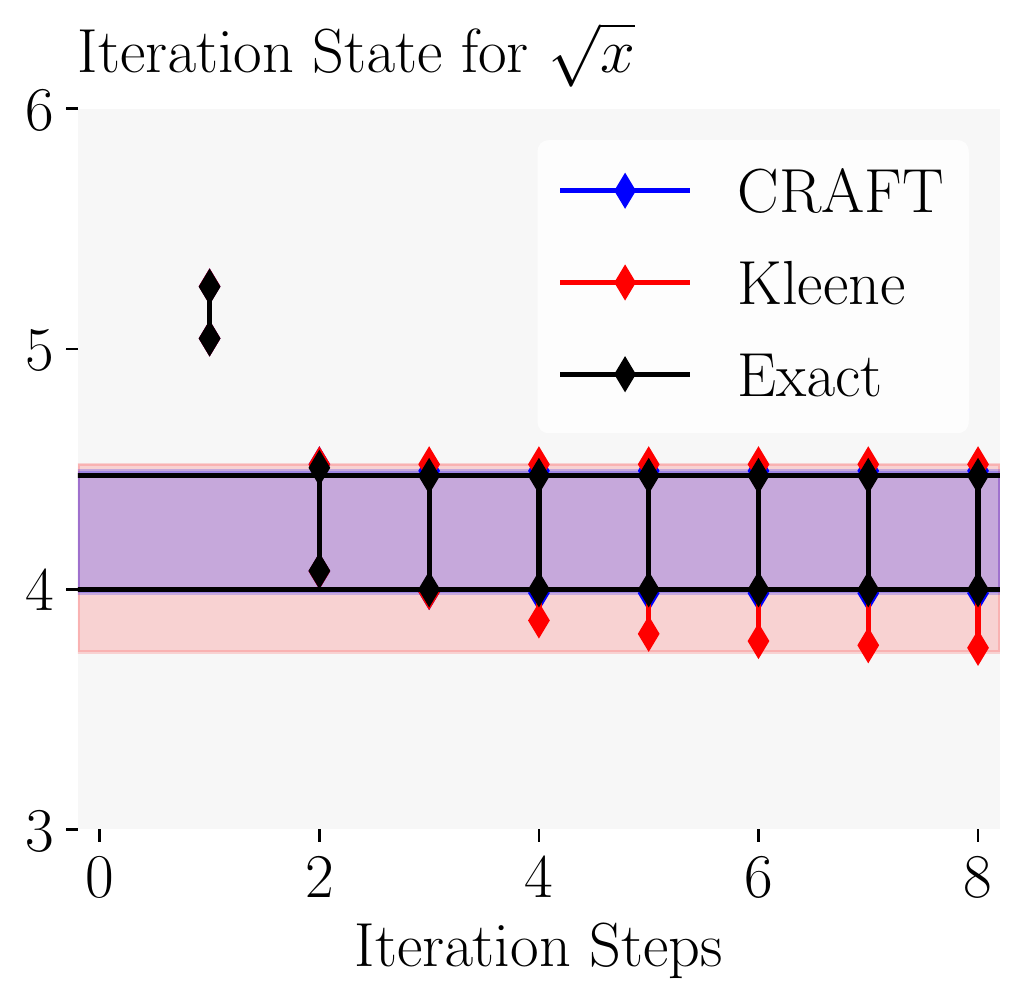}
			\subcaption{$\X=[16, 20]$}
			\label{fig:root_vis_a}
		\end{subfigure}
	   \begin{subfigure}[t]{0.49\textwidth}
		   \includegraphics[width=0.95 \textwidth]{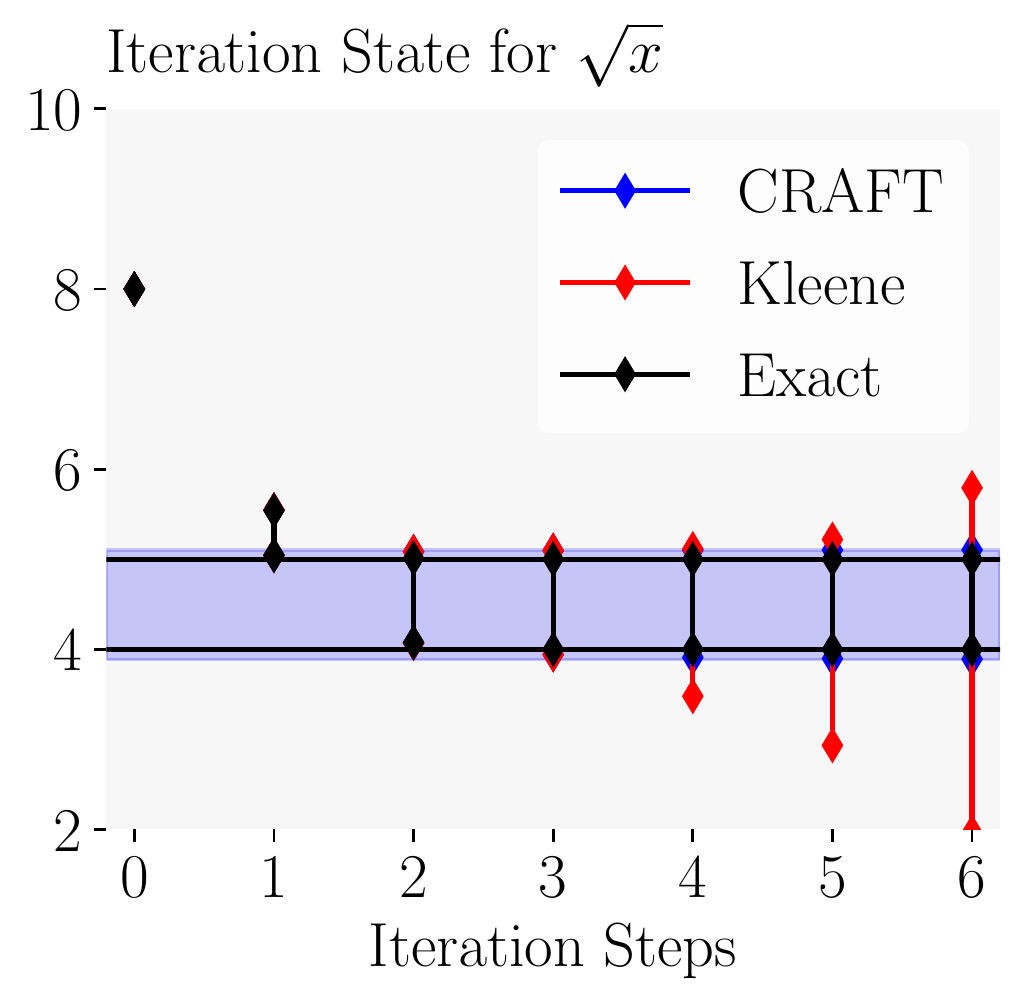}
			\subcaption{$\X=[16, 25]$}
			\label{fig:root_vis_b}
	   \end{subfigure}
		\vspace{-2mm}
		\captionof{figure}{Comparison of root ($1/s_i$) over-approximations for different intervals $\X$. We show the final interval as shaded region.}
		\label{fig:root_vis}
		\vspace{-5mm}
	\end{minipage}
\end{figure}

\subsection{Case Study: Analysis of Square Root Approximation} \label{sec:householder}
\begin{wraptable}[12]{r}{0.5145\textwidth}
	\vspace{-4.5mm}
	\renewcommand{\arraystretch}{1.2}
	\caption{Comparison of the fixpoint over-approximations obtained with different methods. Exact mathematical fixpoints (Exact), their over-approximation obtained via \tool and Kleene iteration.}
	\vspace{-2mm}
	\centering
		\scalebox{0.87}{
	\begin{threeparttable}
		\centering
		\begin{tabular}{ll cc}
			\toprule
			\multicolumn{2}{c}{\multirow{2.5}{*}{Method}} &\multicolumn{2}{c}{Root Interval $1/\gamma(\S^*)$}\\
			\cmidrule{3-4}
			& &\text{$\X$ = [16,20]} & \text{$\X$ = [16,25]}\\
			\midrule
			Exact & $\S^*$ & [4.000, 4.472] & [4.000, 5.000] \\
			\tool &  $\hat{\S}^*_\text{cr}$ & [3.983, 4.493] & [3.887, 5.104] \\
			Kleene iteration & $\hat{\S}^*_\text{kl}$ & [3.738, 4.520] & [0.000, $\,\:\:\infty\,\:\:$) \\
			\bottomrule
		\end{tabular}
	\end{threeparttable}
		}
	\label{tab:householder}
	\vspace{-3mm}
\end{wraptable}
In this section, we provide a simple example of the wider applicability of our abstract interpretation approach of fixpoint iterations and its advantages compared to Kleene iteration.

We consider the Householder method to compute (the reciprocal of) square roots, illustrated in \cref{fig:root_iter} and commonly used as a test case \citep{GoubaultPBG07,ZonotopeGhorbalGP09}. We consider the input set $\bc{X} = [16, 20] \subseteq \R$ (with exact fixpoint set $[4, \sqrt{20} \approx 4.472]$), the initialization $s_0 = 2^{-3}$, and the termination threshold $\epsilon = 10^{-8}$. We use the Zonotope domain \cite{ZonotopeGhorbalGP09} and compare to Kleene iteration with semantic unrolling \cite{BlanchetCCFMMMR02}, i.e., we iterate $\hat{\S}_{i} = f^\#(\hat{\S}_{i-1})$ if we can show the termination condition to not be satisfied and else $\hat{\S}_{i} = \hat{\S}_{i-1} \sqcup f^\#(\hat{\S}_{i-1})$. Using Kleene iteration, we thus obtain the fixpoint set abstraction $\hat{\S}^*_\text{kl}$, shown in red in \cref{fig:root_vis_a}, which contains all intermediate iteration states for which the termination condition might trigger. 
Our abstract interpreter \tool allows us to, instead, compute iterations as $\hat{\S}_{i} = f^\#(\hat{\S}_{i-1})$ until our contraction-based termination condition triggers ($\hat{\S}_{i} \sqsubseteq \hat{\S}_{i-1}$). This yields a more precise fixpoint set over-approximation $\hat{\S}^*_\text{cr}$, shown in blue in \cref{fig:root_vis_a}. Further, if we consider a more challenging precondition of $\X = [16, 25]$, Kleene iteration quickly diverges (see \cref{fig:root_vis_b}), while \tool computes the precise abstraction $\hat{\S}^*_\text{cr}$ (see \cref{tab:householder}). Note that \tool requires $10$ and $18$ iterations for $\X = [16, 20]$ and $\X = [16, 25]$, respectively, while Kleene iteration requires $30$ for $\X = [16, 20]$. We show truncated versions of the iteration in \cref{fig:root_vis} for readability.  In \cref{app:householder} we discuss how the termination condition can be analyzed in this setting to obtain an over-approximation of all reachable outputs instead of the true mathematical fixpoints.

\subsection{Limitations}
From the existence of (global) Lipschitz bounds on monDEQs \citep{ELMonDEQPabbarajuWK21}  and the convergence guarantees in the concrete, it follows that an exact abstract iteration converges to bounded fixpoint sets for every bounded input region. However, \tool has some limitations that could prevent us from computing them: (i) Our termination criterion (\cref{the:containment}) requires a contraction of the abstract iteration state. However, this contraction is not guaranteed to occur, even when using exact abstractions, and might not be detected by our incomplete containment check, even if it does occur. (ii) While an exact abstract iteration is guaranteed to converge, there is no such guarantee for its over-approximation, which could diverge due to imprecisions accumulated by the use of incomplete abstract transformers.
Despite these limitations, which we share with conventional Kleene iteration, we observe empirically that we can find fixpoint set approximations in all evaluated cases when using \pr splitting (see \cref{tab:main_results}).

\section{Related Work} \label{sec:related_work}
We now briefly review related work on Zonotope order reduction, neural network verification, and abstract interpretation.

\paragraph{Zonotope Order Reduction}
While the flexibility and expressiveness of Zonotope have made it a popular abstract domain for safety and reachability analysis \citep{kuhn1998rigorously,althoff2008verification,yang2018comparison}, its representation size can grow quickly. %
To alleviate this limitation, \citet{kuhn1998rigorously} suggested soundly over-approximating Zonotope using smaller, less precise representations, i.e., fewer error terms and thus smaller error matrices. This is called order-reduction via outer-approximation. 
While a range of such methods was introduced in the following years \citep{combastel2003state,girard2005reachability,yazarel2004geometric,yang2018comparison}, they were designed for Zonotope in very ($p\leq10$) or relatively ($p\leq100$) low dimensional spaces and of low order ($k\leq50$). Thus, they generally scale poorly to the larger dimensions ($p\geq 500$) and high orders ($k\geq1000$) we consider \citep{KopetzkiSA17O}. In this setting, the PCA-based method proposed by \citet{KopetzkiSA17O}, which we build on, was found to produce the tightest approximations while still being computationally tractable.

\paragraph{Incomplete Neural Network Verification}
Incomplete verification approaches (such as ours) are generally fast and efficient but sacrifice precision, i.e., they may fail to certify properties that do hold. They can be divided into bound propagation \citep{GehrMDTCV18,zhang2018crown,SinghGPV19,DeepZSinghGMPV18,xu2020automatic} and optimization problem based approaches, using linear programming (LP) \citep{singh2019beyond,muller2021prima,FerrariMJV22} or semidefinite programming (SDP) formulations \citep{raghunathan2018semidefinite}.
However, existing approaches are unable to handle (unbounded) fixpoint iterations and thereby (mon)DEQ verification without non-trivial adaptations.

In contrast to the above deterministic approaches, which analyze models as they are, stochastic defenses such as randomized smoothing \cite{lecuyer2018dp,CohenRK19} construct new robust models by introducing noise into the inference process. They establish robustness guarantees for these new models with high probability but incur significant runtime costs at both certification- and inference-time. This drawback is further exacerbated by the relatively expensive fixpoint iterations needed in (mon)DEQ inference.

\paragraph{Certification of monDEQs}
Two main approaches have been proposed to certify the robustness of monDEQs: (i) \citet{ELMonDEQPabbarajuWK21} use the special structure of monDEQs to bound the global Lipschitz constant of the network, and (ii) \citet{chen2021semialgebraic} adapt an SDP-based approach by introducing a semi-algebraic representation of the ReLU-operator used in monDEQs. 

While the global Lipschitz bounds from \citet{ELMonDEQPabbarajuWK21} do not require a per-sample analysis, they are inherently loose, especially in the challenging setting of $\ell_\infty$ perturbations, where our approach achieves much higher precision, as demonstrated in \cref{sec:eval}.

Depending on the encoding, the SDP-encoding by \citet{chen2021semialgebraic} allows to bound the score difference between classes, the global Lipschitz constant, or yields an ellipsoidal relationship between inputs and outputs. All three approaches only scale to an implicit layer size of $87$ neurons due to the limitations of the underlying SDP solver \citep{chen2021semialgebraic}. Additionally, the most effective approach suffers from long runtimes (up to $1400$s per sample) even for these small networks, making the certification of many inputs or larger networks infeasible. We compare favorably to this approach in terms of precision, runtime, and scalability in \cref{sec:eval}.

Orthogonally, \citet{LMonDEQRevay2020} show a way of bounding the Lipschitz constant of a monDEQ by construction but do not report any robustness certificates. Further, enforcing small Lipschitz constants this way reduces the resulting accuracy significantly, thus limiting the utility of the obtained networks. 

Recently, \citet{wei2022certified} built on the work from \citet{LMonDEQRevay2020} by further restricting the parametrization of monDEQs to make them more amenable to verification with the Box domain and thus permitting larger models with higher accuracy to be analyzed. However, for the general monDEQs we consider, their method reduces to analysis with the Box domain, which we found to be too imprecise to prove any property.

\paragraph{Abstract Interpretation of Unbounded Loops}
Abstract interpreters employ many techniques to either speed up the analysis of unbounded loops or make it more precise \citep{GoubaultPBG07}. These approaches, include semantic unrolling \citep{BlanchetCCFMMMR02}, widening, and narrowing \citep{CousotC77,CousotC77b,AmatoSSAV16,CousotC92b,Bourdoncle93}.
\citet{GangeNSSS13} discuss considerations for Kleene iteration on non-lattice abstract domains such as \domain.

\section{Conclusion} We presented a novel abstract interpretation approach for fixpoint iterators with convergence guarantees based on two key contributions: (i) theoretical insights which allow us to compute sound and precise fixpoint abstractions without using joins, and (ii) a new abstract domain, \domain, which allows for precise propagation of abstract elements and enables efficient inclusion checks ($\bc{O}(p^3)$ in dimension $p$). To demonstrate the effectiveness of this approach, we implemented our framework in a tool called \tool and evaluated it on the robustness verification of monDEQs, a novel neural architecture constituting a particularly challenging instance of a high-dimensional fixpoint iterators. 

In an extensive evaluation, we demonstrated %
that \tool exceeds state-of-the-art performance in monDEQ verification by two orders-of-magnitude in terms of speed, one order of magnitude in terms of scalability, and about $25\%$ in terms of certification rate, demonstrating the merit of our framework. %

\section{Acknowledgements}

We would like to thank the anonymous reviewers for their constructive comments, our colleague Florian Dorner for his insightful feedback, and our Shepherd Swarat Chaudhuri.

This work has been done as part of the EU grant ELSA (European Lighthouse on Secure and Safe AI, grant agreement No. 101070617) and the SERI grant SAFEAI (Certified Safe, Fair and Robust Artificial Intelligence, contract no. MB22.00088). Views and opinions expressed are however those of the authors only and do not necessarily reflect those of the European Union or European Commission. Neither the European Union nor the European Commission can be held responsible for them. The work has received funding from the Swiss State Secretariat for Education, Research and Innovation (SERI).

\section{Further Resources}
We have published all code, models, and instructions required to reproduce our results on Zenodo \citep{CRAFT_artifact} with an updated version being available at \url{github.com/eth-sri/craft}.

\message{^^JLASTBODYPAGE \thepage^^J}

\clearpage
\bibliographystyle{ACM-Reference-Format}
\bibliography{references}


\begin{thebibliography}{68}

%%% the Web address.

\ifx \showCODEN    \undefined \def \showCODEN     #1{\unskip}     \fi
\ifx \showDOI      \undefined \def \showDOI       #1{#1}\fi
\ifx \showISBNx    \undefined \def \showISBNx     #1{\unskip}     \fi
\ifx \showISBNxiii \undefined \def \showISBNxiii  #1{\unskip}     \fi
\ifx \showISSN     \undefined \def \showISSN      #1{\unskip}     \fi
\ifx \showLCCN     \undefined \def \showLCCN      #1{\unskip}     \fi
\ifx \shownote     \undefined \def \shownote      #1{#1}          \fi
\ifx \showarticletitle \undefined \def \showarticletitle #1{#1}   \fi
\ifx \showURL      \undefined \def \showURL       {\relax}        \fi
% invisible to TeX
\providecommand\bibfield[2]{#2}
\providecommand\bibinfo[2]{#2}
\providecommand\natexlab[1]{#1}
\providecommand\showeprint[2][]{arXiv:#2}

\bibitem[Althoff et~al\mbox{.}(2008)]%
        {althoff2008verification}
\bibfield{author}{\bibinfo{person}{Matthias Althoff}, \bibinfo{person}{Olaf
  Stursberg}, {and} \bibinfo{person}{Martin Buss}.}
  \bibinfo{year}{2008}\natexlab{}.
\newblock \showarticletitle{Verification of uncertain embedded systems by
  computing reachable sets based on zonotopes}.
\newblock \bibinfo{journal}{\emph{IFAC Proceedings Volumes}}
  \bibinfo{volume}{41}, \bibinfo{number}{2} (\bibinfo{year}{2008}).
\newblock


\bibitem[Amato and Scozzari(2012)]%
        {ParallelotopeAmatoS12}
\bibfield{author}{\bibinfo{person}{Gianluca Amato} {and}
  \bibinfo{person}{Francesca Scozzari}.} \bibinfo{year}{2012}\natexlab{}.
\newblock \showarticletitle{The Abstract Domain of Parallelotopes}.
\newblock \bibinfo{journal}{\emph{Electron. Notes Theor. Comput. Sci.}}
  \bibinfo{volume}{287} (\bibinfo{year}{2012}).
\newblock
\urldef\tempurl%
\url{https://doi.org/10.1016/j.entcs.2012.09.003}
\showDOI{\tempurl}


\bibitem[Amato et~al\mbox{.}(2016)]%
        {AmatoSSAV16}
\bibfield{author}{\bibinfo{person}{Gianluca Amato}, \bibinfo{person}{Francesca
  Scozzari}, \bibinfo{person}{Helmut Seidl}, \bibinfo{person}{Kalmer Apinis},
  {and} \bibinfo{person}{Vesal Vojdani}.} \bibinfo{year}{2016}\natexlab{}.
\newblock \showarticletitle{Efficiently intertwining widening and narrowing}.
\newblock \bibinfo{journal}{\emph{Sci. Comput. Program.}}
  \bibinfo{volume}{120} (\bibinfo{year}{2016}).
\newblock
\urldef\tempurl%
\url{https://doi.org/10.1016/j.scico.2015.12.005}
\showDOI{\tempurl}


\bibitem[Amos and Kolter(2017)]%
        {OPTNetAmosK17}
\bibfield{author}{\bibinfo{person}{Brandon Amos} {and} \bibinfo{person}{J.~Zico
  Kolter}.} \bibinfo{year}{2017}\natexlab{}.
\newblock \showarticletitle{OptNet: Differentiable Optimization as a Layer in
  Neural Networks}. In \bibinfo{booktitle}{\emph{Proc. of ICML}},
  Vol.~\bibinfo{volume}{70}.
\newblock


\bibitem[Bai et~al\mbox{.}(2019)]%
        {DEQBaiKK19}
\bibfield{author}{\bibinfo{person}{Shaojie Bai}, \bibinfo{person}{J.~Zico
  Kolter}, {and} \bibinfo{person}{Vladlen Koltun}.}
  \bibinfo{year}{2019}\natexlab{}.
\newblock \showarticletitle{Deep Equilibrium Models}. In
  \bibinfo{booktitle}{\emph{Advances in Neural Information Processing Systems
  32: Annual Conference on Neural Information Processing Systems 2019, NeurIPS
  2019, December 8-14, 2019, Vancouver, BC, Canada}}.
\newblock


\bibitem[Blanchet et~al\mbox{.}(2002)]%
        {BlanchetCCFMMMR02}
\bibfield{author}{\bibinfo{person}{Bruno Blanchet}, \bibinfo{person}{Patrick
  Cousot}, \bibinfo{person}{Radhia Cousot}, \bibinfo{person}{J{\'{e}}r{\^{o}}me
  Feret}, \bibinfo{person}{Laurent Mauborgne}, \bibinfo{person}{Antoine
  Min{\'{e}}}, \bibinfo{person}{David Monniaux}, {and} \bibinfo{person}{Xavier
  Rival}.} \bibinfo{year}{2002}\natexlab{}.
\newblock \showarticletitle{Design and Implementation of a Special-Purpose
  Static Program Analyzer for Safety-Critical Real-Time Embedded Software}. In
  \bibinfo{booktitle}{\emph{The Essence of Computation, Complexity, Analysis,
  Transformation. Essays Dedicated to Neil D. Jones [on occasion of his 60th
  birthday]}}, Vol.~\bibinfo{volume}{2566}.
\newblock
\urldef\tempurl%
\url{https://doi.org/10.1007/3-540-36377-7\_5}
\showDOI{\tempurl}


\bibitem[Bourdoncle(1993)]%
        {Bourdoncle93}
\bibfield{author}{\bibinfo{person}{Fran{\c{c}}ois Bourdoncle}.}
  \bibinfo{year}{1993}\natexlab{}.
\newblock \showarticletitle{Efficient chaotic iteration strategies with
  widenings}. In \bibinfo{booktitle}{\emph{Formal Methods in Programming and
  Their Applications, International Conference, Akademgorodok, Novosibirsk,
  Russia, June 28 - July 2, 1993, Proceedings}}, Vol.~\bibinfo{volume}{735}.
\newblock
\urldef\tempurl%
\url{https://doi.org/10.1007/BFb0039704}
\showDOI{\tempurl}


\bibitem[Chen et~al\mbox{.}(2021)]%
        {chen2021semialgebraic}
\bibfield{author}{\bibinfo{person}{Tong Chen}, \bibinfo{person}{Jean-Bernard
  Lasserre}, \bibinfo{person}{Victor Magron}, {and} \bibinfo{person}{Edouard
  Pauwels}.} \bibinfo{year}{2021}\natexlab{}.
\newblock \showarticletitle{Semialgebraic Representation of Monotone Deep
  Equilibrium Models and Applications to Certification}.
\newblock \bibinfo{journal}{\emph{ArXiv preprint}}
  \bibinfo{volume}{abs/2106.01453} (\bibinfo{year}{2021}).
\newblock


\bibitem[Cohen et~al\mbox{.}(2019)]%
        {CohenRK19}
\bibfield{author}{\bibinfo{person}{Jeremy~M. Cohen}, \bibinfo{person}{Elan
  Rosenfeld}, {and} \bibinfo{person}{J.~Zico Kolter}.}
  \bibinfo{year}{2019}\natexlab{}.
\newblock \showarticletitle{Certified Adversarial Robustness via Randomized
  Smoothing}. In \bibinfo{booktitle}{\emph{Proc. of ICML}},
  Vol.~\bibinfo{volume}{97}.
\newblock


\bibitem[Combastel(2003)]%
        {combastel2003state}
\bibfield{author}{\bibinfo{person}{Christophe Combastel}.}
  \bibinfo{year}{2003}\natexlab{}.
\newblock \showarticletitle{A state bounding observer based on zonotopes}. In
  \bibinfo{booktitle}{\emph{2003 European Control Conference (ECC)}}. IEEE.
\newblock


\bibitem[Cousot and Cousot(1977a)]%
        {CousotC77}
\bibfield{author}{\bibinfo{person}{Patrick Cousot} {and}
  \bibinfo{person}{Radhia Cousot}.} \bibinfo{year}{1977}\natexlab{a}.
\newblock \showarticletitle{Abstract Interpretation: {A} Unified Lattice Model
  for Static Analysis of Programs by Construction or Approximation of
  Fixpoints}. In \bibinfo{booktitle}{\emph{Conference Record of the Fourth
  {ACM} Symposium on Principles of Programming Languages, Los Angeles,
  California, USA, January 1977}}.
\newblock
\urldef\tempurl%
\url{https://doi.org/10.1145/512950.512973}
\showDOI{\tempurl}


\bibitem[Cousot and Cousot(1977b)]%
        {CousotC77b}
\bibfield{author}{\bibinfo{person}{Patrick Cousot} {and}
  \bibinfo{person}{Radhia Cousot}.} \bibinfo{year}{1977}\natexlab{b}.
\newblock \showarticletitle{Static Determination of Dynamic Properties of
  Recursive Procedures}. In \bibinfo{booktitle}{\emph{Formal Description of
  Programming Concepts: Proceedings of the {IFIP} Working Conference on Formal
  Description of Programming Concepts, St. Andrews, NB, Canada, August 1-5,
  1977}}.
\newblock


\bibitem[Cousot and Cousot(1979)]%
        {cousot1979constructive}
\bibfield{author}{\bibinfo{person}{Patrick Cousot} {and}
  \bibinfo{person}{Radhia Cousot}.} \bibinfo{year}{1979}\natexlab{}.
\newblock \showarticletitle{Constructive versions of Tarski’s fixed point
  theorems}.
\newblock \bibinfo{journal}{\emph{Pacific journal of Mathematics}}
  \bibinfo{volume}{82}, \bibinfo{number}{1} (\bibinfo{year}{1979}).
\newblock


\bibitem[Cousot and Cousot(1992a)]%
        {CousotC92}
\bibfield{author}{\bibinfo{person}{Patrick Cousot} {and}
  \bibinfo{person}{Radhia Cousot}.} \bibinfo{year}{1992}\natexlab{a}.
\newblock \showarticletitle{Abstract Interpretation Frameworks}.
\newblock \bibinfo{journal}{\emph{J. Log. Comput.}} \bibinfo{volume}{2},
  \bibinfo{number}{4} (\bibinfo{year}{1992}).
\newblock
\urldef\tempurl%
\url{https://doi.org/10.1093/logcom/2.4.511}
\showDOI{\tempurl}


\bibitem[Cousot and Cousot(1992b)]%
        {CousotC92b}
\bibfield{author}{\bibinfo{person}{Patrick Cousot} {and}
  \bibinfo{person}{Radhia Cousot}.} \bibinfo{year}{1992}\natexlab{b}.
\newblock \showarticletitle{Comparing the Galois Connection and
  Widening/Narrowing Approaches to Abstract Interpretation}. In
  \bibinfo{booktitle}{\emph{Programming Language Implementation and Logic
  Programming, 4th International Symposium, PLILP'92, Leuven, Belgium, August
  26-28, 1992, Proceedings}}, Vol.~\bibinfo{volume}{631}.
\newblock
\urldef\tempurl%
\url{https://doi.org/10.1007/3-540-55844-6\_142}
\showDOI{\tempurl}


\bibitem[Ferrari et~al\mbox{.}(2022)]%
        {FerrariMJV22}
\bibfield{author}{\bibinfo{person}{Claudio Ferrari},
  \bibinfo{person}{Mark~Niklas M{\"{u}}ller}, \bibinfo{person}{Nikola
  Jovanovic}, {and} \bibinfo{person}{Martin~T. Vechev}.}
  \bibinfo{year}{2022}\natexlab{}.
\newblock \showarticletitle{Complete Verification via Multi-Neuron Relaxation
  Guided Branch-and-Bound}. In \bibinfo{booktitle}{\emph{Proc. of ICLR}}.
\newblock


\bibitem[Fu and Li(2021)]%
        {Fu2021Repair}
\bibfield{author}{\bibinfo{person}{Feisi Fu} {and} \bibinfo{person}{Wenchao
  Li}.} \bibinfo{year}{2021}\natexlab{}.
\newblock \showarticletitle{Sound and Complete Neural Network Repair with
  Minimality and Locality Guarantees}.
\newblock \bibinfo{journal}{\emph{ArXiv preprint}}
  \bibinfo{volume}{abs/2110.07682} (\bibinfo{year}{2021}).
\newblock


\bibitem[Gange et~al\mbox{.}(2013)]%
        {GangeNSSS13}
\bibfield{author}{\bibinfo{person}{Graeme Gange}, \bibinfo{person}{Jorge~A.
  Navas}, \bibinfo{person}{Peter Schachte}, \bibinfo{person}{Harald
  S{\o}ndergaard}, {and} \bibinfo{person}{Peter~J. Stuckey}.}
  \bibinfo{year}{2013}\natexlab{}.
\newblock \showarticletitle{Abstract Interpretation over Non-lattice Abstract
  Domains}. In \bibinfo{booktitle}{\emph{Static Analysis - 20th International
  Symposium, {SAS} 2013, Seattle, WA, USA, June 20-22, 2013. Proceedings}},
  Vol.~\bibinfo{volume}{7935}.
\newblock
\urldef\tempurl%
\url{https://doi.org/10.1007/978-3-642-38856-9\_3}
\showDOI{\tempurl}


\bibitem[Gehr et~al\mbox{.}(2018)]%
        {GehrMDTCV18}
\bibfield{author}{\bibinfo{person}{Timon Gehr}, \bibinfo{person}{Matthew
  Mirman}, \bibinfo{person}{Dana Drachsler{-}Cohen}, \bibinfo{person}{Petar
  Tsankov}, \bibinfo{person}{Swarat Chaudhuri}, {and}
  \bibinfo{person}{Martin~T. Vechev}.} \bibinfo{year}{2018}\natexlab{}.
\newblock \showarticletitle{{AI2:} Safety and Robustness Certification of
  Neural Networks with Abstract Interpretation}. In
  \bibinfo{booktitle}{\emph{2018 {IEEE} Symposium on Security and Privacy, {SP}
  2018, Proceedings, 21-23 May 2018, San Francisco, California, {USA}}}.
\newblock
\urldef\tempurl%
\url{https://doi.org/10.1109/SP.2018.00058}
\showDOI{\tempurl}


\bibitem[Ghaoui et~al\mbox{.}(2021)]%
        {ImplicitDLGhaoui2019}
\bibfield{author}{\bibinfo{person}{Laurent~El Ghaoui}, \bibinfo{person}{Fangda
  Gu}, \bibinfo{person}{Bertrand Travacca}, \bibinfo{person}{Armin Askari},
  {and} \bibinfo{person}{Alicia~Y. Tsai}.} \bibinfo{year}{2021}\natexlab{}.
\newblock \showarticletitle{Implicit Deep Learning}.
\newblock \bibinfo{journal}{\emph{{SIAM} J. Math. Data Sci.}}
  \bibinfo{volume}{3}, \bibinfo{number}{3} (\bibinfo{year}{2021}).
\newblock
\urldef\tempurl%
\url{https://doi.org/10.1137/20M1358517}
\showDOI{\tempurl}


\bibitem[Ghorbal et~al\mbox{.}(2009)]%
        {ZonotopeGhorbalGP09}
\bibfield{author}{\bibinfo{person}{Khalil Ghorbal}, \bibinfo{person}{Eric
  Goubault}, {and} \bibinfo{person}{Sylvie Putot}.}
  \bibinfo{year}{2009}\natexlab{}.
\newblock \showarticletitle{The Zonotope Abstract Domain Taylor1+}. In
  \bibinfo{booktitle}{\emph{Computer Aided Verification, 21st International
  Conference, {CAV} 2009, Grenoble, France, June 26 - July 2, 2009.
  Proceedings}}, Vol.~\bibinfo{volume}{5643}.
\newblock
\urldef\tempurl%
\url{https://doi.org/10.1007/978-3-642-02658-4\_47}
\showDOI{\tempurl}


\bibitem[Girard(2005)]%
        {girard2005reachability}
\bibfield{author}{\bibinfo{person}{Antoine Girard}.}
  \bibinfo{year}{2005}\natexlab{}.
\newblock \showarticletitle{Reachability of uncertain linear systems using
  zonotopes}. In \bibinfo{booktitle}{\emph{International Workshop on Hybrid
  Systems: Computation and Control}}. Springer.
\newblock


\bibitem[Goodfellow et~al\mbox{.}(2015)]%
        {GoodfellowSS14}
\bibfield{author}{\bibinfo{person}{Ian~J. Goodfellow},
  \bibinfo{person}{Jonathon Shlens}, {and} \bibinfo{person}{Christian
  Szegedy}.} \bibinfo{year}{2015}\natexlab{}.
\newblock \showarticletitle{Explaining and Harnessing Adversarial Examples}. In
  \bibinfo{booktitle}{\emph{Proc. of ICLR}}.
\newblock


\bibitem[Goubault and Putot(2008)]%
        {GoubaultP2008}
\bibfield{author}{\bibinfo{person}{Eric Goubault} {and} \bibinfo{person}{Sylvie
  Putot}.} \bibinfo{year}{2008}\natexlab{}.
\newblock \showarticletitle{Perturbed affine arithmetic for invariant
  computation in numerical program analysis}.
\newblock \bibinfo{journal}{\emph{CoRR}}  \bibinfo{volume}{abs/0807.2961}
  (\bibinfo{year}{2008}).
\newblock
\showeprint[arXiv]{0807.2961}


\bibitem[Goubault et~al\mbox{.}(2007)]%
        {GoubaultPBG07}
\bibfield{author}{\bibinfo{person}{Eric Goubault}, \bibinfo{person}{Sylvie
  Putot}, \bibinfo{person}{Philippe Baufreton}, {and} \bibinfo{person}{Jean
  Gassino}.} \bibinfo{year}{2007}\natexlab{}.
\newblock \showarticletitle{Static Analysis of the Accuracy in Control Systems:
  Principles and Experiments}. In \bibinfo{booktitle}{\emph{Formal Methods for
  Industrial Critical Systems, 12th International Workshop, {FMICS} 2007,
  Berlin, Germany, July 1-2, 2007, Revised Selected Papers}},
  Vol.~\bibinfo{volume}{4916}.
\newblock
\urldef\tempurl%
\url{https://doi.org/10.1007/978-3-540-79707-4\_3}
\showDOI{\tempurl}


\bibitem[Gover and Krikorian(2010)]%
        {gover2010determinants}
\bibfield{author}{\bibinfo{person}{Eugene Gover} {and} \bibinfo{person}{Nishan
  Krikorian}.} \bibinfo{year}{2010}\natexlab{}.
\newblock \showarticletitle{Determinants and the volumes of parallelotopes and
  zonotopes}.
\newblock \bibinfo{journal}{\emph{Linear Algebra Appl.}} \bibinfo{volume}{433},
  \bibinfo{number}{1} (\bibinfo{year}{2010}).
\newblock


\bibitem[Gowal et~al\mbox{.}(2018)]%
        {GowalDSBQUAMK18}
\bibfield{author}{\bibinfo{person}{Sven Gowal}, \bibinfo{person}{Krishnamurthy
  Dvijotham}, \bibinfo{person}{Robert Stanforth}, \bibinfo{person}{Rudy Bunel},
  \bibinfo{person}{Chongli Qin}, \bibinfo{person}{Jonathan Uesato},
  \bibinfo{person}{Relja Arandjelovic}, \bibinfo{person}{Timothy~A. Mann},
  {and} \bibinfo{person}{Pushmeet Kohli}.} \bibinfo{year}{2018}\natexlab{}.
\newblock \showarticletitle{On the Effectiveness of Interval Bound Propagation
  for Training Verifiably Robust Models}.
\newblock \bibinfo{journal}{\emph{ArXiv preprint}}
  \bibinfo{volume}{abs/1810.12715} (\bibinfo{year}{2018}).
\newblock


\bibitem[Gowal et~al\mbox{.}(2019)]%
        {GowalUQHMK19}
\bibfield{author}{\bibinfo{person}{Sven Gowal}, \bibinfo{person}{Jonathan
  Uesato}, \bibinfo{person}{Chongli Qin}, \bibinfo{person}{Po{-}Sen Huang},
  \bibinfo{person}{Timothy~A. Mann}, {and} \bibinfo{person}{Pushmeet Kohli}.}
  \bibinfo{year}{2019}\natexlab{}.
\newblock \showarticletitle{An Alternative Surrogate Loss for PGD-based
  Adversarial Testing}.
\newblock \bibinfo{journal}{\emph{ArXiv preprint}}
  \bibinfo{volume}{abs/1910.09338} (\bibinfo{year}{2019}).
\newblock


\bibitem[{Gurobi Optimization, LLC}(2022)]%
        {gurobi}
\bibfield{author}{\bibinfo{person}{{Gurobi Optimization, LLC}}.}
  \bibinfo{year}{2022}\natexlab{}.
\newblock \bibinfo{title}{{Gurobi Optimizer Reference Manual}}.
\newblock
\newblock


\bibitem[Guth(2013)]%
        {guth2013formal}
\bibfield{author}{\bibinfo{person}{Dwight Guth}.}
  \bibinfo{year}{2013}\natexlab{}.
\newblock \showarticletitle{A formal semantics of Python 3.3}.
\newblock  (\bibinfo{year}{2013}).
\newblock


\bibitem[Jiang et~al\mbox{.}(2020)]%
        {JiangLP20}
\bibfield{author}{\bibinfo{person}{Shunhua Jiang}, \bibinfo{person}{Zhao Song},
  \bibinfo{person}{Omri Weinstein}, {and} \bibinfo{person}{Hengjie Zhang}.}
  \bibinfo{year}{2020}\natexlab{}.
\newblock \showarticletitle{Faster Dynamic Matrix Inverse for Faster LPs}.
\newblock \bibinfo{journal}{\emph{ArXiv preprint}}
  \bibinfo{volume}{abs/2004.07470} (\bibinfo{year}{2020}).
\newblock


\bibitem[Julian and Kochenderfer(2019)]%
        {julian2019guaranteeing}
\bibfield{author}{\bibinfo{person}{Kyle~D Julian} {and}
  \bibinfo{person}{Mykel~J Kochenderfer}.} \bibinfo{year}{2019}\natexlab{}.
\newblock \showarticletitle{Guaranteeing safety for neural network-based
  aircraft collision avoidance systems}. In \bibinfo{booktitle}{\emph{Digital
  Avionics Systems Conference (DASC)}}.
\newblock


\bibitem[Kellner(2015)]%
        {kellner2015containment}
\bibfield{author}{\bibinfo{person}{Kai Kellner}.}
  \bibinfo{year}{2015}\natexlab{}.
\newblock \showarticletitle{Containment problems for projections of polyhedra
  and spectrahedra}.
\newblock \bibinfo{journal}{\emph{ArXiv preprint}}
  \bibinfo{volume}{abs/1509.02735} (\bibinfo{year}{2015}).
\newblock


\bibitem[Kim et~al\mbox{.}(2021)]%
        {KimJDR2021}
\bibfield{author}{\bibinfo{person}{Suyong Kim}, \bibinfo{person}{Weiqi Ji},
  \bibinfo{person}{Sili Deng}, \bibinfo{person}{Yingbo Ma}, {and}
  \bibinfo{person}{Christopher Rackauckas}.} \bibinfo{year}{2021}\natexlab{}.
\newblock \showarticletitle{Stiff neural ordinary differential equations}.
\newblock \bibinfo{journal}{\emph{Chaos: An Interdisciplinary Journal of
  Nonlinear Science}} \bibinfo{volume}{31}, \bibinfo{number}{9}
  (\bibinfo{year}{2021}).
\newblock


\bibitem[Kopetzki et~al\mbox{.}(2017)]%
        {KopetzkiSA17O}
\bibfield{author}{\bibinfo{person}{Anna{-}Kathrin Kopetzki},
  \bibinfo{person}{Bastian Sch{\"{u}}rmann}, {and} \bibinfo{person}{Matthias
  Althoff}.} \bibinfo{year}{2017}\natexlab{}.
\newblock \showarticletitle{Methods for order reduction of zonotopes}. In
  \bibinfo{booktitle}{\emph{56th {IEEE} Annual Conference on Decision and
  Control, {CDC} 2017, Melbourne, Australia, December 12-15, 2017}}.
\newblock
\urldef\tempurl%
\url{https://doi.org/10.1109/CDC.2017.8264508}
\showDOI{\tempurl}


\bibitem[Krizhevsky et~al\mbox{.}(2009)]%
        {krizhevsky2009learning}
\bibfield{author}{\bibinfo{person}{Alex Krizhevsky}, \bibinfo{person}{Geoffrey
  Hinton}, {et~al\mbox{.}}} \bibinfo{year}{2009}\natexlab{}.
\newblock \showarticletitle{Learning multiple layers of features from tiny
  images}.
\newblock  (\bibinfo{year}{2009}).
\newblock


\bibitem[K{\"u}hn(1998)]%
        {kuhn1998rigorously}
\bibfield{author}{\bibinfo{person}{Wolfgang K{\"u}hn}.}
  \bibinfo{year}{1998}\natexlab{}.
\newblock \showarticletitle{Rigorously computed orbits of dynamical systems
  without the wrapping effect}.
\newblock \bibinfo{journal}{\emph{Computing}} \bibinfo{volume}{61},
  \bibinfo{number}{1} (\bibinfo{year}{1998}).
\newblock


\bibitem[Kulmburg and Althoff(2021)]%
        {Kulmburg2021OnTC}
\bibfield{author}{\bibinfo{person}{Adrian Kulmburg} {and}
  \bibinfo{person}{Matthias Althoff}.} \bibinfo{year}{2021}\natexlab{}.
\newblock \showarticletitle{On the co-NP-completeness of the zonotope
  containment problem}.
\newblock \bibinfo{journal}{\emph{European Journal of Control}}
  (\bibinfo{year}{2021}).
\newblock


\bibitem[LeCun et~al\mbox{.}(1998)]%
        {lecun1998gradient}
\bibfield{author}{\bibinfo{person}{Yann LeCun}, \bibinfo{person}{L{\'e}on
  Bottou}, \bibinfo{person}{Yoshua Bengio}, {and} \bibinfo{person}{Patrick
  Haffner}.} \bibinfo{year}{1998}\natexlab{}.
\newblock \showarticletitle{Gradient-based learning applied to document
  recognition}.
\newblock \bibinfo{journal}{\emph{Proc. IEEE}} \bibinfo{volume}{86},
  \bibinfo{number}{11} (\bibinfo{year}{1998}).
\newblock


\bibitem[Lecuyer et~al\mbox{.}(2018)]%
        {lecuyer2018dp}
\bibfield{author}{\bibinfo{person}{Mathias Lecuyer}, \bibinfo{person}{Vaggelis
  Atlidakis}, \bibinfo{person}{Roxana Geambasu}, \bibinfo{person}{Daniel Hsu},
  {and} \bibinfo{person}{Suman Jana}.} \bibinfo{year}{2018}\natexlab{}.
\newblock \showarticletitle{Certified Robustness to Adversarial Examples with
  Differential Privacy}.
\newblock \bibinfo{journal}{\emph{2019 IEEE Symposium on Security and Privacy
  (S\&P)}} (\bibinfo{year}{2018}).
\newblock


\bibitem[Madry et~al\mbox{.}(2018)]%
        {MadryMSTV18}
\bibfield{author}{\bibinfo{person}{Aleksander Madry},
  \bibinfo{person}{Aleksandar Makelov}, \bibinfo{person}{Ludwig Schmidt},
  \bibinfo{person}{Dimitris Tsipras}, {and} \bibinfo{person}{Adrian Vladu}.}
  \bibinfo{year}{2018}\natexlab{}.
\newblock \showarticletitle{Towards Deep Learning Models Resistant to
  Adversarial Attacks}. In \bibinfo{booktitle}{\emph{Proc. of ICLR}}.
\newblock


\bibitem[Mirman et~al\mbox{.}(2018)]%
        {DiffAIMirmanGV18}
\bibfield{author}{\bibinfo{person}{Matthew Mirman}, \bibinfo{person}{Timon
  Gehr}, {and} \bibinfo{person}{Martin~T. Vechev}.}
  \bibinfo{year}{2018}\natexlab{}.
\newblock \showarticletitle{Differentiable Abstract Interpretation for Provably
  Robust Neural Networks}. In \bibinfo{booktitle}{\emph{Proc. of ICML}},
  Vol.~\bibinfo{volume}{80}.
\newblock


\bibitem[M{\"{u}}ller et~al\mbox{.}(2023)]%
        {CRAFT_artifact}
\bibfield{author}{\bibinfo{person}{Mark~Niklas M{\"{u}}ller},
  \bibinfo{person}{Marc Fischer}, \bibinfo{person}{Robin Staab}, {and}
  \bibinfo{person}{Martin Vechev}.} \bibinfo{year}{2023}\natexlab{}.
\newblock \bibinfo{booktitle}{\emph{{Abstract Interpretation of Fixpoint
  Iterators with Applications to Neural Networks - Artifact}}}.
\newblock
\urldef\tempurl%
\url{https://doi.org/10.5281/zenodo.7794269}
\showDOI{\tempurl}


\bibitem[M\"{u}ller et~al\mbox{.}(2022)]%
        {muller2021prima}
\bibfield{author}{\bibinfo{person}{Mark~Niklas M\"{u}ller},
  \bibinfo{person}{Gleb Makarchuk}, \bibinfo{person}{Gagandeep Singh},
  \bibinfo{person}{Markus P\"{u}schel}, {and} \bibinfo{person}{Martin Vechev}.}
  \bibinfo{year}{2022}\natexlab{}.
\newblock \showarticletitle{PRIMA: General and Precise Neural Network
  Certification via Scalable Convex Hull Approximations}.
\newblock \bibinfo{journal}{\emph{Proc. ACM Program. Lang.}}
  \bibinfo{volume}{6}, \bibinfo{number}{POPL}, Article \bibinfo{articleno}{43}
  (\bibinfo{year}{2022}), \bibinfo{numpages}{33}~pages.
\newblock
\urldef\tempurl%
\url{https://doi.org/10.1145/3498704}
\showDOI{\tempurl}


\bibitem[OpenReview(2021)]%
        {LBEN_OpenReview}
\bibfield{author}{\bibinfo{person}{OpenReview}.}
  \bibinfo{year}{2021}\natexlab{}.
\newblock \bibinfo{title}{OpenReview discussion on Lipschitz-Bounded
  Equilibrium Networks}.
\newblock
  \bibinfo{howpublished}{\url{https://openreview.net/forum?id=bodgPrarPUJ}}.
\newblock


\bibitem[Pabbaraju et~al\mbox{.}(2021)]%
        {ELMonDEQPabbarajuWK21}
\bibfield{author}{\bibinfo{person}{Chirag Pabbaraju}, \bibinfo{person}{Ezra
  Winston}, {and} \bibinfo{person}{J.~Zico Kolter}.}
  \bibinfo{year}{2021}\natexlab{}.
\newblock \showarticletitle{Estimating Lipschitz constants of monotone deep
  equilibrium models}. In \bibinfo{booktitle}{\emph{Proc. of ICLR}}.
\newblock


\bibitem[Paszke et~al\mbox{.}(2019)]%
        {PyTorch}
\bibfield{author}{\bibinfo{person}{Adam Paszke}, \bibinfo{person}{Sam Gross},
  \bibinfo{person}{Francisco Massa}, \bibinfo{person}{Adam Lerer},
  \bibinfo{person}{James Bradbury}, \bibinfo{person}{Gregory Chanan},
  \bibinfo{person}{Trevor Killeen}, \bibinfo{person}{Zeming Lin},
  \bibinfo{person}{Natalia Gimelshein}, \bibinfo{person}{Luca Antiga},
  \bibinfo{person}{Alban Desmaison}, \bibinfo{person}{Andreas K{\"{o}}pf},
  \bibinfo{person}{Edward Yang}, \bibinfo{person}{Zachary DeVito},
  \bibinfo{person}{Martin Raison}, \bibinfo{person}{Alykhan Tejani},
  \bibinfo{person}{Sasank Chilamkurthy}, \bibinfo{person}{Benoit Steiner},
  \bibinfo{person}{Lu Fang}, \bibinfo{person}{Junjie Bai}, {and}
  \bibinfo{person}{Soumith Chintala}.} \bibinfo{year}{2019}\natexlab{}.
\newblock \showarticletitle{PyTorch: An Imperative Style, High-Performance Deep
  Learning Library}. In \bibinfo{booktitle}{\emph{Advances in Neural
  Information Processing Systems 32: Annual Conference on Neural Information
  Processing Systems 2019, NeurIPS 2019, December 8-14, 2019, Vancouver, BC,
  Canada}}.
\newblock


\bibitem[Putot(2012)]%
        {putot2012static}
\bibfield{author}{\bibinfo{person}{Sylvie Putot}.}
  \bibinfo{year}{2012}\natexlab{}.
\newblock \showarticletitle{Static analysis of numerical programs and systems}.
\newblock \bibinfo{journal}{\emph{Habilitation {\`a} diriger des recherches,
  Universit{\'e} de Paris-Sud}} (\bibinfo{year}{2012}).
\newblock


\bibitem[Raghunathan et~al\mbox{.}(2018)]%
        {raghunathan2018semidefinite}
\bibfield{author}{\bibinfo{person}{Aditi Raghunathan}, \bibinfo{person}{Jacob
  Steinhardt}, {and} \bibinfo{person}{Percy Liang}.}
  \bibinfo{year}{2018}\natexlab{}.
\newblock \showarticletitle{Semidefinite relaxations for certifying robustness
  to adversarial examples}. In \bibinfo{booktitle}{\emph{Advances in Neural
  Information Processing Systems 31: Annual Conference on Neural Information
  Processing Systems 2018, NeurIPS 2018, December 3-8, 2018, Montr{\'{e}}al,
  Canada}}.
\newblock


\bibitem[Revay et~al\mbox{.}(2020)]%
        {LMonDEQRevay2020}
\bibfield{author}{\bibinfo{person}{Max Revay}, \bibinfo{person}{Ruigang Wang},
  {and} \bibinfo{person}{Ian~R. Manchester}.} \bibinfo{year}{2020}\natexlab{}.
\newblock \showarticletitle{Lipschitz Bounded Equilibrium Networks}.
\newblock \bibinfo{journal}{\emph{ArXiv preprint}}
  \bibinfo{volume}{abs/2010.01732} (\bibinfo{year}{2020}).
\newblock


\bibitem[Ryu and Boyd(2016)]%
        {ryu2016primer}
\bibfield{author}{\bibinfo{person}{Ernest~K Ryu} {and} \bibinfo{person}{Stephen
  Boyd}.} \bibinfo{year}{2016}\natexlab{}.
\newblock \showarticletitle{Primer on monotone operator methods}.
\newblock \bibinfo{journal}{\emph{Appl. Comput. Math}} \bibinfo{volume}{15},
  \bibinfo{number}{1} (\bibinfo{year}{2016}).
\newblock


\bibitem[Sadraddini and Tedrake(2019)]%
        {SadraddiniT19}
\bibfield{author}{\bibinfo{person}{Sadra Sadraddini} {and}
  \bibinfo{person}{Russ Tedrake}.} \bibinfo{year}{2019}\natexlab{}.
\newblock \showarticletitle{Linear Encodings for Polytope Containment
  Problems}. In \bibinfo{booktitle}{\emph{58th {IEEE} Conference on Decision
  and Control, {CDC} 2019, Nice, France, December 11-13, 2019}}.
\newblock
\urldef\tempurl%
\url{https://doi.org/10.1109/CDC40024.2019.9029363}
\showDOI{\tempurl}


\bibitem[Serre et~al\mbox{.}(2021)]%
        {gpupoly}
\bibfield{author}{\bibinfo{person}{Fran{\c c}ois Serre},
  \bibinfo{person}{Christoph M{\"u}ller}, \bibinfo{person}{Gagandeep Singh},
  \bibinfo{person}{Markus P{\"u}schel}, {and} \bibinfo{person}{Martin Vechev}.}
  \bibinfo{year}{2021}\natexlab{}.
\newblock \showarticletitle{Scaling Polyhedral Neural Network Verification on
  {GPU}s}. In \bibinfo{booktitle}{\emph{Proc. Machine Learning and Systems
  (MLSys)}}.
\newblock


\bibitem[Singh et~al\mbox{.}(2019a)]%
        {singh2019beyond}
\bibfield{author}{\bibinfo{person}{Gagandeep Singh}, \bibinfo{person}{Rupanshu
  Ganvir}, \bibinfo{person}{Markus P{\"{u}}schel}, {and}
  \bibinfo{person}{Martin~T. Vechev}.} \bibinfo{year}{2019}\natexlab{a}.
\newblock \showarticletitle{Beyond the Single Neuron Convex Barrier for Neural
  Network Certification}. In \bibinfo{booktitle}{\emph{Advances in Neural
  Information Processing Systems 32: Annual Conference on Neural Information
  Processing Systems 2019, NeurIPS 2019, December 8-14, 2019, Vancouver, BC,
  Canada}}.
\newblock


\bibitem[Singh et~al\mbox{.}(2018)]%
        {DeepZSinghGMPV18}
\bibfield{author}{\bibinfo{person}{Gagandeep Singh}, \bibinfo{person}{Timon
  Gehr}, \bibinfo{person}{Matthew Mirman}, \bibinfo{person}{Markus
  P{\"{u}}schel}, {and} \bibinfo{person}{Martin~T. Vechev}.}
  \bibinfo{year}{2018}\natexlab{}.
\newblock \showarticletitle{Fast and Effective Robustness Certification}. In
  \bibinfo{booktitle}{\emph{Advances in Neural Information Processing Systems
  31: Annual Conference on Neural Information Processing Systems 2018, NeurIPS
  2018, December 3-8, 2018, Montr{\'{e}}al, Canada}}.
\newblock


\bibitem[Singh et~al\mbox{.}(2019b)]%
        {SinghGPV19}
\bibfield{author}{\bibinfo{person}{Gagandeep Singh}, \bibinfo{person}{Timon
  Gehr}, \bibinfo{person}{Markus P{\"{u}}schel}, {and}
  \bibinfo{person}{Martin~T. Vechev}.} \bibinfo{year}{2019}\natexlab{b}.
\newblock \showarticletitle{An abstract domain for certifying neural networks}.
\newblock \bibinfo{journal}{\emph{{PACMPL}}} \bibinfo{volume}{3},
  \bibinfo{number}{{POPL}} (\bibinfo{year}{2019}).
\newblock
\urldef\tempurl%
\url{https://doi.org/10.1145/3290354}
\showDOI{\tempurl}


\bibitem[Szegedy et~al\mbox{.}(2014)]%
        {szegedy2013intriguing}
\bibfield{author}{\bibinfo{person}{Christian Szegedy},
  \bibinfo{person}{Wojciech Zaremba}, \bibinfo{person}{Ilya Sutskever},
  \bibinfo{person}{Joan Bruna}, \bibinfo{person}{Dumitru Erhan},
  \bibinfo{person}{Ian~J. Goodfellow}, {and} \bibinfo{person}{Rob Fergus}.}
  \bibinfo{year}{2014}\natexlab{}.
\newblock \showarticletitle{Intriguing properties of neural networks}. In
  \bibinfo{booktitle}{\emph{Proc. of ICLR}}.
\newblock


\bibitem[Tashiro et~al\mbox{.}(2020)]%
        {TashiroSE20}
\bibfield{author}{\bibinfo{person}{Yusuke Tashiro}, \bibinfo{person}{Yang
  Song}, {and} \bibinfo{person}{Stefano Ermon}.}
  \bibinfo{year}{2020}\natexlab{}.
\newblock \showarticletitle{Output Diversified Initialization for Adversarial
  Attacks}.
\newblock \bibinfo{journal}{\emph{ArXiv preprint}}
  \bibinfo{volume}{abs/2003.06878} (\bibinfo{year}{2020}).
\newblock


\bibitem[Wang et~al\mbox{.}(2019)]%
        {WangDWK19}
\bibfield{author}{\bibinfo{person}{Po{-}Wei Wang}, \bibinfo{person}{Priya~L.
  Donti}, \bibinfo{person}{Bryan Wilder}, {and} \bibinfo{person}{J.~Zico
  Kolter}.} \bibinfo{year}{2019}\natexlab{}.
\newblock \showarticletitle{SATNet: Bridging deep learning and logical
  reasoning using a differentiable satisfiability solver}. In
  \bibinfo{booktitle}{\emph{Proc. of ICML}}, Vol.~\bibinfo{volume}{97}.
\newblock


\bibitem[Wang et~al\mbox{.}(2018)]%
        {Wang19Formal}
\bibfield{author}{\bibinfo{person}{Shiqi Wang}, \bibinfo{person}{Kexin Pei},
  \bibinfo{person}{Justin Whitehouse}, \bibinfo{person}{Junfeng Yang}, {and}
  \bibinfo{person}{Suman Jana}.} \bibinfo{year}{2018}\natexlab{}.
\newblock \showarticletitle{Formal Security Analysis of Neural Networks using
  Symbolic Intervals}. In \bibinfo{booktitle}{\emph{27th {USENIX} Security
  Symposium, {USENIX} Security 2018, Baltimore, MD, USA, August 15-17, 2018}}.
\newblock


\bibitem[Wei and Kolter(2022)]%
        {wei2022certified}
\bibfield{author}{\bibinfo{person}{Colin Wei} {and} \bibinfo{person}{J~Zico
  Kolter}.} \bibinfo{year}{2022}\natexlab{}.
\newblock \showarticletitle{Certified Robustness for Deep Equilibrium Models
  via Interval Bound Propagation}. In \bibinfo{booktitle}{\emph{International
  Conference on Learning Representations}}.
\newblock


\bibitem[Weng et~al\mbox{.}(2018)]%
        {WengZCSHDBD18}
\bibfield{author}{\bibinfo{person}{Tsui{-}Wei Weng}, \bibinfo{person}{Huan
  Zhang}, \bibinfo{person}{Hongge Chen}, \bibinfo{person}{Zhao Song},
  \bibinfo{person}{Cho{-}Jui Hsieh}, \bibinfo{person}{Luca Daniel},
  \bibinfo{person}{Duane~S. Boning}, {and} \bibinfo{person}{Inderjit~S.
  Dhillon}.} \bibinfo{year}{2018}\natexlab{}.
\newblock \showarticletitle{Towards Fast Computation of Certified Robustness
  for ReLU Networks}. In \bibinfo{booktitle}{\emph{Proc. of ICML}},
  Vol.~\bibinfo{volume}{80}.
\newblock


\bibitem[Winston and Kolter(2020)]%
        {MonDEQWinstonK20}
\bibfield{author}{\bibinfo{person}{Ezra Winston} {and} \bibinfo{person}{J.~Zico
  Kolter}.} \bibinfo{year}{2020}\natexlab{}.
\newblock \showarticletitle{Monotone operator equilibrium networks}. In
  \bibinfo{booktitle}{\emph{Advances in Neural Information Processing Systems
  33: Annual Conference on Neural Information Processing Systems 2020, NeurIPS
  2020, December 6-12, 2020, virtual}}.
\newblock


\bibitem[Wong and Kolter(2018)]%
        {WongK18}
\bibfield{author}{\bibinfo{person}{Eric Wong} {and} \bibinfo{person}{J.~Zico
  Kolter}.} \bibinfo{year}{2018}\natexlab{}.
\newblock \showarticletitle{Provable Defenses against Adversarial Examples via
  the Convex Outer Adversarial Polytope}. In \bibinfo{booktitle}{\emph{Proc. of
  ICML}}, Vol.~\bibinfo{volume}{80}.
\newblock


\bibitem[Xu et~al\mbox{.}(2020)]%
        {xu2020automatic}
\bibfield{author}{\bibinfo{person}{Kaidi Xu}, \bibinfo{person}{Zhouxing Shi},
  \bibinfo{person}{Huan Zhang}, \bibinfo{person}{Yihan Wang},
  \bibinfo{person}{Kai{-}Wei Chang}, \bibinfo{person}{Minlie Huang},
  \bibinfo{person}{Bhavya Kailkhura}, \bibinfo{person}{Xue Lin}, {and}
  \bibinfo{person}{Cho{-}Jui Hsieh}.} \bibinfo{year}{2020}\natexlab{}.
\newblock \showarticletitle{Automatic Perturbation Analysis for Scalable
  Certified Robustness and Beyond}. In \bibinfo{booktitle}{\emph{Advances in
  Neural Information Processing Systems 33: Annual Conference on Neural
  Information Processing Systems 2020, NeurIPS 2020, December 6-12, 2020,
  virtual}}.
\newblock


\bibitem[Yang and Scott(2018)]%
        {yang2018comparison}
\bibfield{author}{\bibinfo{person}{Xuejiao Yang} {and}
  \bibinfo{person}{Joseph~K Scott}.} \bibinfo{year}{2018}\natexlab{}.
\newblock \showarticletitle{A comparison of zonotope order reduction
  techniques}.
\newblock \bibinfo{journal}{\emph{Automatica}}  \bibinfo{volume}{95}
  (\bibinfo{year}{2018}).
\newblock


\bibitem[Yazarel and Pappas(2004)]%
        {yazarel2004geometric}
\bibfield{author}{\bibinfo{person}{Hakan Yazarel} {and}
  \bibinfo{person}{George~J Pappas}.} \bibinfo{year}{2004}\natexlab{}.
\newblock \showarticletitle{Geometric programming relaxations for linear system
  reachability}. In \bibinfo{booktitle}{\emph{Proceedings of the 2004 American
  Control Conference}}, Vol.~\bibinfo{volume}{1}. IEEE.
\newblock


\bibitem[Zhang et~al\mbox{.}(2018)]%
        {zhang2018crown}
\bibfield{author}{\bibinfo{person}{Huan Zhang}, \bibinfo{person}{Tsui{-}Wei
  Weng}, \bibinfo{person}{Pin{-}Yu Chen}, \bibinfo{person}{Cho{-}Jui Hsieh},
  {and} \bibinfo{person}{Luca Daniel}.} \bibinfo{year}{2018}\natexlab{}.
\newblock \showarticletitle{Efficient Neural Network Robustness Certification
  with General Activation Functions}. In \bibinfo{booktitle}{\emph{Advances in
  Neural Information Processing Systems 31: Annual Conference on Neural
  Information Processing Systems 2018, NeurIPS 2018, December 3-8, 2018,
  Montr{\'{e}}al, Canada}}.
\newblock


\end{thebibliography}

\message{^^JLASTREFERENCESPAGE \thepage^^J}

\ifbool{includeappendix}{%
	\clearpage
	\appendix
 	\section{Case Study: Analysis of Square Root Approximation} \label{app:householder}

\begin{wraptable}[12]{r}{0.52\textwidth}
	\vspace{-4.5mm}
	\renewcommand{\arraystretch}{1.2}
	\caption{Extended version of \cref{tab:householder} including not only the \tool over-approximation of the true mathematical fixpoints (\tool fix) but also that of all reachable values (\tool reach).}
	\vspace{-2mm}
	\centering
		\scalebox{0.87}{
	\begin{threeparttable}
		\centering
		\begin{tabular}{ll cc}
			\toprule
			\multicolumn{2}{c}{\multirow{2.5}{*}{Method}} &\multicolumn{2}{c}{Root Interval $1/\gamma(\S^*)$}\\
			\cmidrule{3-4}
			& &\text{$\X$ = [16,20]} & \text{$\X$ = [16,25]}\\
			\midrule
			Exact & $\S^*$ & [4.000, 4.472] & [4.000, 5.000] \\
			\tool fix &  $\hat{\S}^*_\text{cr,f}$ & [3.983, 4.493] & [3.887, 5.104] \\
			\tool reach & $\hat{\S}^*_\text{cr,r}$ & [3.982, 4.495] & [3.885, 5.106] \\
			Kleene iteration & $\hat{\S}^*_\text{kl}$ & [3.738, 4.520] & [0.000, $\,\:\:\infty\,\:\:$) \\
			\bottomrule
		\end{tabular}
	\end{threeparttable}
		}
	\label{tab:householder_2}
	\vspace{-3mm}
\end{wraptable}

In this section, we expand our case study on the analysis of Householder's method to consider the effect of different termination criteria

\paragraph{Accounting for the Termination Criterion}
A standard application of \tool computes an over-approximation of the true mathematical fixpoints of the abstracted iteration, i.e., in this case, the (reciprocal of the) square roots. We can, however, also analyze a termination condition to compute the set of all reachable values for this specific termination condition. 

To this end, we first bound the difference $\delta$ between any $s$ satisfying the termination condition and the true mathematical fixpoint $1/\sqrt{x}$:

\begin{theorem}\label{thm:max_dev}
	Given the termination condition \lstinline{$s > 0$ and $|s*s - 1/x| < \epsilon$} for the program \lstinline{root} in \cref{fig:root_iter}, let  $\epsilon < \frac{1}{x}$ and $\delta = s - 1 / \sqrt{x}$, then we have $|\delta|<\sqrt{\epsilon}$ for all $s$ that may be returned by \lstinline{root}.
\end{theorem}
\begin{proof}
	We set $s = \frac{1}{\sqrt{x}} + \delta$ and obtain:
	\begin{align*}
	\epsilon >	\left|s^2 - \frac{1}{x} \right|  =	\left|\left(\frac{1}{\sqrt{x}} + \delta\right)^2 - \frac{1}{x} \right|  =	\left|\frac{2 \delta}{\sqrt{x}} + \delta^2 \right|.
	\end{align*}
	We first consider the positive case of the absolute value:
	\begin{align*}
	\frac{2\delta}{\sqrt{x}} + \delta^2  = \epsilon \\
	\delta = \frac{-\frac{2}{\sqrt{x}} \pm \sqrt{\frac{4}{x}+ 4 \epsilon}}{2} = -\sqrt{\frac{1}{x}} \pm \sqrt{\frac{1}{x}+ \epsilon} < \sqrt{\epsilon},
	\end{align*}
	where we use $s>0 \implies \delta > - \frac{1}{\sqrt{x}}$ to rule out one branch and the fact that $\sqrt{\cdot}$ is a concave function. We now consider the negative case of the absolute value:
	\begin{align}
	\frac{2\delta}{\sqrt{x}} + \delta^2  = -\epsilon \\
	\delta = \frac{-\frac{2}{\sqrt{x}} \pm \sqrt{\frac{4}{x} - 4 \epsilon}}{2} = -\sqrt{\frac{1}{x}} \pm \sqrt{\frac{1}{x} - \epsilon} > -\sqrt{\epsilon}
	\end{align}
	where we again use $s>0 \implies \delta > - \frac{1}{\sqrt{x}}$ to rule out one branch and assume that $\epsilon$ is chosen suitably ($\epsilon < \frac{1}{x}$).
\end{proof}

Using this result, we can simply expand any obtained abstraction by $\sqrt{\epsilon} \geq |\delta|$ to ensure that we capture all reachable values:

\begin{theorem}
	Let $\hat{\S}^*$ be a sound over-approximation of the true fixpoint set of the iteration in \lstinline{root} for $x \in \X$. Further, let $\bar{\delta} = \sqrt{\epsilon}$ be the maximum difference between any $s$ satisfying the termination condition and the corresponding true fixpoint $1/\sqrt{x}$ for any $x \in \R^{\geq 0}$ and $0 < \epsilon < \frac{1}{x}$. Then, the Minkowski sum $\gamma(\hat{\S}^*) + [-\bar{\delta}, \bar{\delta}]$ is a sound over-approximation of all reachable values of \lstinline{root} for $x \in \X$.
\end{theorem}
\begin{proof}
	By \cref{thm:max_dev}, the difference $\delta = s - 1 / \sqrt{x}$ between any $s$ satisfying the termination condition of \lstinline{root} and the corresponding true fixpoint $1 / \sqrt{x}$ is bounded by $\bar{\delta} = \sqrt{\epsilon}$.
	By condition, we have that this true fixpoint $s^* = 1 / \sqrt{x}$ is included in $\gamma(\hat{\S}^*)$.
	Thus, $s = s^* + \delta$ with $|\delta| \leq \bar{\delta}$ is included in the Minkowski sum $\gamma(\hat{\S}^*) + [-\bar{\delta}, \bar{\delta}]$ by its definition.
	Note that \lstinline{root} only returns values $s$ satisfying the termination condition to conclude the proof.
\end{proof}

This result allows us to expand our obtained fixpoint set $\hat{\S}^*_\text{cr,f}$ marginally to obtain the reachable values $\hat{\S}^*_\text{cr,r}$, still yielding a much tighter over-approximation than Kleene iteration (see \cref{tab:householder_2}).
We expect that similar results can be obtained for many interesting termination conditions.

\section{Deferred Proofs \& Further Theory} \label{sec:proofs}

Below we provide the detailed proofs for our key theorems, restating them for convenience.

\subsection{Abstracting Fixpoint Iterations}
\paragraph{Contraction-Based Termination Condition}
Below, we prove \cref{the:contraction} on the soundness of our contraction-based termination condition.

\contraction*

\begin{proof}
  We have
  \begin{equation}\label{eqn:cont_proof_2}
    \S_{n+1} \subseteq \gamma(\hat{\S}_{n+1}) \subseteq \gamma(\hat{\S}_{n})
  \end{equation}
  where the first $\subseteq$ holds by definition and the second follows from the left hand side of \cref{eq:contraction} and the definition of $\sqsubseteq$.
  Now, we can over-approximate ${\S}_{n+1}$ with $\hat{\S}_{n} \sqsupseteq \hat{\S}_{n+1}$ and $\S_{n+2} \subseteq \gamma(\hat{\S}_{n+1})$ follows immediately via \cref{eqn:cont_proof_2}. Thus, $\S_{j} \subseteq \hat{\S}_{n+1} $ for $j > n$ follows by induction and thereby $\Z_{j} \subseteq \gamma(\hat{\Z}_{n+1}) \forall j > n$.

  By the convergence guarantee of the concrete iteration we have:
  for any $\eps \in \R^{> 0}$ there exists a $j \in \bb{N}$ with $j \geq n+1$ such that we have $\norm{\vz_{j} - \vz^*} \leq \eps$.
  By the definition of $\Z_j$ we also have $\vz_j \in \Z_j$.
  For $\eps \to 0$ and hence $\norm{\vz_j - \vz^*} \to 0$ it follows that $\vz^* \in \Z_j \cup \partial \Z_j = \overline{\Z}_j$.
  Thus for each $\vx \in \X$, there exists a $j_{\vx}$ such that $\vz^*(\vx) \in \overline{\Z}_{j_\vx} \subseteq \gamma(\hat{\Z}_{n+1})$, where we get the inclusion relation from above and the subset relation $\overline{\Z}_{j_\vx} \subseteq \gamma(\hat{\Z}_{n+1})$ from the closedness of $\gamma(\hat{\Z})$ .
  Finally, $\Z^{*} \subseteq \bigcup_{x\in\X} \overline{\Z}_{j_x} \subseteq \gamma(\hat{\Z}_{n+1})$.
\end{proof}

\paragraph{Fixpoint Set Preservation}
We now prove our results on fixpoint set preservation (\cref{the:iter_FWBW,the:iter_PR}), starting with our statement for sound abstractions of locally Lipschitz iterators with convergence guarantees.

\iterpw*

\begin{proof}
	To prove by contradiction, let $\tilde{\vz}$ be a point close to the fixpoint $\vz^*$ s.t. $\norm{\tilde{z}-z^*} \leq \eps$ with map $\tilde{\vz}' = \vg_\alpha(\vx,\tilde{\vz},\vu)$ under the fixpoint iterator. Let us further assume that $\forall \vu \in \U^*$ for some $\U^*$ an application of $\vz' = \vg_\alpha(\vx,\vz^*,\vu)$ does not map back to $\vz^*$, i.e., $\norm{\vz^*-\vz'} > d$.
	\begin{itemize}
		\item Recall that $\vg_\alpha(\vx,\tilde{\vz},\vu)$ is locally Lipschitz with $L < \infty$ by assumption.
		\item It follows from $\norm{\tilde{z}-z^*} \leq \eps$ that $\norm{\vg_\alpha(\vx,\tilde{\vz},\vu) - \vg_\alpha(\vx,\vz^*,\vu)} = \norm{\tilde{\vz}' - \vz'} \leq L \eps$.
		\item Hence by the inverse triangle inequality $\norm{\vz^* - \tilde{\vz}'} \geq \big \rVert \norm{\vz^*-\vz'} - \norm{\tilde{\vz}' - \vz'}\big\lVert \geq d - L \eps$
		\item Choose $\eps < d / (L+1) \implies\norm{\vz^* - \tilde{\vz}'} \geq  d - L \eps > \eps$. Note that $d$ does not depend on $\epsilon$.
		\item It follows that $\norm{\vg_\alpha(\vx,\tilde{\vz},\vu) - \vz^*} > \eps \quad \forall \tilde{\vz} \in \{\vz \mid \, \norm{\vz-\vz^*} \leq \eps\}$ which contradicts the convergence guarantee.
	\end{itemize}
	By contradiction, it follows that $\exists \vu \in \U^*: \vz^* = \vg_\alpha(\vx,\vz^*,\vu)$.
\end{proof}

This result implies that we can apply any sound abstractions of a locally Lipschitz iterator to an over-approximation of the true fixpoint set and obtain a possibly tighter over-approximation of the true fixpoint set.
Note that $\vg_\alpha(\vx,\vz,\vu)$ is locally Lipschitz with $L < \infty$ in $\vu$, $\vz$, and $\vx$ for both \pr and \fwdbwd as they are the composition of linear maps of finite widths with globally Lipschitz functions, and thus that \cref{the:iter_PR} is applicable.

Now we additionally show fixpoint set preservation for \fwdbwd for a wider range of $\alpha$:

\iterfwbw*
\begin{proof}
	As we consider only $\vf$ and $\gfwdbwd$ here, we have $\vs \coloneqq \vz$.
	For any concrete fixpoint $\vz^* = \vf(\vx,\vz^*) = ReLU(\mW\vz^* + \mU\vx + \vb)$, we consider an iteration of Forward-Backward splitting as per \cref{eqn:fwd_bwd} with $\vz_{n} = \vz^*$:
	\begin{equation*}
	\vz_{n+1} = ReLU((1-\alpha)\vz^* + \alpha \underbrace{(\mW\vz^* + \mU\vx + \vb)}_{\vz'}) = \vz^{*}
	\end{equation*}
	We show this by considering the expression element-wise. Suppose $z^{'} \leq 0$, then due to $z^{*} = f(x, z^{*}) = ReLU(z')$, we know $z^* = 0$ and else $z^{*} = z'$. Then
	\begin{align*}
	&ReLU( (1-\alpha)z^* + \alpha z' )
	= \left\{
	\begin{array}{@{}l@{}ll@{}}
	ReLU( (1-\alpha)0   &+ \alpha z' ) & \text{if } z' \leq 0\\
	ReLU( (1-\alpha)z^* &+ \alpha z^* )  & \text{else}  \\
	\end{array}
	\right\} = z^{*}.
	\end{align*}
	In the first case we know $(1-\alpha)0 + \alpha z' \leq 0$ as $z' \leq 0$.
	It follows that one step of Forward-Backward splitting will always map a fixpoint upon itself in the concrete. %
	Since $\hat{\Z}_{n}$ includes all fixpoints for $\X$, any sound $\hat{\Z}_{n+1}  =  \vg^{\fwdbwd\#}_\alpha(\X,\hat{\Z}_{n})$ includes all fixpoints for $\X$.
\end{proof}

\subsection{Error Consolidation}\label{app:consolidating}
Below we prove \cref{thm:consolidation} on the soundness of the error consolidation of an improper \domain.
\consolidation*

\begin{proof}
	Without loss of generality let $\va=\mbf{0}$, $\vb=\mbf{0}$ and $\hat{\Z} = \mA \ez = \sum_{j=1}^{k} \mA_{j} \eez_{j}$ with $k$ error terms, stored in the columns of $\mA$. We can express the contribution of every error term as $\mA_{j} \eez_j = \tilde{\mA} \tilde{\ez}'^{(j)}$ with $\tilde{\ez}'^{(j)} = \tilde{\mA}^{-1} \mA_{j} \eez_{j}$ as $\tilde{\mA}$ is a basis of $\R^p$ and hence invertible. From $\eez_{j} \in [-1, 1]$ it follows that $\tilde{\ez}'^{(j)} \in \diag(\tilde{\mA}^{-1} \mA_{j}) \tilde{\ez}^{(j)}$ with $\tilde{\ez}^{(j)} \in [-1, 1]^{p}$.
	This allows us to rewrite %
	\begin{align*}
	\hat{\Z}  &=
	\left\{
	\begin{array}{@{}l@{}}
	\tilde{\mA} \sum_{j=1}^{k} \tilde{\mA}^{-1} \mA_{j} \eez_{j}\\
	\forall\, \ez \in [-1, 1]^{k}
	\end{array}
	\right\}
	\subseteq
	\left\{
	\begin{array}{@{}l@{}}
	\tilde{\mA}  \sum_{j=1}^{k} \diag(\tilde{\mA}^{-1} \mA_{j}) \tilde{\ez}^{(j)}\\
	\forall\, \tilde{\ez}^{(1)}, \dots,\tilde{\ez}^{(k)} \in [-1, 1]^{p}
	\end{array}
	\right\}\\
	&=
	\left\{
	\begin{array}{@{}l@{}}
	\tilde{\mA} \diag(|\tilde{\mA}^{-1} \mA| \mbf{1}) \tilde{\ez}\\
	\forall\, \tilde{\ez} \in [-1, 1]^{p}
	\end{array}
	\right\} =
	\hat{\Z}',
	\end{align*}
	where the second last equality follows from linearity and the choice $\tilde{\ez}_{j} = \pm \sign(\tilde{\mA}^{-1} \mA_j)$.
\end{proof}

\subsection{\domain Containment}\label{app:containment}
Below, we prove \cref{the:containment} on the containment of (improper) \domain in proper \domain.
\containment*

\begin{proof}
	Containment is equivalent to showing that for all error terms $\ez' \in [-1,1]^k, \eb' \in [-1,1]^p$ describing points in $\hat{\Z}'$, there exist $\ez \in [-1,1]^p, \eb \in [-1,1]^p$ of $\hat{\Z}$ such that: %
	$$\mA\ez + \diag(\vb)\eb + \va =  \mA'\ez' + \diag(\vb')\eb' + \va'.$$
	We subtract $\va$ from both sides and over-approximate the right hand side by increasing the Box size by the absolute center difference $|\va' - \va|$ yielding $\vb'' \coloneqq \vb' + |\va' - \va|$. This leaves us to show that we can find $\ez, \eb$ such that $\mA\ez + \diag(\vb)\eb =  \mA'\ez' + \diag(\vb'')\eb''$ holds for all $\ez'$ and $\eb''$ with $\eb'' \in [-1,1]^p$.
	
	We choose $\eb \in [-1, 1]^{p}$ such that
	\begin{align*}
	\diag(\vb)\eb = \sign(\eb'') \min(\vb, \diag(\vb'')|\eb''|),
	\end{align*}
	guaranteed to yield $|\eb| \leq 1$ and obtain by substitution
	\begin{align*}
	\mA\ez  &=  \mA'\ez' + \max(0, \diag(\vb'')\eb'' - \vb)\\
	\ez  &= \mA^{-1} \mA'\ez' + \mA^{-1}  \max(\mbf{0}, \diag(\vb'')\eb'' - \vb)\\
	&\stackrel{(*)}{\leq} |\mA^{-1} \mA'|\mbf{1} + |\mA^{-1}  \diag(\max(\mbf{0}, \vb'' - \vb))|\mbf{1} \stackrel{(**)}{\leq}  \mbf{1},
	\end{align*}
	where in $(*)$ we use the relation shown in \cref{thm:consolidation} for the sound representation of a decomposition and the fact that setting $\eb''$ to a one vector maximizes $\diag(\vb'')\eb''$ and $(**)$ follows directly from the condition of \cref{eq:containment}.
	Taking the absolute value we obtain $|\ez| \leq \mbf{1}$ and have shown that both $\ez$ and $\eb$ exist.
\end{proof}

\subsection{$\mbf{s}$-Step Fixpoint Contraction}
Below, we show, that not only post-fixpoints of a single application of the iterator $\g$ contain all true fixpoints, but also post-fixpoints of $s$ unrolled applications of $\g$:
\begin{restatable}[$\mbf{s}$-Step Fixpoint contraction]{theorem}{scontraction}
	\sloppy
	\label{the:contraction:sstep}
	Let
	\begin{itemize}
		\item $s \in \mathbb{N}^{\geq 1}$
		\item \g be an iterative solver converging to a unique fixpoint $\vz^*$ in finitely many steps for any bounded input,
		\item \gS its sound abstract transformer,
		\item $\hat{\S}_{n+1} \coloneqq \gS(\hat{\X},\hat{\S}_n)$ an abstract element in $\aid$ describing a closed set and denoting an over-approximation of applying \g $n+1$ times for some $\vz_{0},\vu_{0}$ on all inputs $\vx \in \hat{\X}$.
	\end{itemize}
	Then for $[\hat{\Z}_n, \hat{\U}_n] \gets \hat{\S}_n$:
	\begin{equation}
	\label{eq:contraction:sstep}
	\hat{\S}_{n+s} \sqsubseteq \hat{\S}_n
	\mspace{9.0mu}\implies\mspace{9.0mu} \Z^{*} \subseteq \gamma(\hat{\Z}_{n+s}).
	\end{equation}
\end{restatable}

\begin{proof}
  Let $\vs_{n+s} = \vg'_{s}(\vs_{n}) = \vg'( \cdots \vg'(\vs_{n}))$.
  Then \cref{eq:contraction:sstep} follows directly from  applying \cref{the:contraction} to $\vg'_{s}$.
\end{proof}

\subsection{Correctness of \tool}
Below, we provide a slightly extended proof of \cref{thm:soundness} on the correctness of \tool.

\soundness*

\begin{proof}
	$ $\\[-1em]
	\begin{enumerate}
		\item Since \gSss{}{1} is a sound over-approximation of $\vg_{\alpha_1}$, so is $\gSss{*}{1} \coloneqq \expand \circ \consolidate \circ \gSss{}{1}$.
		Therefore, \gSss{*}{1} also fulfils \cref{the:contraction}.
		Thus, showing containment $\hat{\S_{n}} \sqsupseteq \hat{\S}_{n+1} = \gSss{*}{1}(\hat{\X}, \hat{\S_{n}})$, implies $\S^* \subseteq \gamma(\hat{\S}_n)$.
		\item Since, $\S^* \subseteq \gamma(\hat{\S}_n)$ and $\gSss{}{2}$ preserves this property by \cref{thm:preserving}, we have $\Z^* \subseteq \gamma(\hat{\Z}_k)$ for all $k \geq n$.
		As $\hat{\Y} \gets \mV \hat{\Z}_k + \vv$ is exact, and therefore also sound, $\post(\hat{\Y})$ implies $\pre(\vx) \models \post(\vh(\vx))$.
	\end{enumerate}
\end{proof}

\subsection{Other Activation Functions} \label{app:activation_functions}
In order for \tool to be able to certify monDEQs utilizing an activation function $\sigma$ other than $ReLU$, we require:
\begin{itemize}
  \item We need convergence and uniqueness guarantees for the original monDEQ in the concrete (via operator splitting); to this end Theorem 1 in \citet{MonDEQWinstonK20} requires $\sigma$ to be a proximal operator of a CCP function, which most common Deep Learning activation functions are.
  \item In order to utilize \fwdbwd for $\alpha$ without convergence guarantee in the second stage of \tool, a version of \cref{the:iter_FWBW} would be needed, which shows that Forward-Backward splitting still preserves fixpoints as our proof of \cref{the:iter_FWBW} relies on the ReLU function. Both \pr and \fwdbwd splitting can however be used leveraging \cref{the:iter_PR}, as longs as the chosen $\alpha$ is guaranteed to lead to convergence.
  \item Lastly, we need a \domain  transformer for the activation function. For many choices, existing Zonotope transformers
  (such as those for Sigmoid and Tanh discussed in \citet{DeepZSinghGMPV18}) can be adapted easily.
\end{itemize}

\section{Implementation Details}
\label{sec:impl-deta}
\cref{alg:verif} is a slightly simplified version of the \tool algorithm that we actually implemented.
Here, we discuss the differences.

\paragraph{Consolidation and Inclusion check}
In practice we perform the consolidation (line~\ref{alg:verif:consolidate}) only every $r^{\text{th}}$ iteration and only recompute the PCA basis for consolidation every $30$ steps.
Since we require a consolidated basis for the inclusion check (line~\ref{alg:verif:check})
we always keep the 10 most recently consolidated abstract iteration states $\hat{\S}_i$ and check the current
$\hat{\S}$ against all of these.
Note that this requires the use of \cref{the:contraction:sstep} rather than \cref{the:contraction}.

\paragraph{Abortion Heuristics}
In practice we abort the main loop early in cases where we likely wont be able to verify the input.
Before convergence, we abort if the concretization of the \domain $\hat{\S}_i$ reaches a width of $10^{9}$ in any direction.
After convergence we abort if in $3r'$ steps we did not observe any improvement in
$\max_{i \neq t} \gamma(\hat{\Y}_{t}- \hat{\Y}_{i})$.

\section{Parameter Choices \& Experimental Details}
\label{sec:parameter-choices}

\subsection{Model Training}
\label{sec:model-training}
 All monDEQs were trained with monotonicity parameter $m=20.0$ using standard minibatch gradient descent and implicit differentiation as outlined in \citet{MonDEQWinstonK20}. We use a batchsize of $128$ and train for $10$ epochs.

\subsection{\tool Parameters}
 \begin{wraptable}[10]{r}{0.45\textwidth}
	\vspace{-5mm}
	\renewcommand{\arraystretch}{1.0}
	\centering
	\footnotesize
	\caption{\tool verification parameters.}
	\label{tab:parms}
	\vspace{-3mm}
	\scalebox{1.0}{
		\centering
		\begin{tabular}{@{}lllrrr}
			\toprule
			\textit{Dataset} & \textit{Model} & $r$ & $r'$ & $\alpha_{\text{\pr}}$ & expansion\\
			\midrule
			\mnist
			& \fcf & 3 & 50 & 0.1 & const\\
			& \fces & 3 & 50 & 0.1 & const\\
			& \fco & 5 & 50 & 0.06 & const\\
			& \fct & 5 & 50 & 0.05 & const\\
			& \convsm & 5 & 50 & 0.05 & -\\
			\midrule
			\cifar
			& \fct & 3 & 30 & 0.06 & exp\\
			& \convsm & 3 & 30 & 0.06 & exp\\
			\bottomrule
		\end{tabular}
	}
\end{wraptable}
\label{sec:tool-parameters}
Generally, we use the default values discussed in \cref{sec:impl-deta} unless stated otherwise.
For all experiments we use $n_{max} = 500$ and summarize the main parameters in \cref{tab:parms}.
By default we use $r = 3$ and increase it to $5$ on larger \mnist models for better convergence.

For expansion, `const'  denotes $w_{mul} = 10^{-3}$ and $w_{add} = 10^{-2}$,
`exp' denotes initialization with `const' and scaling by 1.1 and 1.2 respectively every second iteration, and
`-' denotes no expansion.

All parameters and in particular the values for $\alpha$ used in \pr were found by coarse manual search.
Overall, we observe large stability with respect to most parameters and in particular the value of $\alpha$ does not have a large impact on \pr, as we show in  \cref{sec:eval:ablation}.

When switching from \pr to \fwdbwd, we apply a line search to determine an optimal $\alpha_{2}$ (with regard to certification).

\subsection{Adversarial Attack}
\label{sec:adversarial-attack}
In order to determine a bound on the certifiable accuracy of the models, we compute their empirical accuracy
with respect to a strong attack.
We apply a targeted version (towards all classes) of PGD \citep{MadryMSTV18}
with $20$ restarts and $50$ steps utilizing margin loss \citep{GowalUQHMK19} after 5 steps of output diversification \citep{TashiroSE20}.

\subsection{Comparison with Lipschitz-Bound-Based Methods} \label{app:lipschitz_comp}
Three existing works derive Lipschitz-Bounds for monDEQs, either via a posteriori analysis \citep{ELMonDEQPabbarajuWK21,chen2021semialgebraic} or construction \citep{LMonDEQRevay2020}. 

The by-construction approach of \citet{LMonDEQRevay2020} can yield relatively small Lipschitz bounds (w.r.t. $\ell_2$-perturbations). However, this comes at the cost of a significant natural accuracy reduction. Further, \citet{LMonDEQRevay2020} provide no robustness certificates of any form and have, to the best of our knowledge, not addressed recent questions regarding the theoretical validity of some of their results \citep{LBEN_OpenReview}.

\citet{ELMonDEQPabbarajuWK21} compute bounds on the Lipschitz constant with respect to the $\ell_2$ norm, denoted $L_2$. To derive robustness certificates w.r.t. $\ell_\infty$-perturbations, these have to be converted to bounds w.r.t. the $\ell_\infty$-norm: $L_\infty \leq \sqrt{q} L_2$ where $q$ is the input dimensionality. Equivalently, to obtain a robustness certificate for $\ell_\infty$-perturbations with radius $\epsilon_\infty$, robustness to $\ell_2$-perturbations with $\epsilon_2 = \sqrt{q} \epsilon_\infty$ has to be shown. The largest perturbations that \citet{ELMonDEQPabbarajuWK21} report any bounds on is $\epsilon_2 = 0.20$ corresponding to $\epsilon_\infty = 0.0071$ (on \mnist with $q=784$). There, they can only show a certified accuracy of $50\%$ while we obtain $99\%$ at the larger $\epsilon_\infty = 0.01$.

Only \citet{chen2021semialgebraic} compute Lipschitz-bounds $L_\infty$ directly w.r.t. the $\ell_\infty$ norm. While they demonstrate that these are almost an order of magnitude tighter than those obtained by converting $\ell_2$ Lipschitz bounds, they are still significantly less precise than their `Robustness Model' to which we compare favorably above.

\section{Additional Ablation Studies} \label{app:ablation}
\begin{wrapfigure}[10]{r}{0.47 \textwidth}
	\vspace{-16mm}
	\begin{subfigure}[t]{.46\linewidth}
		\centering
		\includegraphics[width=1.0\linewidth]{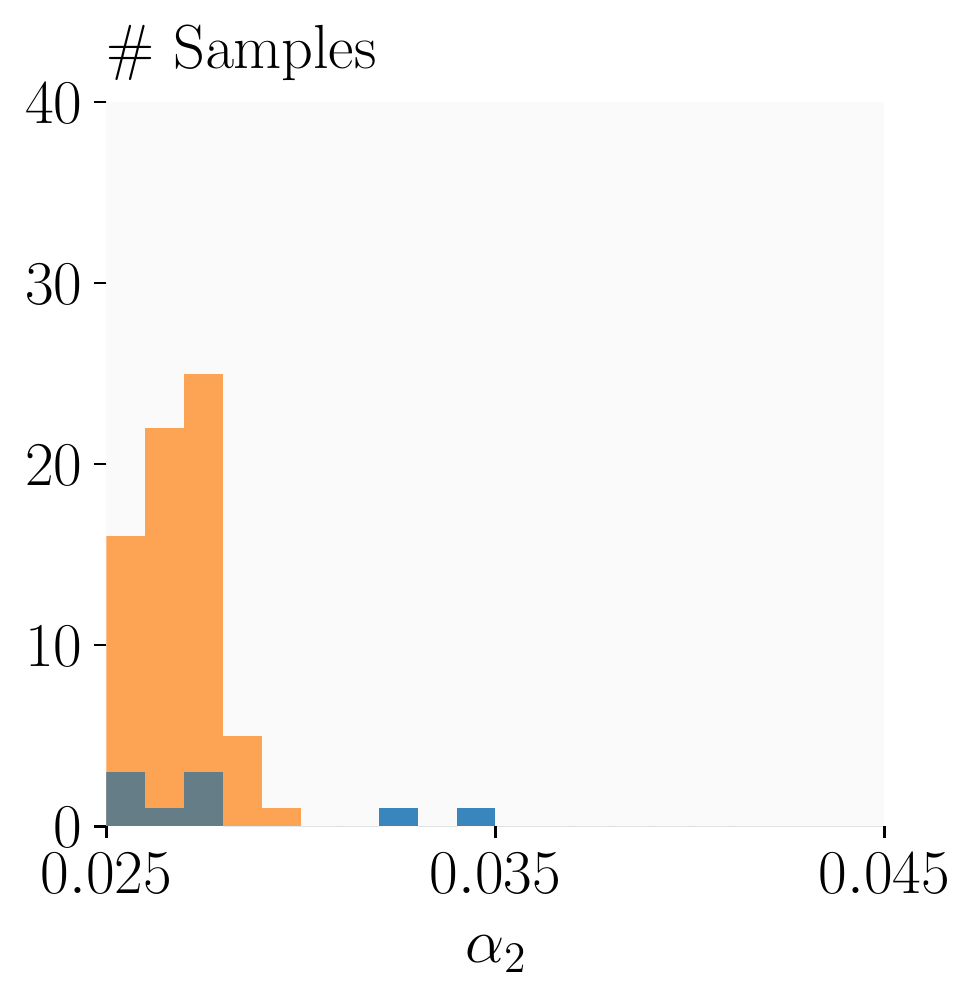}
		\vspace{-7mm}
		\caption{$\alpha_1 = 0.02$}
		\vspace{-1mm}
	\end{subfigure}
	\hfil
	\begin{subfigure}[t]{.46\linewidth}
		\centering
		\includegraphics[width=1.0\linewidth]{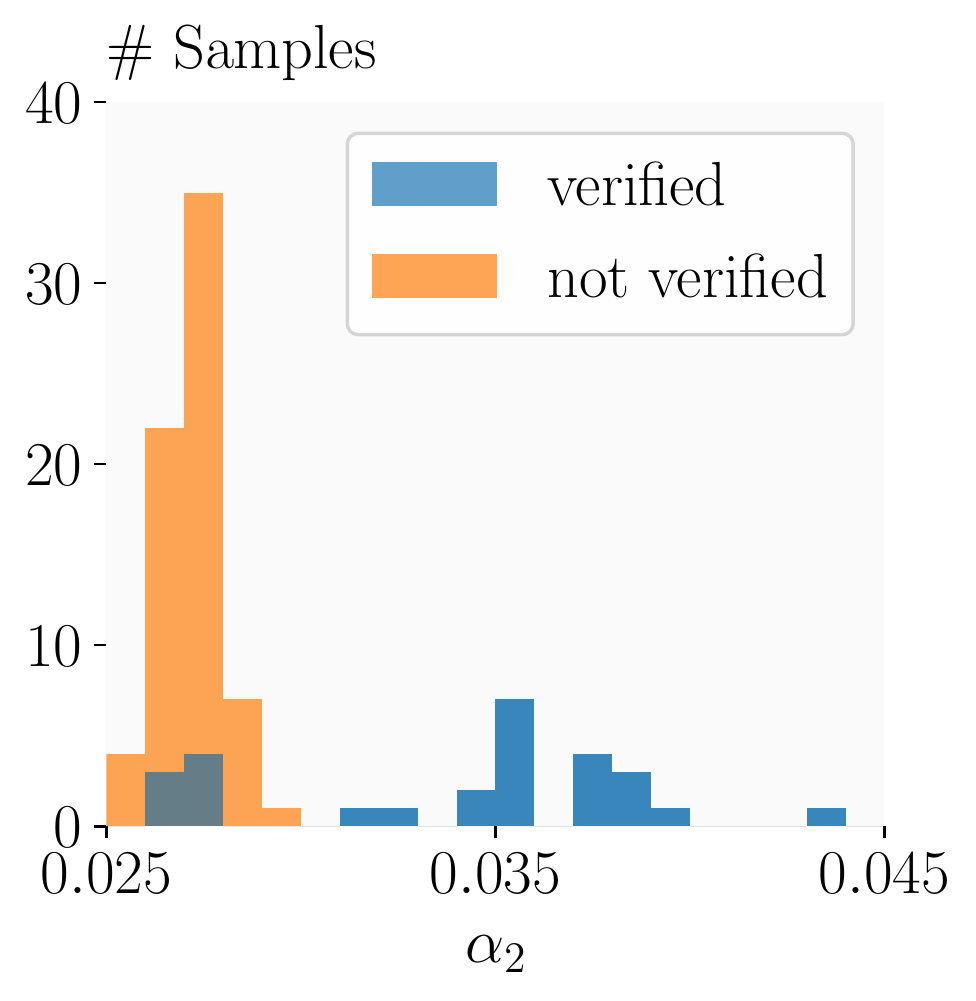}
		\vspace{-7mm}
		\caption{$\alpha_1 = 0.12$}
		\vspace{-1mm}
	\end{subfigure}
	\vspace{-2mm}
	\caption{Distributions of $\alpha_2$ chosen for $\gSss{\fwdbwd}{2}$ depending on $\alpha_1$ chosen for $\gSss{\pr}{1}$. Both settings verify the same number of samples. Only samples not verified with \pr are shown.}
	\label{fig:alpha_hist}
	\vspace{-2mm}
\end{wrapfigure}
\subsection{Adaptive $\alpha$}\label{app:ablation_adaptive_alpha}
Once we have shown containment, we can apply an arbitrary number of fixpoint set preserving abstract iterations ($\gSss{}{2}$ in \cref{alg:verif}) to tighten our abstraction. Since we have shown \gfwdbwdS to be fixpoint set preserving even for changing $\alpha$ (\cref{the:iter_FWBW}), we can optimize $\alpha_2$ for the certification of the postcondition using line search.

In \cref{fig:alpha_hist}, we visualize the $\alpha_2$ selected in this manner for two different $\alpha_1$ used in \tool's first stage ($\gSss{\pr}{1}$).
We observe that the optimal $\alpha_2$, which may still fail to verify, depends heavily on both the concrete sample and the parameters of the first stage's solver, highlighting the value of choosing $\alpha_2$ adaptively.

\begin{wrapfigure}[22]{r}{0.38 \textwidth}
	\vspace{-10mm}
	\centering
	\begin{subfigure}[t]{0.95\linewidth}
		\centering
		\includegraphics[width=0.83\linewidth]{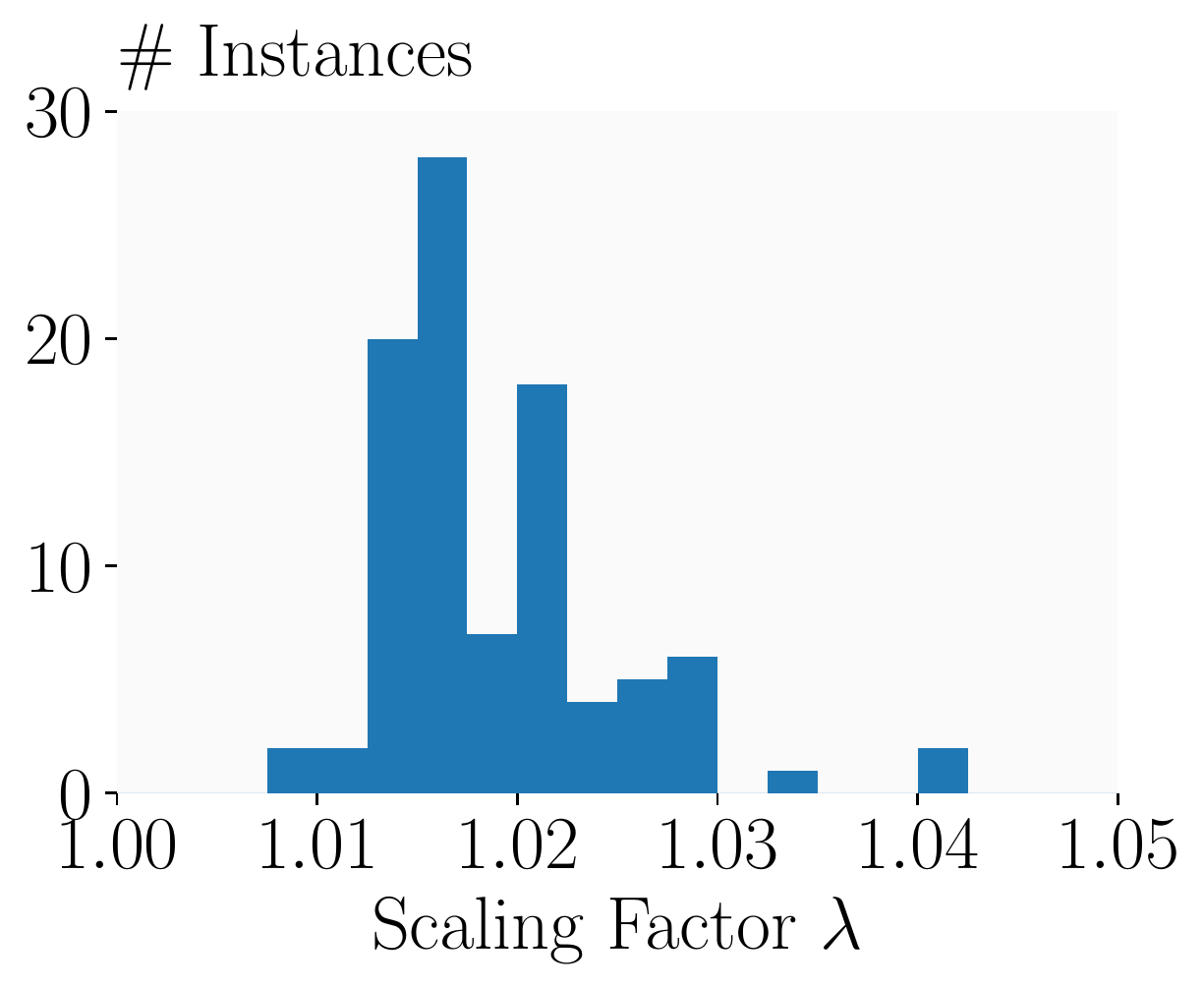}
		\vspace{-1mm}
		\caption{Precision}
		\label{fig:containment_comp_ratio}
	\end{subfigure}
	\begin{subfigure}[t]{0.95\linewidth}
		\centering
		\includegraphics[width=0.9\linewidth]{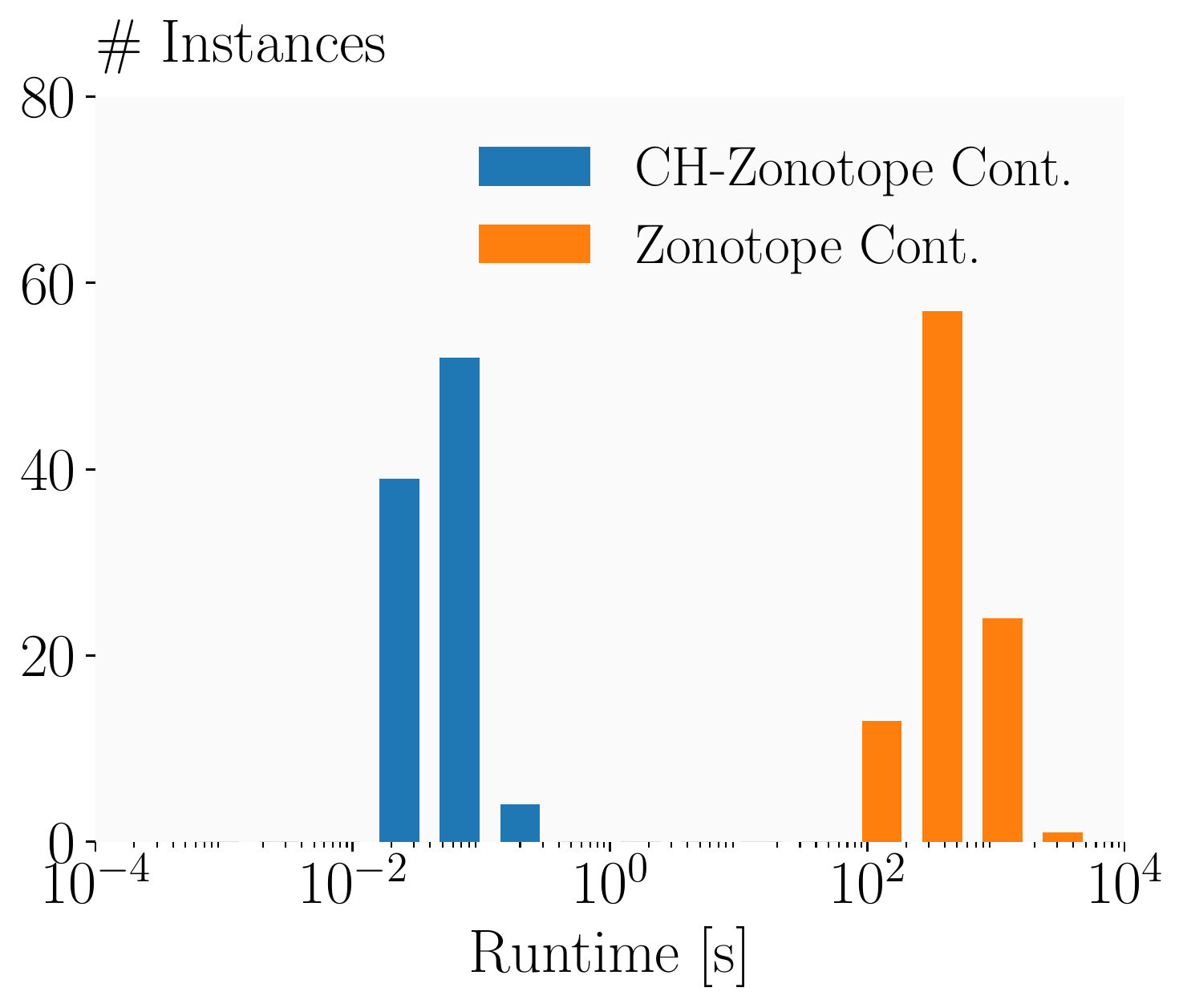}
		\vspace{-2mm}
		\caption{Runtime}
		\label{fig:containment_comp_time}
	\end{subfigure}
	\vspace{-4mm}
	\caption{Comparison between \domain Cont. (ours) and Zonotope Cont. \citep{SadraddiniT19} for single 40d containment checks.}
\end{wrapfigure}

\subsection{Containment Checks in High-Dimensions} \label{app:ablation_containment}
To assess the tightness of our approximate containment check (\emph{\domain Cont.}) in a realistic setting, we compare it against an (in small dimensions) close-to-lossless approximate containment check for general Zonotope (\emph{Zonotope Cont.}, Theorem 3 of \citet{SadraddiniT19}), implemented using GUROBI \cite{gurobi}. To obtain a tractable setting despite the much larger complexity of Zonotope Cont. ($\tilde{\bc{O}}(p^6)$ compared to \domain Cont.'s $\bc{O}(p^3)$), we choose the smallest model (\fcf) and a \fwdbwd iterator yielding problems in $p=40$ dimensions.
We apply \tool until we find containment with \domain Cont. and then conduct a binary search to find the largest scaling factor $\lambda$ such that, when applied to the inner \domain, Zonotope Cont. still succeeds.
Evaluating the first 100 samples, we observe that \domain Cont. is on average only $1.8\%$ less precise (see \cref{fig:containment_comp_ratio}) while being more than $4$ orders of magnitude faster on these 40-dimensional problems (see \cref{fig:containment_comp_time}). %
Conducting the hundreds of containment checks required by \tool to verify a single property would become practically intractable with Zonotope Cont. where a single check takes over $500s$ on average and up to $2000s$ in some instances.%

 \begin{wrapfigure}[17]{r}{0.39 \textwidth}
 	\vspace{-2mm}
 	\centering
 	\includegraphics[width=.87\linewidth]{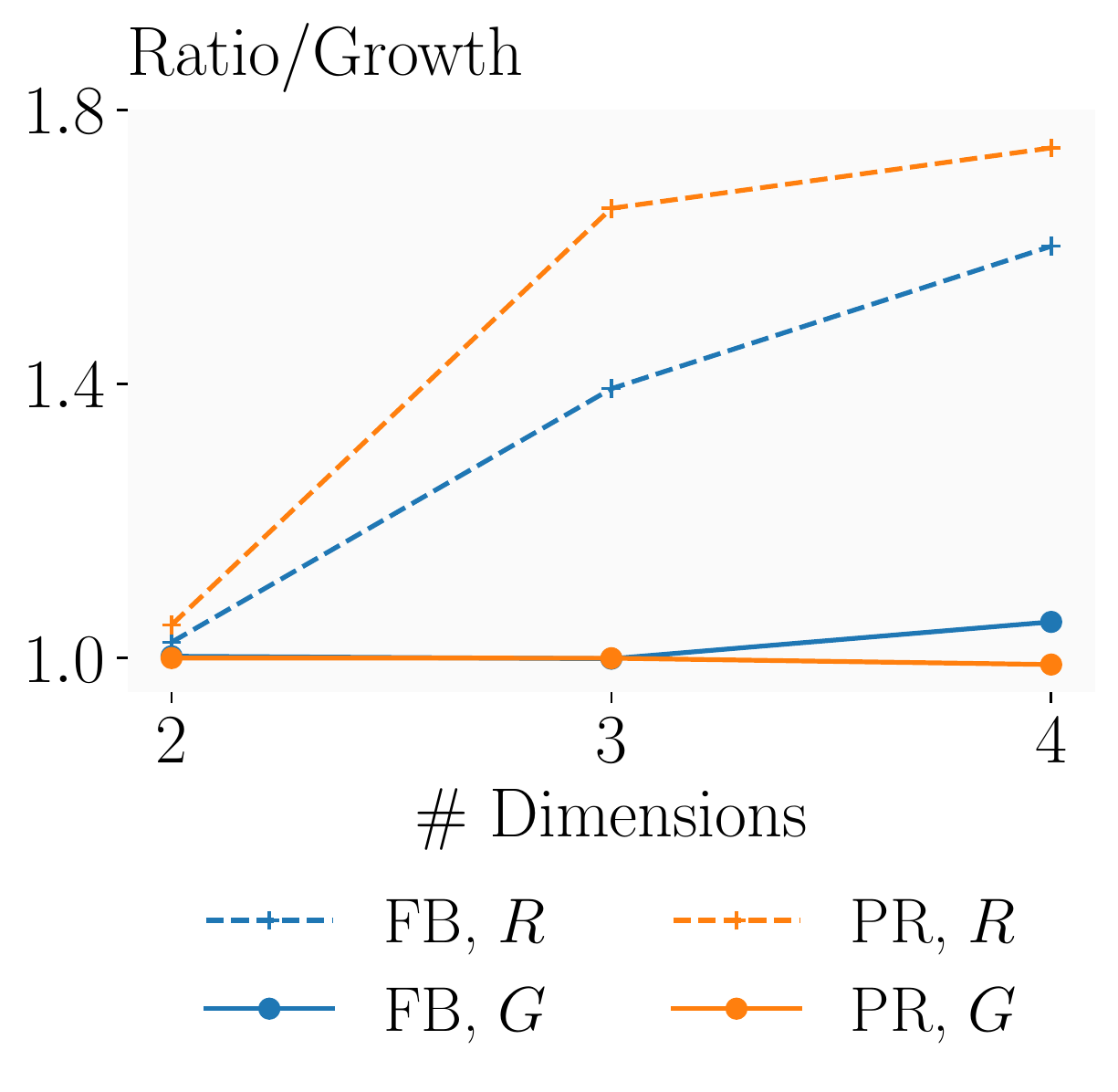}
 	\vspace{-3mm}
 	\caption{Effect of problem dimensionality on volume growth and ratio under error consolidation depending on the applied iterator.}
 	\label{fig:volume_ratio}
 \end{wrapfigure}
\subsection{Error Consolidation Case Study}  \label{app:ablation_error_consolidation}
To evaluate the effectiveness of our error consolidation method, we analyze its effect on the abstraction volume.
To enable the computation of exact volumina despite the exponential complexity \citep{gover2010determinants}, we train monDEQs with 2, 3, and 4 hidden dimensions on a toy dataset with 5 dimensional inputs sampled from a mixture of Gaussians and 3 classes and illustrate results in \cref{fig:volume_ratio}. 
We report two metrics: (i) the increase in abstraction volume induced by one application of error consolidation (volume ratio $R=\vol(\consolidate(\hat{\Z}_{n}))/\vol(\hat{\Z}_n)$), and (ii), the combined effect of error consolidation and the contractive properties of the iterator (volume growth $G = \vol(\hat{\Z}_{n+k})/\vol(\hat{\Z}_n)$),
where $\hat{\Z}_{n+k}$ is computed by consolidating $\hat{\Z}_{n}$ and applying $k=5$ iterations of $\vg^\#_\alpha$.
In both cases, we run for $250$ iterations, compute the average $G$ and $R$ over the last 50 iterations and report the median over $100$ inputs. To ensure a meaningful comparison, we have excluded samples where one dimension (and thereby the volume) collapses to $0$.
We observe that while the volume ratio $R$ increases with dimensionality, this volume growth is counteracted by the contractive property of the underlying iterator leading to a roughly constant volume with $G\approx 1$. While \fwdbwd shows a slight increase of this growth with dimensionality, we observe no such trend for \pr, despite its auxiliary variables leading to a faster accumulation of error terms and thus larger volume ratios $R$.

\begin{wrapfigure}[23]{r}{0.38 \textwidth}
	\vspace{-4.5mm}
	\centering
	\begin{subfigure}[t]{0.95\linewidth}
		\centering
		\includegraphics[width=1.0\linewidth]{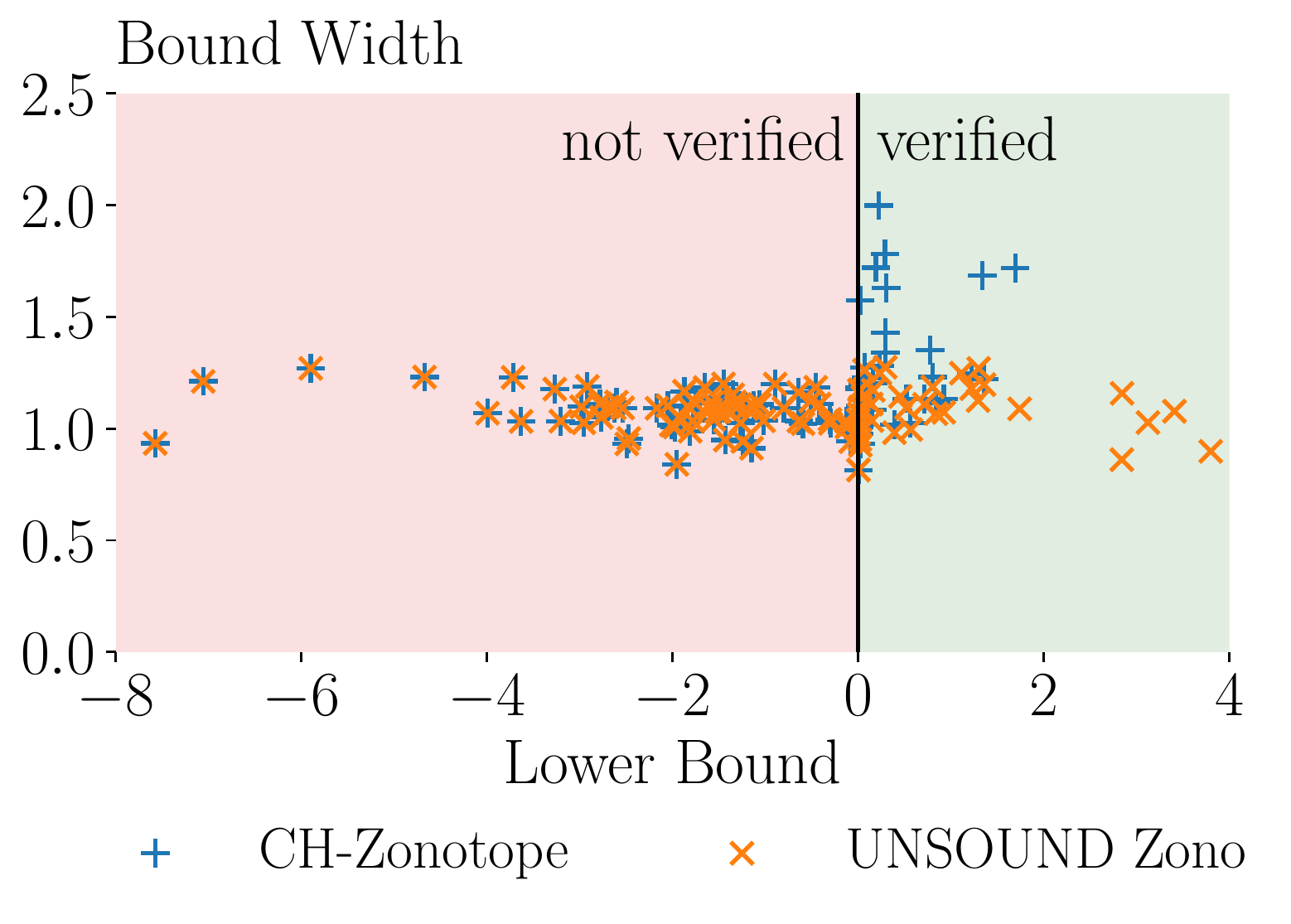}
		\vspace{-6mm}
				\caption{\mnist \fcf}
				\label{fig:unsound_zono_m}
	\end{subfigure}
		\begin{subfigure}[t]{0.95\linewidth}
			\centering
			\includegraphics[width=1.0\linewidth]{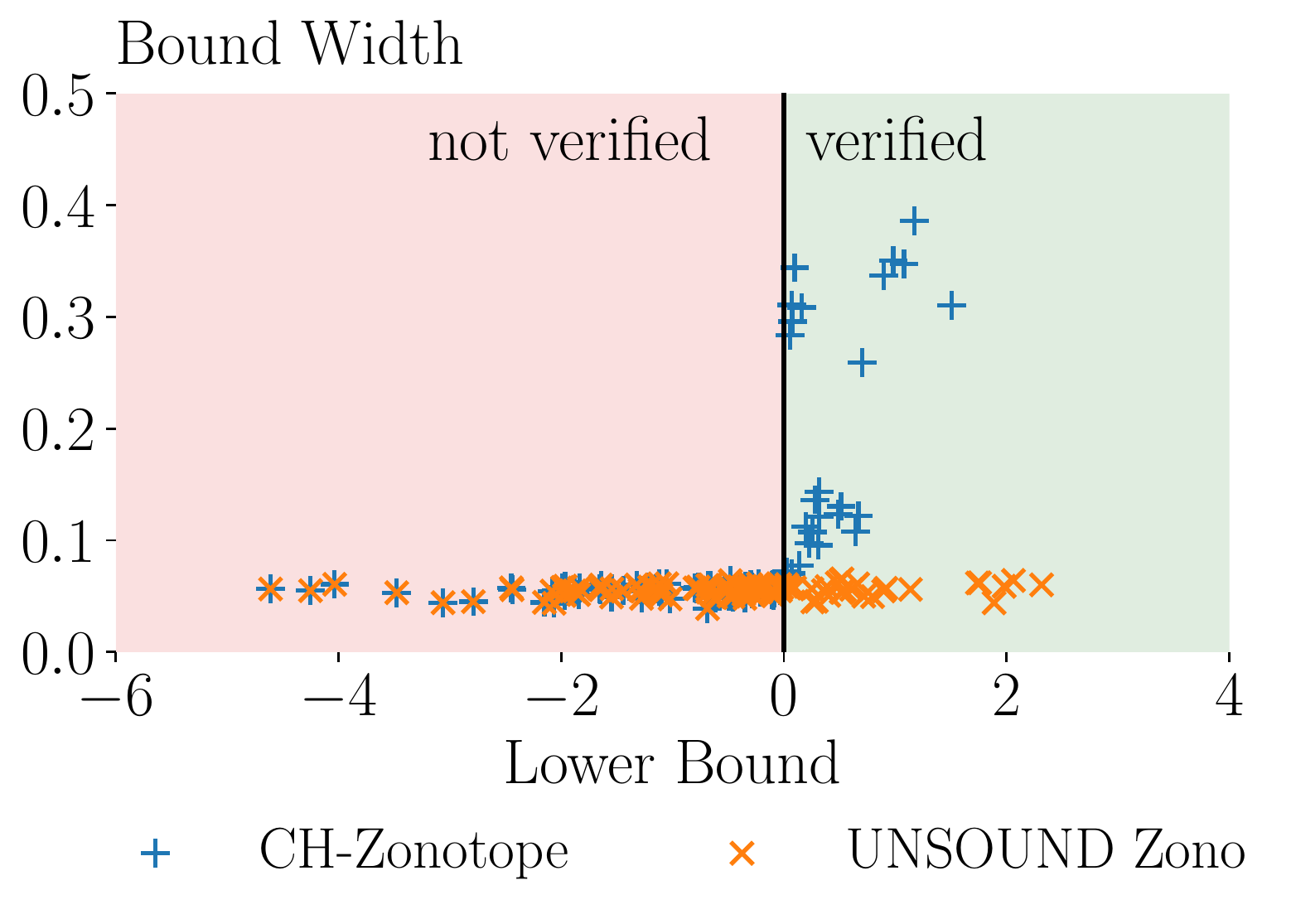}
			\vspace{-6mm}
			\caption{\cifar \convsm}
			\label{fig:unsound_zono_c}
		\end{subfigure}
	\vspace{-4mm}
	\caption{Comparison of bounds obtained using \domain with error consolidation and containment checks and Zonotope without (unsound).}
	\label{fig:unsound_zono}
\end{wrapfigure}
\paragraph{Error Consolidation in \tool} As \tool mostly requires error consolidation to obtain the proper \domain required for our efficient containment check, we only apply it until containment has been shown (see \cref{sec:tool}) and then tighten the obtained abstraction by applying additional solver iterations.
To analyze the combined effect of these two factors on the overall precision of \tool, we conduct the following experiment: We first run \tool using \domain as usual, tracking exactly how many iterations of which iterator were applied. Then, we use a standard Zonotope and apply exactly the same iterations but with neither error consolidation nor containment checks.
Note that the bounds obtained with the latter approach do not imply any guarantees as we have not shown containment.
Finally, we compare the bounds on the verification objective obtained with the two methods in \cref{fig:unsound_zono}, where points to the right of the vertical line correspond to verified properties and points to the left to unverified ones.
We observe that the bounds obtained for unverified properties are practically identical as the imprecisions introduced by error consolidation are offset by the contractive properties of the iterator. For verified properties, \tool terminates as soon as the property has been verified and thus before this contraction takes place, leading to much larger differences. We did not find any instance where the unsound Zonotope bounds would have verified a property that \domain did not verify.
We thus conclude that error consolidation has a negligible effect on the overall precision of \tool while enabling the efficient containment checks required to make the analysis tractable.
Further, error consolidation significantly reduces \tool's runtime, even when compared to Zonotope without any containment checks, e.g., on \fcf the average runtime is reduced almost $5$-fold.

}{}

\message{^^JLASTPAGE \thepage^^J}

\end{document}